\documentclass{article}
\usepackage{arxiv}
\usepackage{amsmath,amssymb,amsthm}
 \usepackage{paralist}
 \usepackage{graphics} 
  \usepackage{epsfig} 
  \usepackage{graphicx}  \usepackage{epstopdf} 
   \usepackage[colorlinks=true]{hyperref}
   \hypersetup{urlcolor=blue, citecolor=red}

\newtheorem{theorem}{Theorem}[section]

\newtheorem{remark}{Remark}

\usepackage{aliases}

\title{The efficacy and generalizability of conditional GANs for posterior inference in physics-based inverse problems} 
\renewcommand{\shorttitle}{Efficacy and generalizability of conditional GANs}
\renewcommand{\headeright}{}
\renewcommand{\undertitle}{}

\author{Deep Ray \\
Department of  Aerospace and Mechanical Engineering\\
University of Southern California\\
Los Angeles, CA 90089, USA\\
deepray@usc.edu
\And
Harisankar Ramaswamy \\
Department of  Aerospace and Mechanical Engineering\\
University of Southern California\\
Los Angeles, CA 90089, USA\\
hramaswa@usc.edu
\And
Dhruv V. Patel \\
Department of Mechanical Engineering\\
Stanford University\\
Stanford, CA 94305, USA\\
dvpatel@stanford.edu
\And
Assad A. Oberai\\
Department of  Aerospace and Mechanical Engineering\\
University of Southern California\\
Los Angeles, CA 90089, USA\\
aoberai@usc.edu
}

\begin{document}
\maketitle

\bigskip

\begin{abstract}
In this work, we train conditional Wasserstein generative adversarial networks to effectively sample from the posterior of physics-based Bayesian inference problems. The generator is constructed using a U-Net architecture, with the latent information injected using conditional instance normalization. The former facilitates a multiscale inverse map, while the latter enables the decoupling of the latent space dimension from the dimension of the measurement, and introduces stochasticity at all scales of the U-Net. We solve PDE-based inverse problems to demonstrate the performance of our approach in quantifying the uncertainty in the inferred field. Further, we show the generator can learn inverse maps which are local in nature, which in turn promotes generalizability when testing with out-of-distribution samples.

\end{abstract}


\section{Introduction}
Inverse problems arise in various areas of science and engineering, such as computerized tomography \cite{natterer2001}, seismology \cite{Gouveia1997,Yilmaz2001}, climate-modeling \cite{Jackson2004,Huang2005inverse}, and astronomy \cite{Craig1986InversePI}. While the forward/direct problem is generally well-studied and easier to solve, the inverse problem can be notoriously hard to tackle due to the lack of well-posedness \cite{hadamard1902problemes}. Reconstructing the  inferred field from measurements corrupted by noise, which is usually the case for most realistic problems, makes the task even more challenging. Bayesian inference provides a principled strategy to overcome these challenges by posing the problem in a stochastic framework. Here Bayes' rule is used to update the prior belief about the inferred field, based on the knowledge gained from  measurements \cite{dashti2017}. The end result is the recovery of an expression of the posterior distribution for the inferred field, or the ability to generate samples from such a density. As opposed to deterministic methods of solving inverse problems, such as regularization-based strategies \cite{engl1996regularization}, Bayesian inference can be used to quantify the uncertainty in the reconstructed field. The knowledge of this uncertainty can play a crucial role in applications where important decisions need to be made based on the inferred field. 

There has been growing interest in using deep learning tools to overcome some of the computational bottlenecks that arise when using Bayesian inference in practise. For instance, the Bayes' update formula requires an explicit prior to be constructed. Practitioners prefer to use simple densities, such as a multivariate Gaussian distribution, which in some cases ensures an expression for the posterior that is easy to sample from. However, such a density is not suitable to represent complicated priors in most practical problems, particularly in applications where the prior needs to be constructed from snapshots of the inferred field. 

Deep generative models have proven to be successful in learning the underlying probability distribution from data arising from medical imaging, geophysical processes and computer vision \cite{Whang2021,goh2021solving, rizzuti2020parameterizing,ongie2020deep}. In \cite{patel2021gan,patel2021ganb}, generative adversarial networks (GANs) were used to learn the prior from data arising in physics-based models. Further, this strategy helped in significantly reducing the overall dimension of the problem, thus making posterior sampling algorithms like MCMC computationally tractable. Physics informed GANs \cite{Yang2020pigan} have also been designed to solve both forward and inverse problems, where the underlying stochastic differential equations have been encoded into the GAN architecture. In \cite{adler2018deep}, a conditional GAN (cGAN) was designed to directly approximate and sample from the posterior distribution arising in the image reconstruction problem for computed tomography (CT) of the liver. At the same time, a mathematical framework was developed to prove the weak convergence of the learned distribution. More recently, cGANs have been used to solve inverse problems associated with nearshore bathymetry \cite{qian2020application}. 

Motivated by \cite{adler2018deep}, we propose a deep learning based algorithm to solve Bayesian inverse problems modelled by partial differential equations (PDEs). In particular, we train a conditional Wasserstein GAN (cWGAN) on a dataset of samples drawn from the joint probability distribution of the inferred field and the measurement. Such a dataset can be either acquired from experiments, or generated synthetically by assuming a suitable prior on the inferred field and numerically solving the forward PDE model. Further, samples from the posterior distribution can be generated by simply querying the trained network instead of depending on MCMC algorithms, which provides a significant computational boost over traditional Bayesian inference techniques. Using a Wasserstein adversarial loss to train the cGAN allows us to take advantage of the theory developed in \cite{adler2018deep} which states that the generated posterior statistics converge to the true posterior statistics. We highlight below the main contributions of the present work:   
\begin{enumerate}
    \item We construct a cWGAN which uses a residual block-based U-Net architecture for the generator subnetwork. While the general structure of the generator is similar to \cite{adler2018deep}, we use a much simpler architecture for the critic as compared to the conditional mini-batch critic constructed in \cite{adler2018deep}.
    \item We inject the latent variable at different scales of the U-Net using conditional instance normalization (CIN) \cite{dumoulin2017learned}. This helps introduce multi-scale stochasticity, while giving the flexibility to choose the latent space dimension independently of the dimension of the measurement or the inferred field. This is another key distinction between the architecture proposed in this paper as compared to \cite{adler2018deep} where the latent variable is stacked as an additional channel to the generator input. 
    \item In contrast to prior works that typically illustrate the performance of conditional GANs on a single physics-based application, such as CT imaging of the liver \cite{Adler2017}, nearshore bathymetry \cite{qian2020application} or hydro-mechanical processes in porous media \cite{kadeethum2021}, we consider multiple PDE-based inverse problems to demonstrate the performance of the proposed algorithm. Specifically, we consider the problem of inferring the initial condition of the heat conduction problem, the coefficient of conductivity for the steady-state heat conduction problem, and the shear modulus in elasticity imaging. Further, the cWGAN for the elasticity imaging problem is trained on synthetic data and then tested on experimentally measured data.
    \item Finally, we investigate the generalizablity of the cWGAN by testing the performance of the trained network on out-of-distribution (OOD) measurements. We also provide a theoretical justification for the generalization (Theorem \ref{thm:gen2}) which is related to the local nature of the inverse map learned by the cWGAN.
\end{enumerate}

The rest of the paper is organized as follows. Section \ref{sec:problem_setup} describes the problem setup and how a cWGAN can be used to solve inverse problems. Section \ref{sec:gan} details and motivates the specialized cWGAN architecture considered in this work. Section \ref{sec:general} provides a theoretical explanation for the generalization observed in the numerical results presented in Section \ref{sec:results}. Concluding remarks are made in Section \ref{sec:conclusion}.


\section{Problem formulation}\label{sec:problem_setup}
We begin by considering a forward/direct problem given by the map
\[
\f : \dbx \rightarrow \dby
\]
which takes an input $\x \in \dbx \subset \Ro^{\Nx}$ to produce an output/measurement $\y \in \dby \subset \Ro^{\Ny}$. For the examples considered in the present work, the forward problem involves solving a PDE. In this context, $\x$ denotes quantities such as the initial condition, source term or material properties, $\f$ denotes the solution operator of the PDE approximated by a suitable finite volume or finite element scheme, while $\y$ denotes the (discretized) solution of the PDE. Further, the measurement $\y$ might be corrupted by noise, which is typical of most realistic scenarios.

We are interested in solving the inverse problem: given a measurement $\y$, infer the associated $\x$. To solve this in the Bayesian framework, we assume that the inferred field is modeled as realizations of the vector-valued random variable $\rvx$ with a (prior) density $\px$. Then, the forward map induces a marginal density $\py$ for the random variable $\rvy=\f(\rvx)$. The goal is to determine the posterior density $\pxgy$ given a measurement $\y$. In particular, we want to effectively generate samples from the posterior, which can then be utilized to determine posterior statistics. 

We assume access to $M$ realizations $\{\x^{(i)}\}_{i=1}^M$ from $\px$, and construct the dataset of pairwise samples $\mathcal{S} = \{(\x^{(i)},\y^{(i)})\}_{i=1}^M$ by solving the forward model, exactly or approximately. Note that the samples in $\mathcal{S}$ represent realizations from the joint distribution $\pxy$ which are supplied to a GAN. 

A GAN typically comprises two neural subnetworks, namely a generator $\g$ and a critic $d$. Let $\z \in \dbz \subset \Ro^{\Nz}$ be a latent variable modeled using the random variable $\rvz$ with a distribution $\pz$, that is easy to sample from; for instance a multivariate Gaussian distribution. Then the conditional GAN subnetworks are given by the mappings
\begin{equation*}\label{eqn:gen_critic}
\g: \dbz \times \dby \rightarrow \dbx, \qquad d:\dbx \times \dby \rightarrow \Ro.
\end{equation*}
For a given $\y$, the generator induces a distribution on $\dbx$. We denote this conditional distribution as $\pgxgy$. The critic's role is to distinguish between true samples $(\x,\y) \sim \pxy$ and fake samples $(\x^g,\y)$, where $\x^g \sim \pgxgy$. 

To train the conditional GAN, we define the following objective function
\begin{equation}\label{eqn:loss}
\begin{aligned}
\loss(d,\g)  
&:=  \ex{(\x,\y) \sim \pxy \\ \z \sim \pz}{d(\x,\y) - d(\g(\z,\y),\y)} \\
&=\ex{\x \sim \pxgy \\ \y \sim \py }{d(\x,\y)} - \ex{ \x^g \sim \pgxgy \\ \y \sim \py}{d(\x^g,\y)}. 
\end{aligned}
\end{equation}
In \eqref{eqn:loss}, we begin with the definition of the loss term (also how it is constructed in practise), and derive the final equality using the definition of the joint distribution and that of $\pgxgy$. The cWGAN is trained by solving the following min-max problem
\begin{equation}\label{eqn:minmax}
    (d^*,\g^*) = \argmin{\g} \ \argmax{d} \ \loss(d,\g).
\end{equation}
Under the assumption that $d$ is 1-Lipschitz, it can be shown that $\g^*$ in \eqref{eqn:minmax} minimizes the (mean) Wasserstein-1 distance between $\pxgy$ and $\pgxgy$ 
\begin{equation*}\label{eqn:w1min}
\g^* = \argmin{\g} \ex{\y \sim \py}{W_1(\pxgy,\pgxgy)},
\end{equation*}
where $\py$ is the marginal distribution of $\rvy$. A proof of this equivalence can be found in \cite{adler2018deep}. The Lipschitz constraint on the critic can be weakly imposed using a gradient penalty term \cite{gulrajani2017improved} while training $d$ (also see Appendix \ref{app:gan}). 

Convergence in the Wasserstein-1 metric (for $\y \sim \py$) is equivalent to the weak convergence of the two measures \cite{villani2008optimal}, i.e.,
\begin{equation}\label{eqn:weak_conv}
    \ex{\x \sim \pxgy}{\ell(\x)} = \ex{\x^g \sim \pgxgy}{\ell(\x^g)} = \ex{\z \sim \pz}{\ell(\g^*(\z,\y))}, \quad \forall \ \ell \in C_b(\dbx),
\end{equation}
where $C_b(\dbx)$ is the space of continuous bounded functions defined on $\dbx$.
This implies that for a given measurement $\y$, any statistic of the posterior distribution can be evaluated by sampling $\z$ from $\pz$ and pushing it through the trained  generator $\g^* (\cdot, \y)$.

Summarizing, the proposed algorithm to evaluate posterior statistics for a given problem (characterized by the forward map $\f$) is:
\begin{enumerate}
\item Generate or collect $\{\x^{(i)}\}_{i=1}^M$ sampled from $\px$.
\item For each $\x^{(i)}$ solve the forward model to determine $\y^{(i)}$. Thus, one constructs the pairwise dataset $\mathcal{S} = \{(\x^{(i)},\y^{(i)})\}_{i=1}^M$.
\item Train the cWGAN on the dataset $\mathcal{S}$ by solving the min-max problem \eqref{eqn:minmax} using a stochastic gradient-based optimizer, such as Adam \cite{kingma2017adam}.
\item Given a measurement $\y$, use Monte-Carlo to compute the desired statistics of the posterior 
\begin{equation}\label{eqn:mc}
\ex{\x \sim \pxgy}{\ell(\x)} \approx \frac{1}{K} \sum_{i=1}^K \ell(\g^*(\z^{(i)},\y)), \quad \z^{(i)} \sim \pz, \quad \ell \in C_b(\dbx).
\end{equation}
\end{enumerate}


\section{The cWGAN architecture}\label{sec:gan}
The problems considered in this work assume the underlying two-dimensional PDEs to be discretized on a uniform grid using a finite difference or finite element solver. The corresponding $\x$ and $\y$ fields arising from the discretization of the measured and inferred fields can be viewed as images with a single channel for scalar-valued fields, or multiple channels for vector-valued fields. 

\subsection{U-Net generator: a multi-scale map} The most popular network architecture to capture image-to-image transformations is a U-Net. First proposed to efficiently segment biomedical images \cite{Unet}, a U-Net extracts features from the input image at multiple scales to construct the multi-scale features of the output field. As shown in Figure \ref{fig:arch}(a), it takes the input image $\y$ and pushes it through a contracting path which extracts the multi-scale features as channels, while downsampling (coarsening) the resolution of the image. This is followed by an expanding path that slowly upscales the image to the desired resolution of the output $\x$. Through skip connections, the U-Net pulls in the multi-scale features of $\y$ learned in the contracting branch, which is concatenated to the features (of the same spatial dimension) being learned in the expanding branch. Details about the constitutive blocks of the U-Net are given in Apendix \ref{app:gan}.

\subsection{Conditional instance normalization (CIN): multi-scale stochasticity} A key difference between the generator architecture for the cWGAN we use compared to one considered in \cite{adler2018deep} is the way the latent variable $\z$ is fed into $\g$. In \cite{adler2018deep}, $\z$ is stacked as an additional channel to the input $\y$ of the generator. We instead use CIN \cite{dumoulin2017learned} to inject the latent information at various scales of the U-Net, as shown in Figure \ref{fig:arch}(a). More precisely, an intermediate tensor $\w$ of shape $H^\prime \times W^\prime  \times C^\prime $ at a given scale of the U-Net, is transformed by the following normalization for each channel $j=1,\cdots,C^\prime$
\begin{equation}\label{eqn:norm}
    CIN(\w,\z)_j= \boldsymbol \alpha(\z)_j\otimes \left(\frac{\w_j-\mu(\w_j)}{\sigma(\w_j)} \right) \oplus \boldsymbol \beta(\z)_j,
\end{equation}
where $\boldsymbol \alpha:\Ro^{N_z} \rightarrow \Ro^{C^\prime}$ and $\boldsymbol \beta:\Ro^{N_z} \rightarrow \Ro^{C^\prime}$ are (learnable) convolution layers that take $\z$ as the input. In \eqref{eqn:norm} we have used the notations $\otimes$ and $\oplus$ respectively to represent element-wise multiplication and summation in the channel direction, while $\mu(.)$ and $\sigma(.)$ respectively compute the mean and fluctuations along the spatial directions. One can interpret the action of CIN in the following manner. Each feature of the intermediate tensor is decomposed into average value and a fluctuating component. The precise value of this average and the intensity of the fluctuation is allowed to be stochastic, depending on the value of the $\z$ through the (learned) maps $\boldsymbol \alpha$ and $\boldsymbol \beta$. 

There are two fundamental advantages of injecting $\z$ using CIN:
\begin{enumerate}
    \item We can choose the dimension $\Nz$ of the latent variable independently of the dimension $\Ny$ of $\y$. 
    \item We introduce stochasticity at multiple scales/levels of the U-Net.
\end{enumerate}
We remark that the CIN strategy has also been effectively used to learn many-to-many maps with Augmented CycleGANs \cite{almahairi2018augmented}.

\subsection{Simple critic architecture}
In \cite{adler2018deep}, the authors noticed the their generator was insensitive to variability in $\z$ (mode collapse), and thus designed a specialized conditional mini-batch critic to overcome this issue. However, we did not notice such issues in our experiments when using the simpler critic architecture shown in Figure \ref{fig:arch}(b). This can perhaps be attributed to how $\z$ is being injected at multiple scales in $\g$, which is yet another advantage of using CIN.

\begin{figure}[htp]
\begin{center}
\subfigure[Generator]{\includegraphics[width=\textwidth]{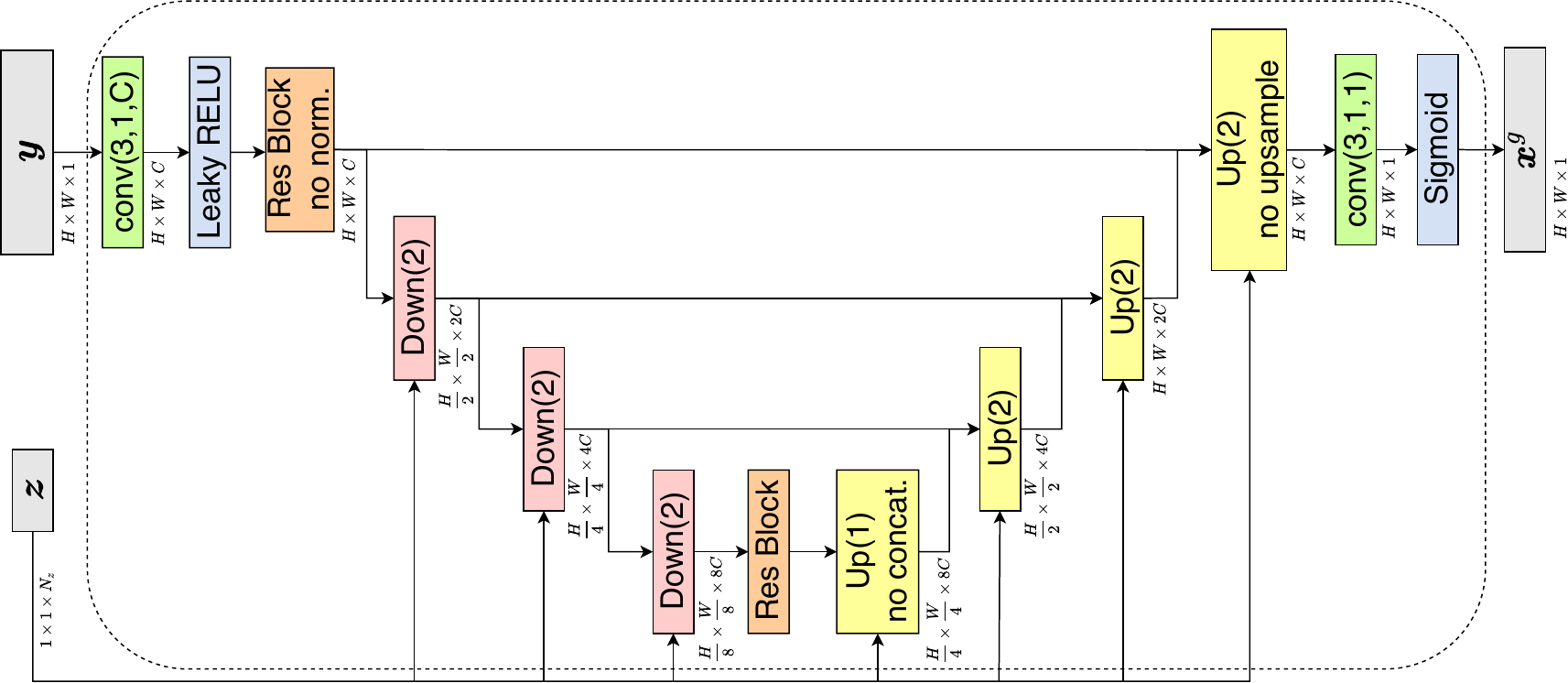}}
\subfigure[Critic]{\includegraphics[width=0.7\textwidth]{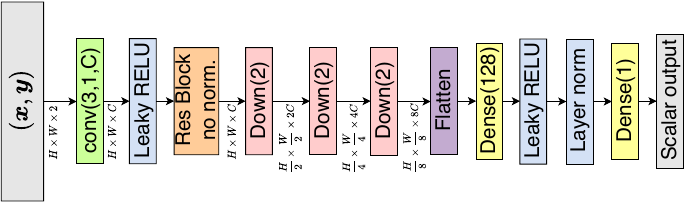}}
\caption{Architecture (Type A) of generator and critic used in the conditional GAN. The spatial dimension $\Nx = H \times W$ of the input, the channel parameter $C$, the latent dimension $N_z$ and the depth of generator and critic vary for each experiment.}
\label{fig:arch}
\end{center}
\end{figure}


\section{A discussion on generalization}\label{sec:general}
In the numerical results presented in Section \ref{sec:results}, we observe that the cWGAN trained using a prior distribution of $\x$ of a certain type, produces useful results even when the measurement corresponds to an $\x$ that is not selected from this distribution. For example, in the inverse heat conduction problem, the cWGAN trained using samples where $\x$ corresponds to a single circular inclusion, produces reasonable results for measurements where $\x$ represents two circular inclusions, or a single elliptical inclusion. Similarly, for the inverse initial condition problem, the cWGAN trained on MNIST dataset of handwritten digits produces reasonable results for the notMNIST dataset which is comprised of digits, letters and other symbols that are not handwritten. We believe that this ``generalizability'' learned by the cWGAN emanates from the fact that the inverse map (the map from $\y$ to $\x$) and therefore the conditional distribution $\pxgy$ may depend only on a subset of the components of the measurement vector. Then if the cWGAN also learns this independence, and is perfectly trained for a single value of the complement of this subset, it learns the true conditional distribution for all measurements. We prove this assertion in the Theorem below, where we make use of $p(\boldsymbol{r}) \weq \tilde{p}(\boldsymbol{r})$ to denote weak equivalence of two distributions defined on $\Omega$, i.e.,
\[
\ex{\boldsymbol{r} \sim p}{\ell(\boldsymbol{r})} = \ex{\boldsymbol{r} \sim \tilde{p}}{\ell(\boldsymbol{r})},  \ \forall \ \ell \in C_b(\Omega).
\]

\begin{theorem}\label{thm:gen2}
Let $\xbrv$ be a subset of the components of the random vector $\rvx$ and let $\xprv$ be the complement of this subset. 
For this choice of $\xbrv$, let $\ybrv$ be a subset of the components of the random vector $\rvy$ and let $\yprv$ be the complement of this subset. Thus, for a given $\xbrv$, we may partition $\dby = \dbyb \times \dby'$ and $\rvy = (\ybrv, \yprv)$. Note that this decomposition depends on the choice of $\xbrv$. Let $\pxbgy$ denote the conditional distribution for $\xbrv$ obtained after marginalizing $\pxgy$ over $\xprv$.  
Further, assume the following conditions hold:
\begin{enumerate}
    \item $\pxbgy$ is weakly independent of $\yprv$. That is for any $\yb \in \dbyb$ and $\yp_1, \yp_2 \in \dby'$, 
    \begin{eqnarray}
    \pxbgy (\xb|(\yb,\yp_1)) \weq \pxbgy (\xb|(\yb,\yp_2)). \label{eq:t2}
    \end{eqnarray}    
    \item The conditional density learned by the generator, $\pgxbgy$, inherits this independence. That is, for any $\yb \in \dbyb$ and $\yp_1, \yp_2 \in \dby'$,   
    \begin{eqnarray}
    \pgxbgy (\xb|(\yb,\yp_1)) \weq \pgxbgy (\xb|(\yb,\yp_2)). \label{eq:t3}
    \end{eqnarray}
    \item For some fixed value of $\yprv = \yp_a \in \dby'$, the conditional density learned by the generator is weakly equal to the true conditional density. That is for any $\yb \in \dbyb$,
    \begin{eqnarray}
    \pgxbgy (\xb|(\yb,\yp_a)) \weq \pxbgy (\xb|(\yb,\yp_a)). \label{eq:t1}
    \end{eqnarray}
\end{enumerate}
Then, for any $\y \in \dby$, $\pgxbgy(\xb|\y) \weq \pxbgy(\xb|\y)$.
\end{theorem}

\begin{proof}
For any $\yb \in \dbyb$, and $\yp_1 \in \dby'$, we begin by using \eqref{eq:t3} with $\yp_2 = \yp_a$. This yields,
\begin{equation*}
\pgxbgy (\xb|(\yb,\yp_1))  \weq \pgxbgy (\xb|(\yb,\yp_a)).
\end{equation*}
Thereafter, we use \eqref{eq:t1} in the equation above to arrive at 
\begin{equation*}
\pgxbgy (\xb|(\yb,\yp_1))  \weq \pxbgy (\xb|(\yb,\yp_a)).
\end{equation*}
Finally, using \eqref{eq:t2} with $\yp_2 = \yp_a$ in the equation above, we have 
\begin{equation*}
\pgxbgy (\xb|(\yb,\yp_1))  \weq \pxbgy (\xb|(\yb,\yp_1)).
\end{equation*}
Since any $\y \in \dby$ can be written as $\y  = (\yb,\yp)$, this completes the proof.
\end{proof}

It is highly non-trivial to show the satisfaction of the above conditions for the PDE-based inverse problems considered in this paper. However, we can interpret and investigate their implications. Further, for the cases considered we can demonstrate numerically that they are satisfied.

The first condition points to the local nature of the (regularized) inverse map. This can be better understood by looking at Figure \ref{fig:schem}, where we have considered $\xbrv$ to be the subset of components of $\rvx$ centered around a given spatial location. For a spatially local inverse, the value of $\rvx$ at a given spatial location depends only on the values of the measurements in its vicinity. Therefore, for the choice of $\xbrv$  displayed in Figure \ref{fig:schem}, the corresponding subset of measurements, $\ybrv$, that influences $\xbrv$ is comprised of components that are in the spatial vicinity of $\xbrv$. We numerically demonstrate this locality for the inverse initial condition problem, while we establish it semi-analytically through a perturbation argument for the inverse thermal conductivity problem.

The second condition of Theorem \ref{thm:gen2} alludes to the locality of the inverse map learned by the cWGAN. We demonstrate this numerically by considering the trained $\g$ and compute the gradient of the $k$-th component of the prediction with respect to the network input $\y$. We show that the gradients are concentrated in local neighbourhoods of the corresponding $\y$ components.

The third condition essentially claims that the cWGAN is well trained. Assuming convergence in the Wassertein-1 metric, the trained generator satisfies the weak relation \eqref{eqn:weak_conv}. Then, the weak expression \eqref{eq:t1} can be obtained by choosing $\ell \in C_b(\dbx)$ to be of the form $\ell = \overline{\ell} \circ \overline{P}$, where $\overline{\ell}$ is a continuous bounded functional of $\xbrv$ while $\overline{P}$ projects $\rvx$ to $\xbrv$.

Using the result of Theorem \ref{thm:gen2}, we conclude that the cWGAN trained using a given dataset will yield reasonable results for another dataset as long as the \textit{spatially local features in the two sets are similar}. This will hold even if the global features in these sets are different. The generalizability results obtained in Section \ref{sec:results} are consistent with this conclusion. In each instance, the cWGAN is trained on data in which the local spatial distribution of the inferred field is of the form of a sharp transition along a smooth curve. Once the cWGAN is trained on this data, and then applied to instances where this local feature is retained, it continues to perform well. This is despite the fact that the the global features among the two datasets are different.

\begin{figure}[htp]
\centering
\includegraphics[width=0.8\textwidth]{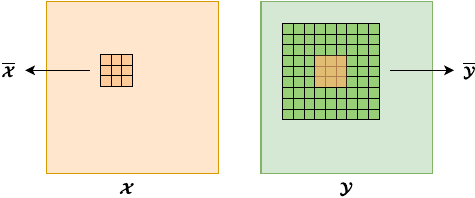}
\caption{A schematic representation of the subsets $\xbrv$ and $\ybrv$ in our problems.}
\label{fig:schem}
\end{figure}


\section{Numerical results}\label{sec:results}
We consider inverse problems arising in the context of three distinct PDE models to demonstrate the performance of the proposed algorithm. The first two models correspond to the time-dependent and steady-state heat conduction equations, for which the training and test data are synthetically generated. The third model solves the inverse elastography problem, where the training data is generated synthetically but the test data comes from real-life experiments. A separate cWGAN is trained for each problem, with specific details about the architecture and hyper-parameters given in Appendix \ref{app:gan}. 

In Table \ref{tab:dim_red}, we list the dimension reduction that is achieved using the cWGAN for all the numerical experiments described in the following sections. We define the dimension compression rate $\dcomp = \lfloor\Nx/\Nz\rfloor$, where $\lfloor . \rfloor$ denotes the floor function.

We also show the performance of the trained networks on out-of-distribution (OOD) test samples, to highlight the generalization capabilities of such an algorithm. Further, based on the discussion in the previous section, we explain this generalization by presenting numerical and analytical evidence that the true and learned inverse maps are local. 

\begin{table}[!htp]
        \renewcommand{\arraystretch}{1.5}
        \centering
        \caption{Dimension reduction with cGANs}
        \begin{tabular}{c c c c c}
        \toprule
        Inferred field & \multicolumn{2}{c}{\begin{tabular}[c]{@{}c@{}}Initial condition\end{tabular}} & \begin{tabular}[c]{@{}c@{}}Conductivity\end{tabular} & \begin{tabular}[c]{@{}c@{}}Sheer modulus\end{tabular} \\
        \cmidrule{2-3}
        Training Data &   Rectangular &   MNIST  & Circles & Circles \\ 
        \midrule
        $\Nx$ & 784 & 784 & 4096 & 3136\\
        $\Nz$ & 3 & 100 & 50 & 50\\
         $\dcomp = \lfloor\Nx/\Nz\rfloor$ & 261 & 7 & 81 & 62\\
        \bottomrule
        \end{tabular}\label{tab:dim_red}
        \end{table}


\subsection{Heat conduction: inferring initial condition}\label{sec:inv_ic} 
Consider the two-dimensional time-dependent heat conduction problem on a bounded domain $\Omega \subset \Ro^2$ with Dirichlet boundary conditions
\begin{alignat}{2}
    \frac{\partial u(\s,t)}{\partial t} -\nabla \cdot (\kappa (\s) \nabla u (\s,t)) &= b(\s), \qquad 
     &&\forall \ (\s,t) \in \Omega \times (0, T)  \label{eqn:pde_t_heat} \\
    u (\s, 0) &= u_0(\s),  \qquad 
     &&\forall \ \s \in \Omega \label{eqn:ic_t_heat} \\
    u (\s, t) &= 0, \qquad
     &&\forall \ (\s,t) \in \partial \Omega \times (0, T). \label{eqn:bc_t_heat}
\end{alignat}
Here $u$ denotes the temperature field, $\kappa$ denotes the conductivity field, $b$ denotes the heat source and $u_0$ denotes the initial temperature field. Given a noisy temperature field at final time $T$, we wish to infer the initial condition $u_0$. This is a severely ill-posed problem as significant information is lost via the diffusion process when moving forward in time. 

We set the spatial domain to be $\Omega = [0,2 \pi]^2$ and the final time to be $T=1$. We denote the discretized initial and final (noisy) temperature field as $\boldsymbol{x}$ and $\boldsymbol{y}$, respectively, evaluated on a $28 \times 28$ Cartesian grid. The training and test data samples pairs are generated by i) sampling $\x$ based on some prior, ii)  computing the final temperature field for the given $\x$ by solving \eqref{eqn:pde_t_heat}-\eqref{eqn:bc_t_heat} using a central-space-backward-time finite difference scheme, and iii) adding uncorrelated Gaussian noise to obtain the noisy $\y$. 

\subsubsection{Rectangular prior} We assume the following rectangular prior on the initial condition
\begin{equation*}
u_0(\s) = \begin{cases}
2 +  2\frac{(s_1 - \xi_1)}{(\xi_3-\xi_1)}&\quad \text{if } s_1 \in [\xi_1,\xi_3], \ s_2 \in[\xi_4,\xi_2],\\
0 & \quad \text{otherwise}, 
\end{cases}
\end{equation*}
with $\xi_1,\xi_4 \sim \mathcal{U}[0.2,0.4]$ and $\xi_2,\xi_3 \sim \mathcal{U}[0.6,0.8]$. This leads to an initial condition which has a uniform background value of zero and a linear profile in the rectangular inclusion whose top-left and bottom-right corners are (randomly) located at $(\xi_1,\xi_2)$ and $(\xi_3,\xi_4)$, respectively. The final temperature fields are obtained by setting $b=0$ and $\kappa=0.64$  in the PDE model, and adding noise sampled from $\mathcal{N}(\boldsymbol{0}, \boldsymbol{I})$. In Figure \ref{fig:rect_samples}, we show a few samples of the discretized $\x$, the clean and noisy $\y$ used to construct the datasets.

\begin{figure}[htp]
\centering
\subfigure[Sample 1]{\includegraphics[width=0.48\textwidth]{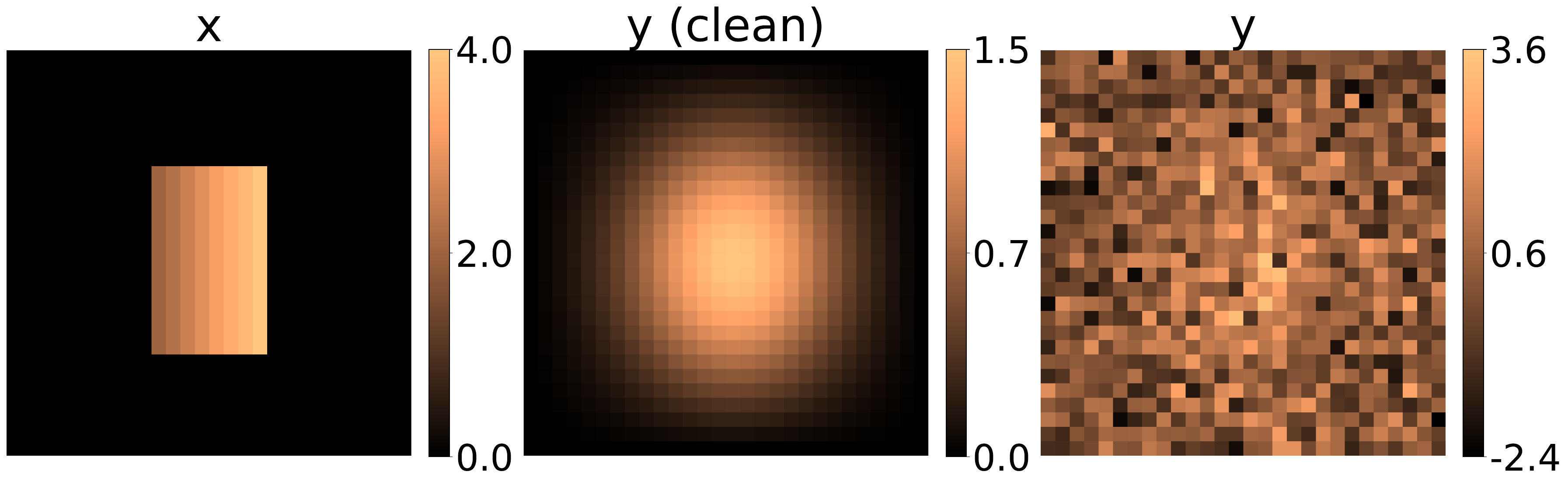}}
\subfigure[Sample 2]{\includegraphics[width=0.48\textwidth]{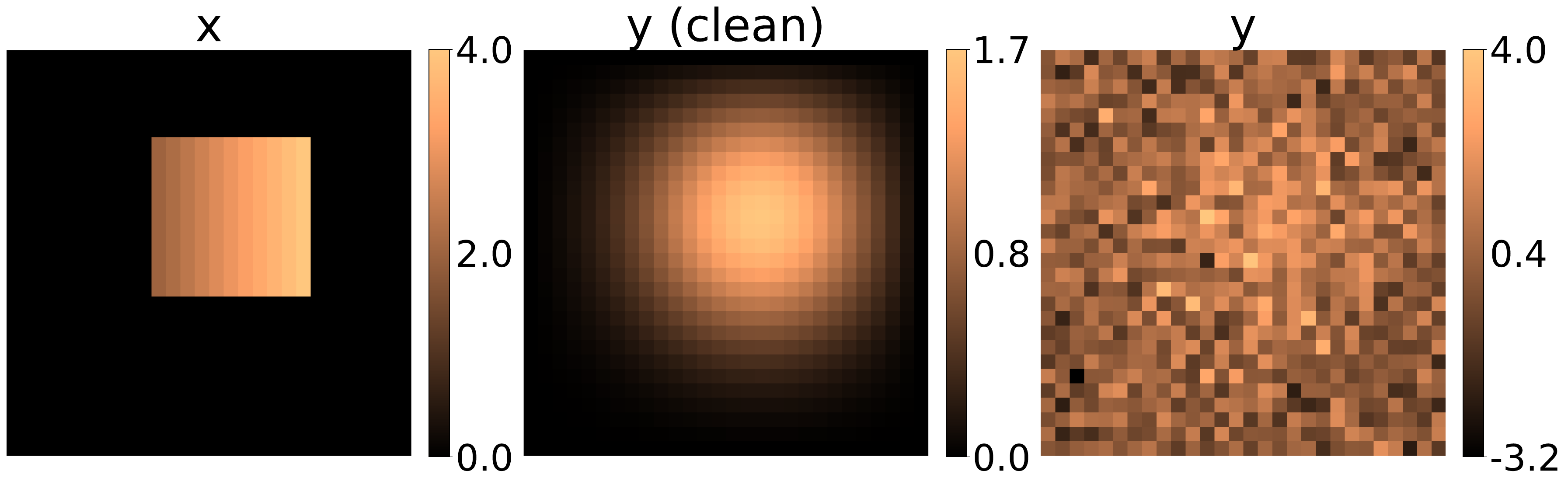}}
\subfigure[Sample 3]{\includegraphics[width=0.48\textwidth]{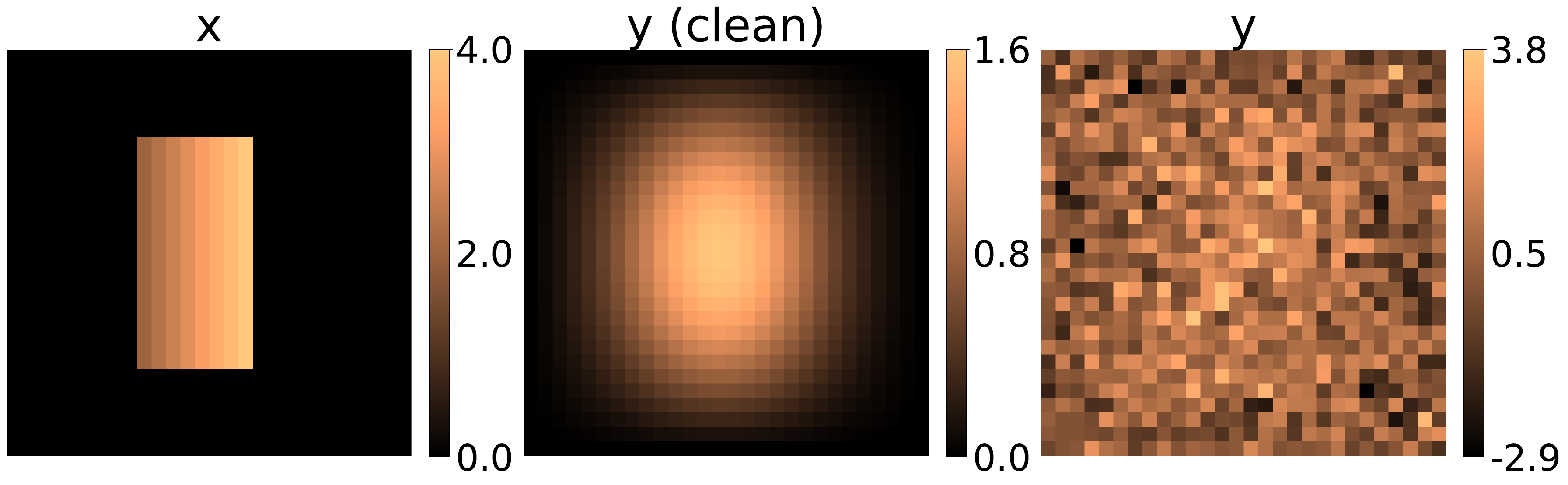}}
\subfigure[Sample 4]{\includegraphics[width=0.48\textwidth]{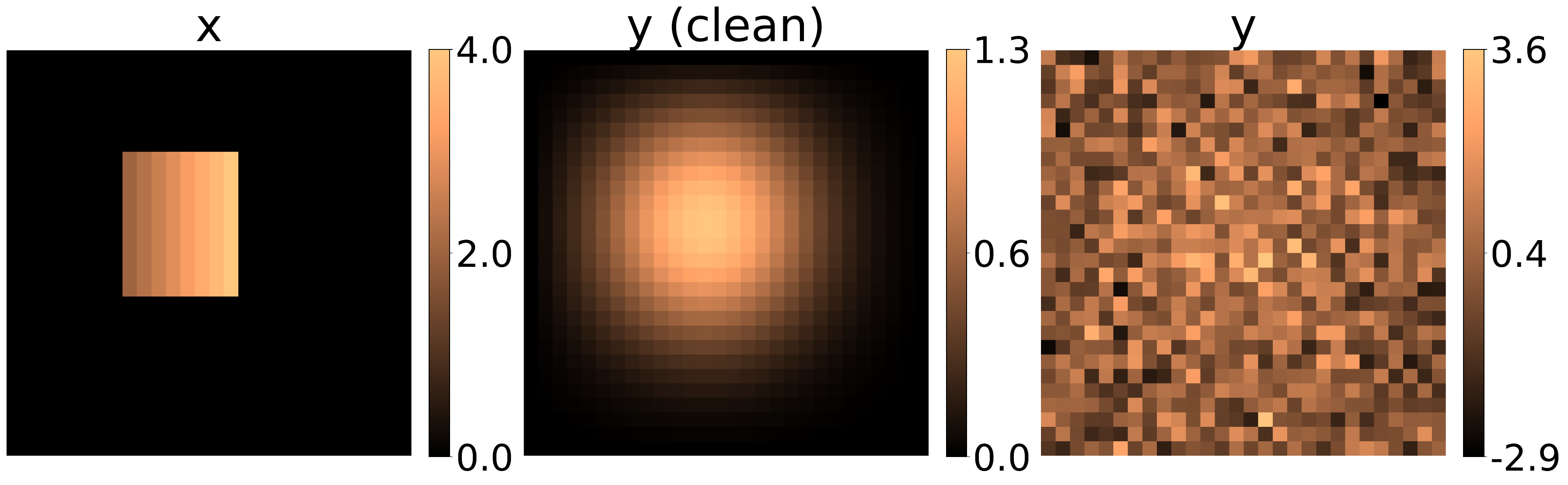}}
\caption{Samples from rectangular dataset used to train the cWGAN. The clean measurements are also shown to contextualize the amount of noise added.}
\label{fig:rect_samples}
\end{figure} 

The cWGAN is trained on a dataset with 10,000 training sample pairs, and tested on a measurement $\y$ not in the training set. The corresponding $\x$, along with reference pixel-wise statistics of posterior (given $\y$) are shown in Figure \ref{fig:ic_ref}. The reference statistics have been computed by performing Monte Carlo sampling directly in the $\boldsymbol{\xi}$ space. Note that this is computationally feasible only because the dimension of the $\boldsymbol{\xi}$ is small (4-dimensional). We observe that the standard deviation (SD) peaks along the edges of the rectangular region indicating that these are the regions of highest uncertainty. 

A priori, it is not clear what $\Nz$ (the dimension of the latent space) should be. Since it represents the dimension of the posterior distribution, it is reasonable to expect $\Nz$ to be much smaller than $\Nx$. We train different cWGANs by varying $\Nz$ and compute the GAN-based statistics using \eqref{eqn:mc} with $K=800$. As can be seen in Figure \ref{fig:rect_mean_SD}, the statistics with all $\Nz$ considered are qualitatively similar to the reference values. In fact, even $\Nz=1$ gives accurate results, which leads to the maximal dimension reduction. We also make a quantitative assessment by considering the the $L^1$ error plotted in Figure \ref{fig:stat_err}. Although we cannot see a trend as $\Nz$ increases, the error appears to be in the narrow range $0.097\pm0.019$ for the mean and $0.089\pm0.014$ for standard deviation.

\begin{figure}[htp]
\begin{center}
\includegraphics[width=0.8\textwidth]{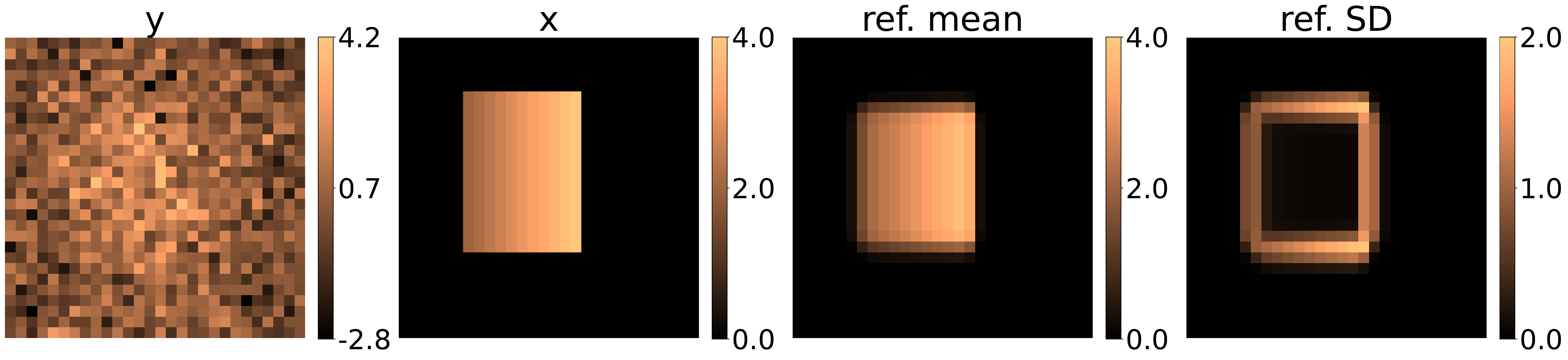}
\caption{Test sample with reference mean and SD.}
\label{fig:ic_ref}
\end{center}
\end{figure}

We remark here that the proposed algorithm can approximate the mean of the posterior, which can be quite different from the ground truth $\x$ used to generate the (clean) $\y$. This difference is quite evident in Figure \ref{fig:ic_ref}. In most realistic scenarios, the $\x$ may not be available for a given noisy measurement. However, we can use techniques to find the most important samples (amongst $K$) generated by $\g$. For instance, we use the reduced rank QR (RRQR) decomposition of the $K$-snapshot matrix to determine the most important samples \cite{chan1992some}. We show the four most important samples (ranked left to right) in Figure \ref{fig:ic_rrqr} for $\Nz=3$. This reduced set of four samples captures a significant amount of the variance observed in the $K = 800$ samples obtained from the generator. We also observe that these samples are quite distinct from each other, with the first sample being closest to the true sample. Note that $\Nz=3$ leads to a dimension compression rate of $\dcomp  = 261$ (see Table \ref{tab:dim_red}) which is a significant reduction. We reiterate that we are able to achieve dimension reduction since the generator architecture allows the prescription of $\Nz$ to be independent of $\Nx$ or $\Ny$.

\begin{figure}[htp]
\centering
\subfigure[Mean]{\includegraphics[width=\textwidth]{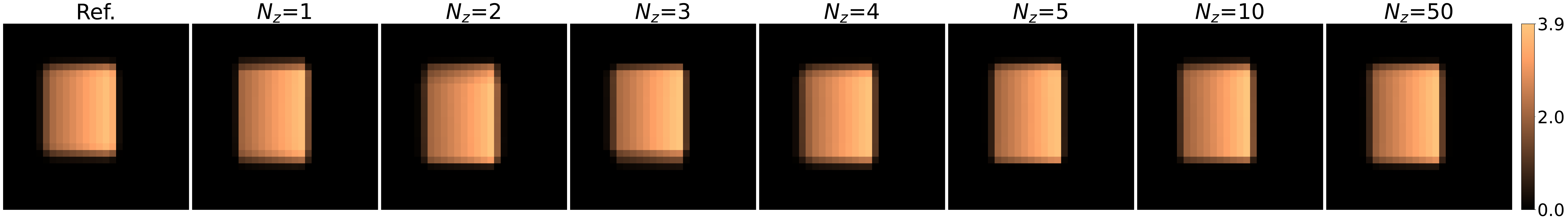}}\\
\subfigure[SD]{\includegraphics[width=\textwidth]{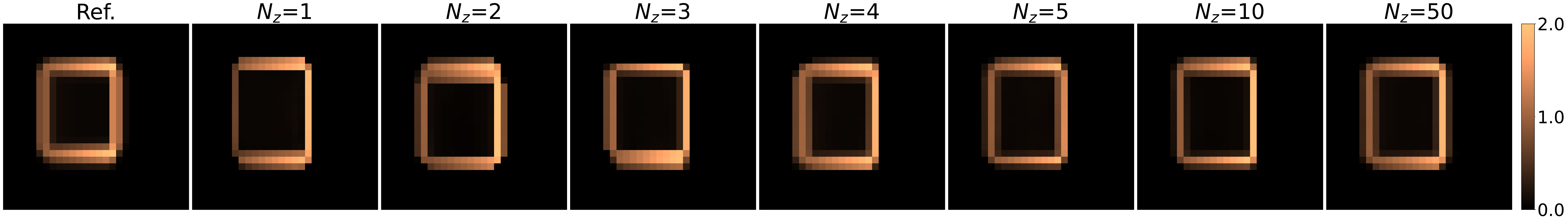}}
\caption{Mean and SD computed with 800 samples of $z$, for cWGANs trained with varying latent space dimension $(\Nz)$.}
\label{fig:rect_mean_SD}
\end{figure}

\begin{figure}[htp]
\begin{center}
\subfigure[Mean]{\includegraphics[width=0.49\textwidth]{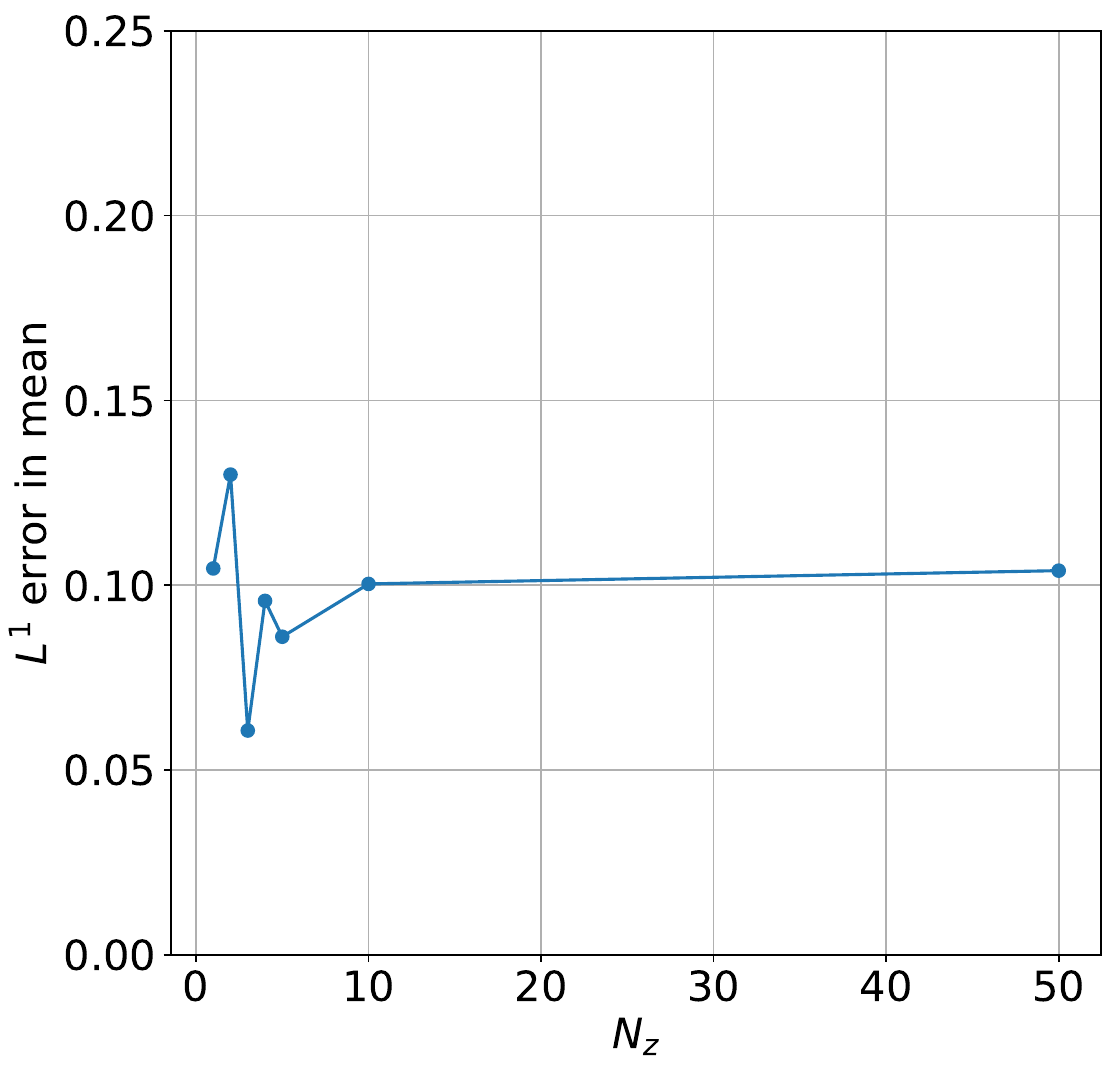}}
\subfigure[SD]{\includegraphics[width=0.49\textwidth]{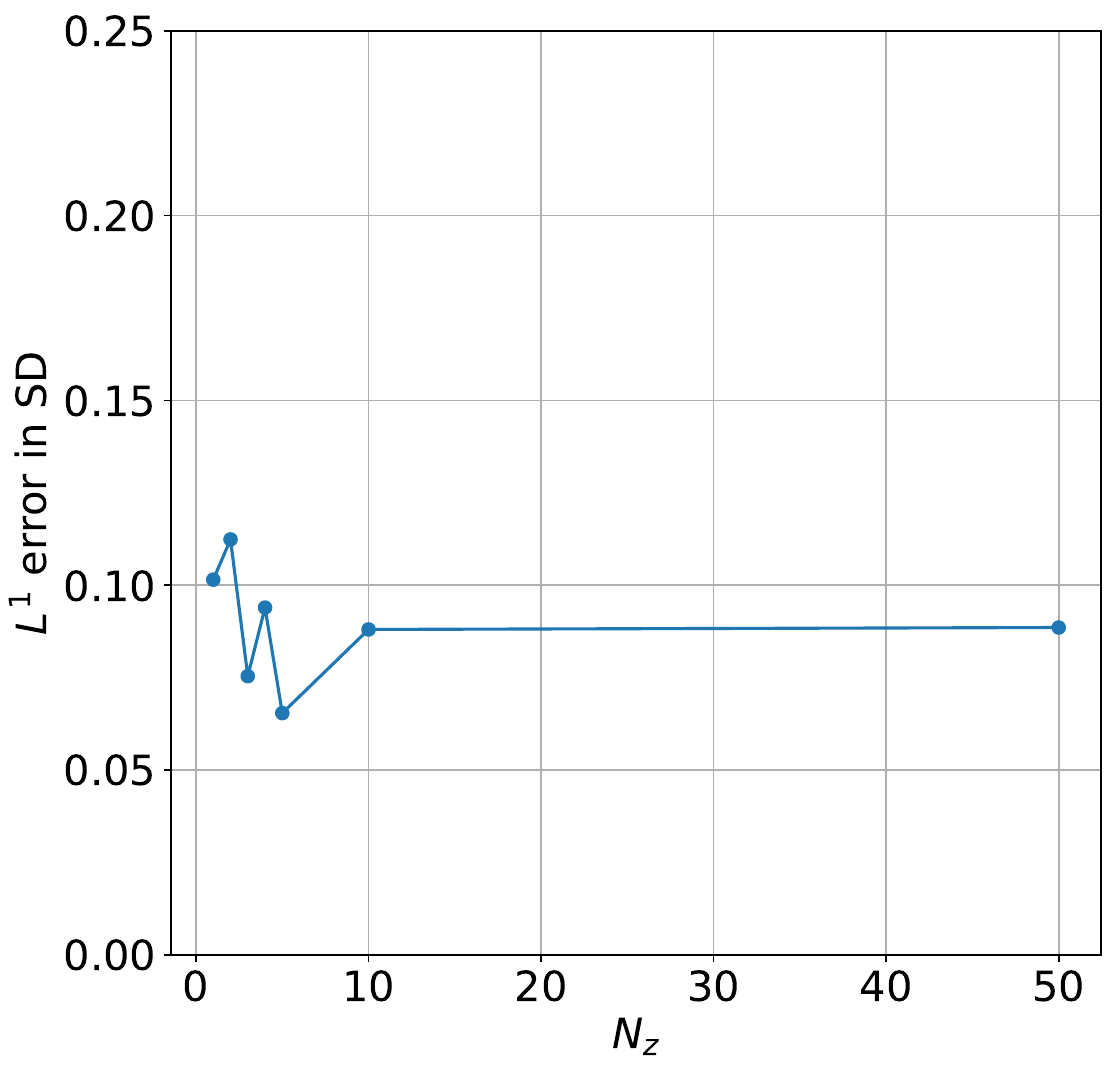}}
\caption{$L_1$ error in computed statistics compared to the reference statistics as a function of the number $\Nz$. }
\label{fig:stat_err}
\end{center}
\end{figure}

\begin{figure}[htp]
\centering
\includegraphics[width=0.8\textwidth]{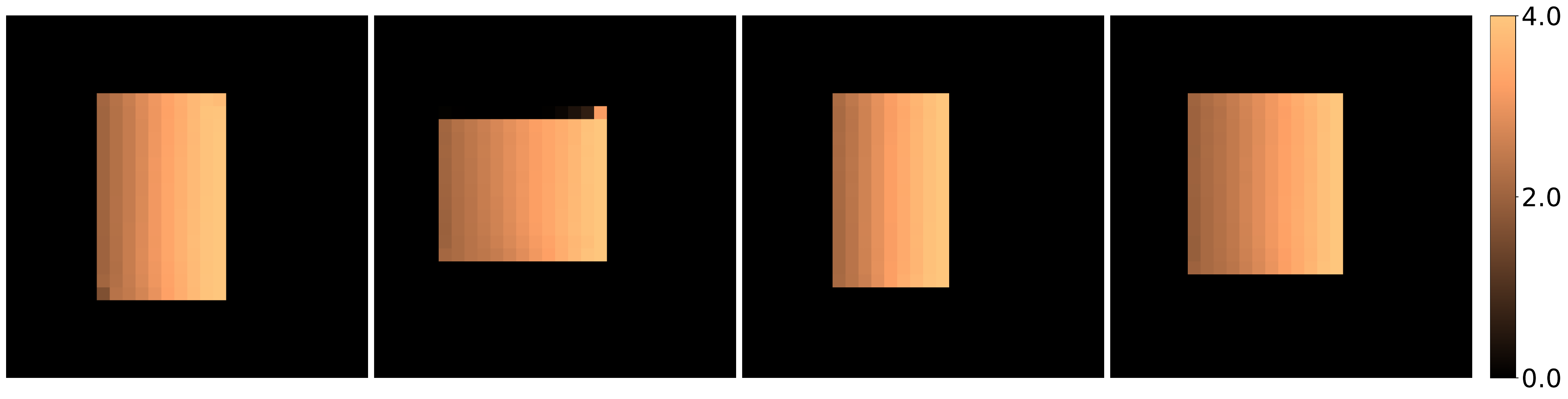}
\caption{Most important samples ranked left to right using RRQR algorithm on 800 samples for $\Nz=3$.}
\label{fig:ic_rrqr}
\end{figure}

\subsubsection{MNIST prior} Next, we assume a more complex prior on $\x$. We use linearly scaled MNIST \cite{lecun2010mnist} handwritten digits as $\x$, so that it has a background value of zero and takes the value $4.0$ on the digit. The final temperature fields are obtained by setting $b=0$ and $\kappa=0.2$ and adding noise sampled from $\mathcal{N}(\boldsymbol{0}, 0.3\boldsymbol{I})$. In Figure \ref{fig:mnist_samples}, we show a few samples used to construct the datasets.

\begin{figure}[htp]
\centering
\subfigure[Sample 1]{\includegraphics[width=0.48\textwidth]{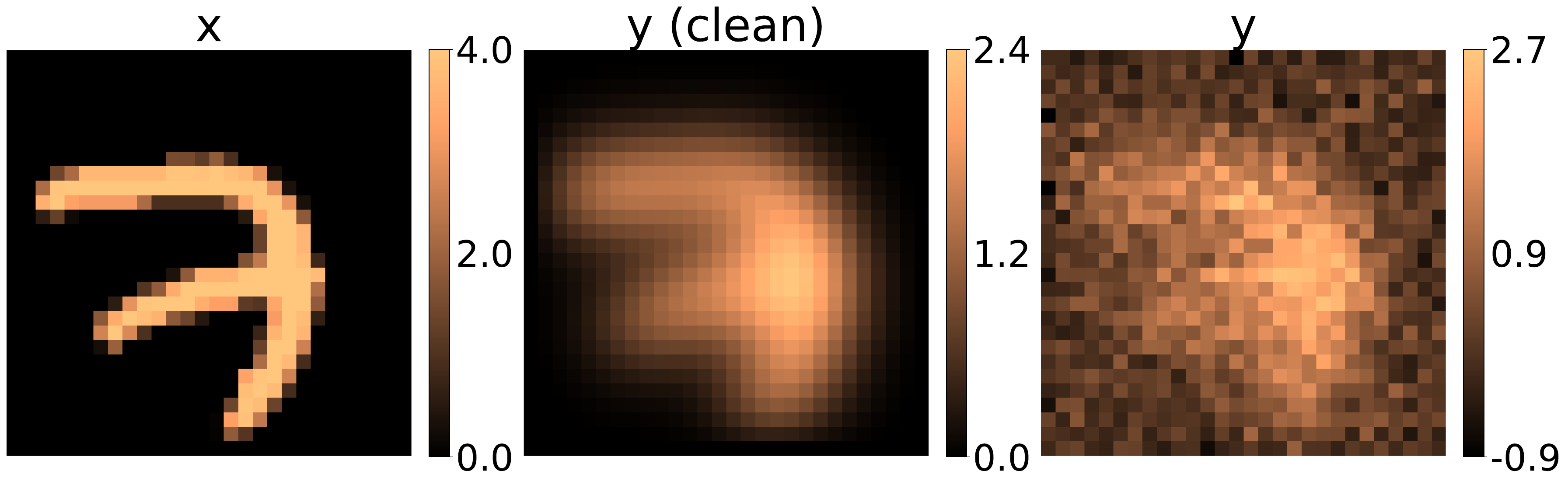}}
\subfigure[Sample 2]{\includegraphics[width=0.48\textwidth]{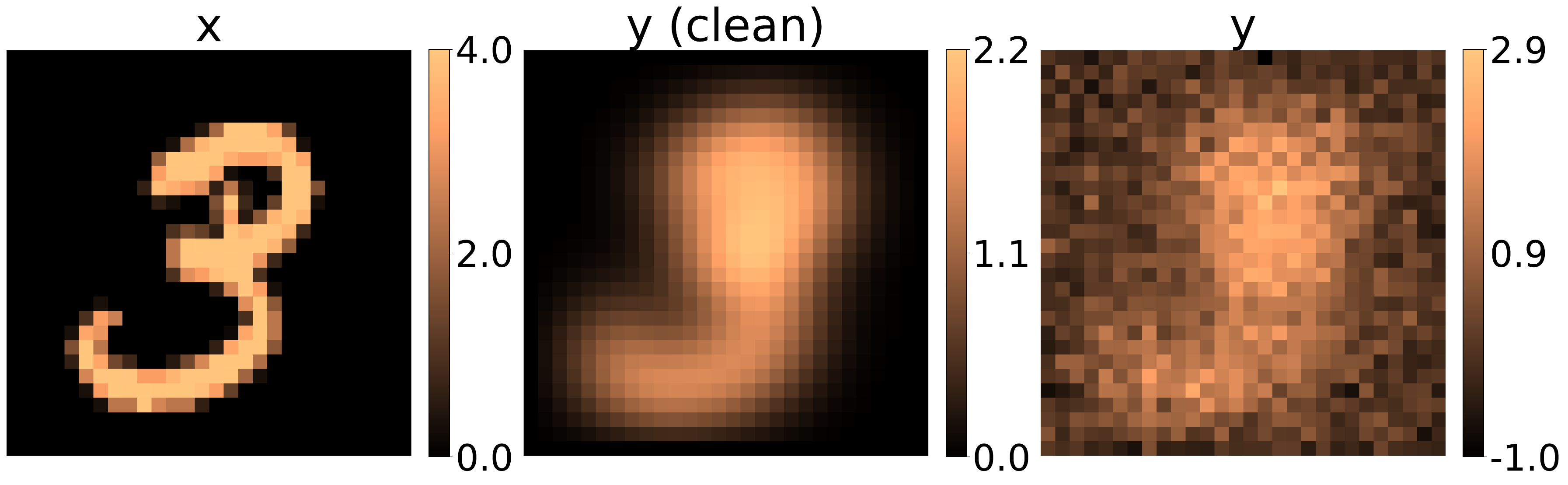}}
\subfigure[Sample 3]{\includegraphics[width=0.48\textwidth]{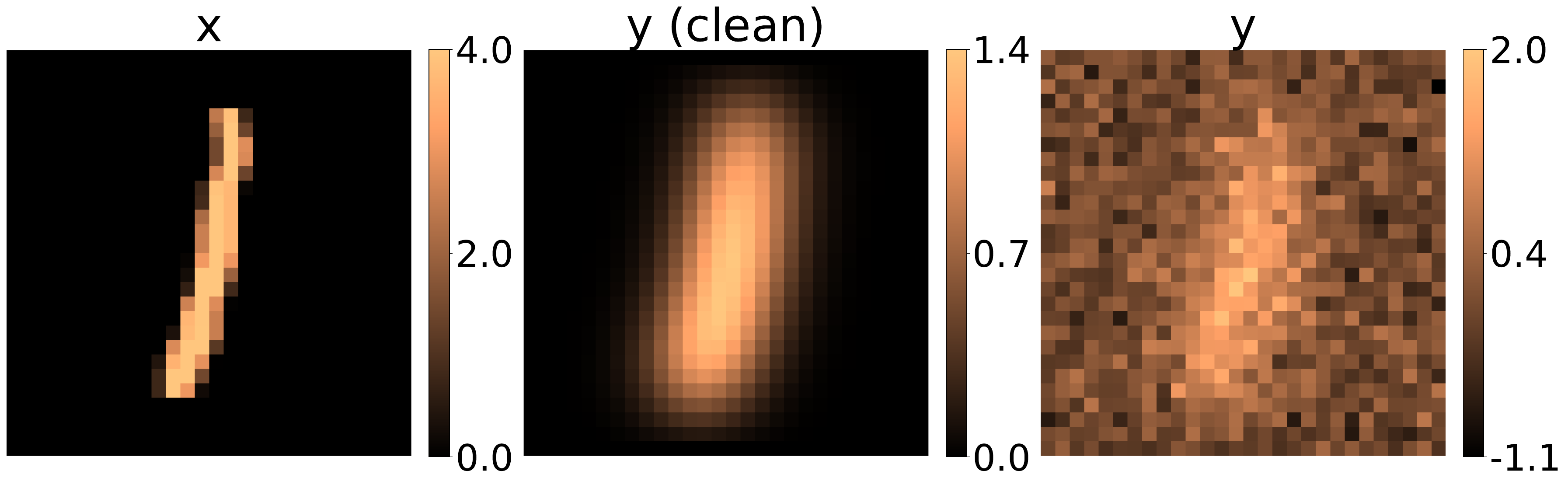}}
\subfigure[Sample 4]{\includegraphics[width=0.48\textwidth]{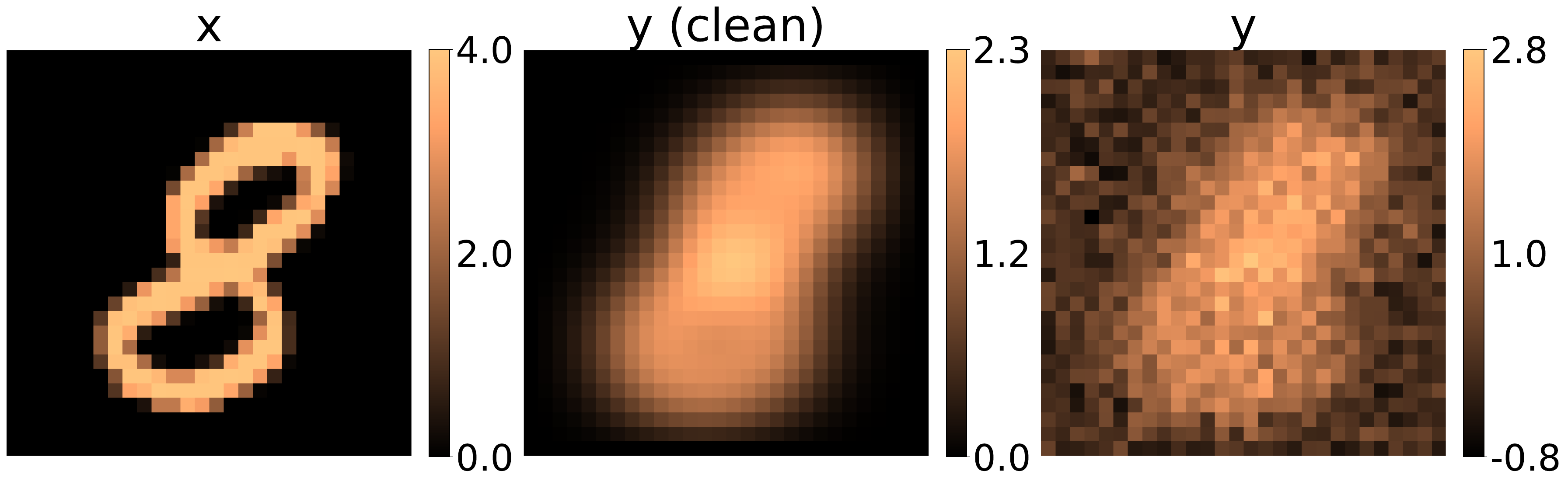}}
\caption{Samples from MNIST dataset used to train the cWGAN. The clean measurements are also shown to contextualize the amount of noise added.}
\label{fig:mnist_samples}
\end{figure}

We train a cWGAN with $\Nz=100$ on a dataset with 10,000 training sample pairs. The trained generator is first used on test samples from the same distribution as the  training set. As can be seen in Figure \ref{fig:mnist_test}, the mean computed using the cWGAN ($K$=800) correctly captures the ground truth $\x$. The SD is the highest at the boundaries of the digits, indicating a higher uncertainty in predicting the sharp transition region of initial temperature field. Furthermore, certain regions where the mean deviates from $\x$ (marked by red rectangles) also have high values of SD. This highlights the utility of computing the SD in that it quantifies the uncertainty in the inference, and points to the regions in the domain where the uncertainty is high and the inference less trustworthy.

\begin{figure}[htp]
\centering
\subfigure[Sample 1]{\includegraphics[width=0.8\textwidth]{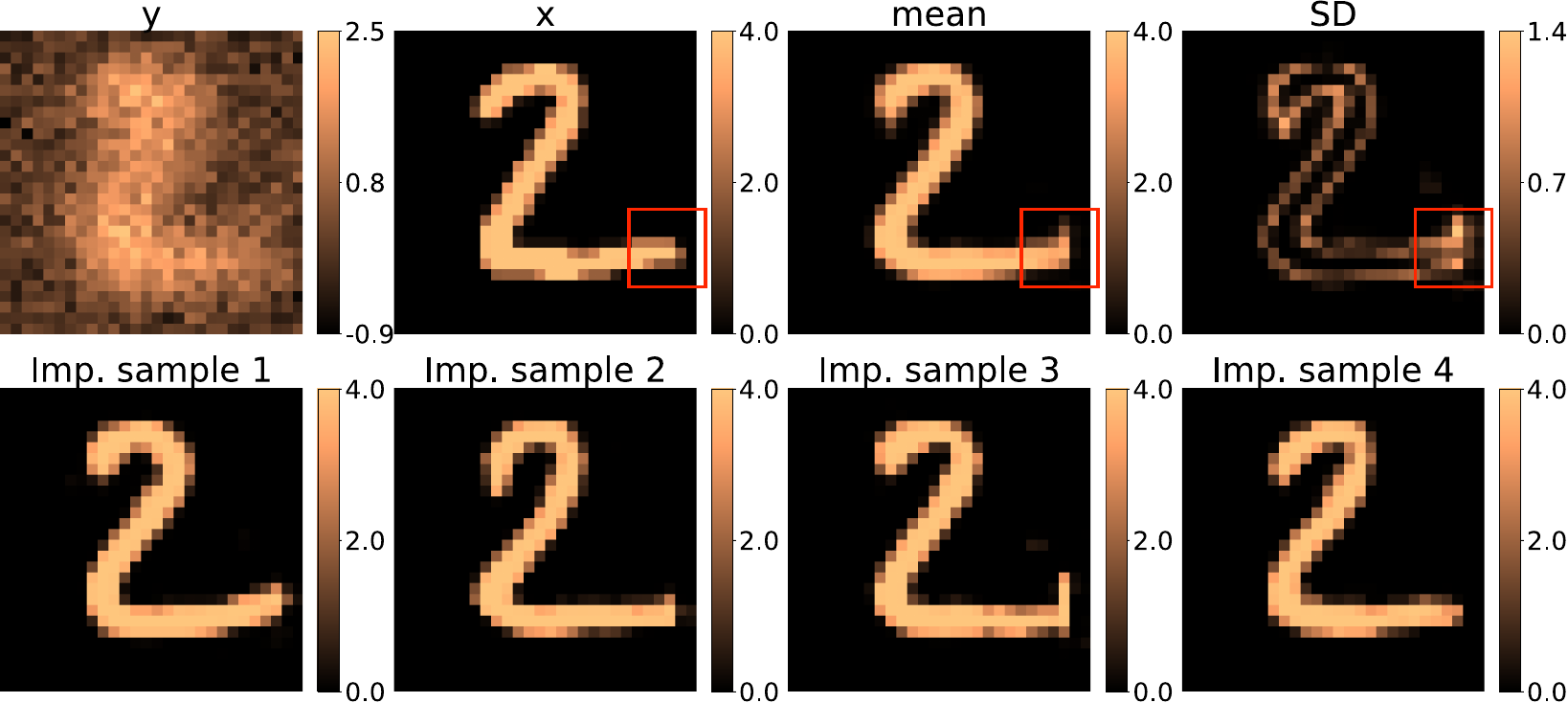}}
\subfigure[Sample 2]{\includegraphics[width=0.8\textwidth]{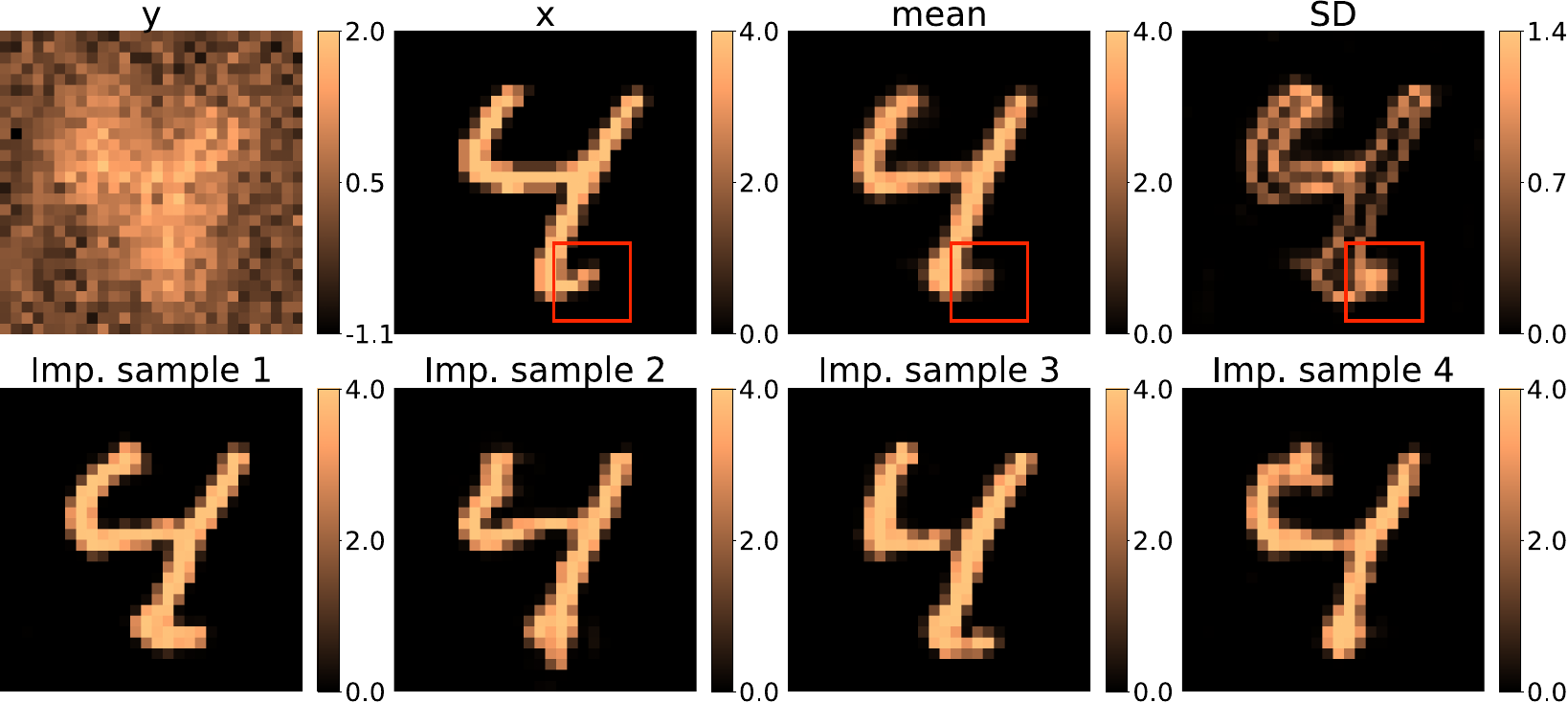}}
\subfigure[Sample 3]{\includegraphics[width=0.8\textwidth]{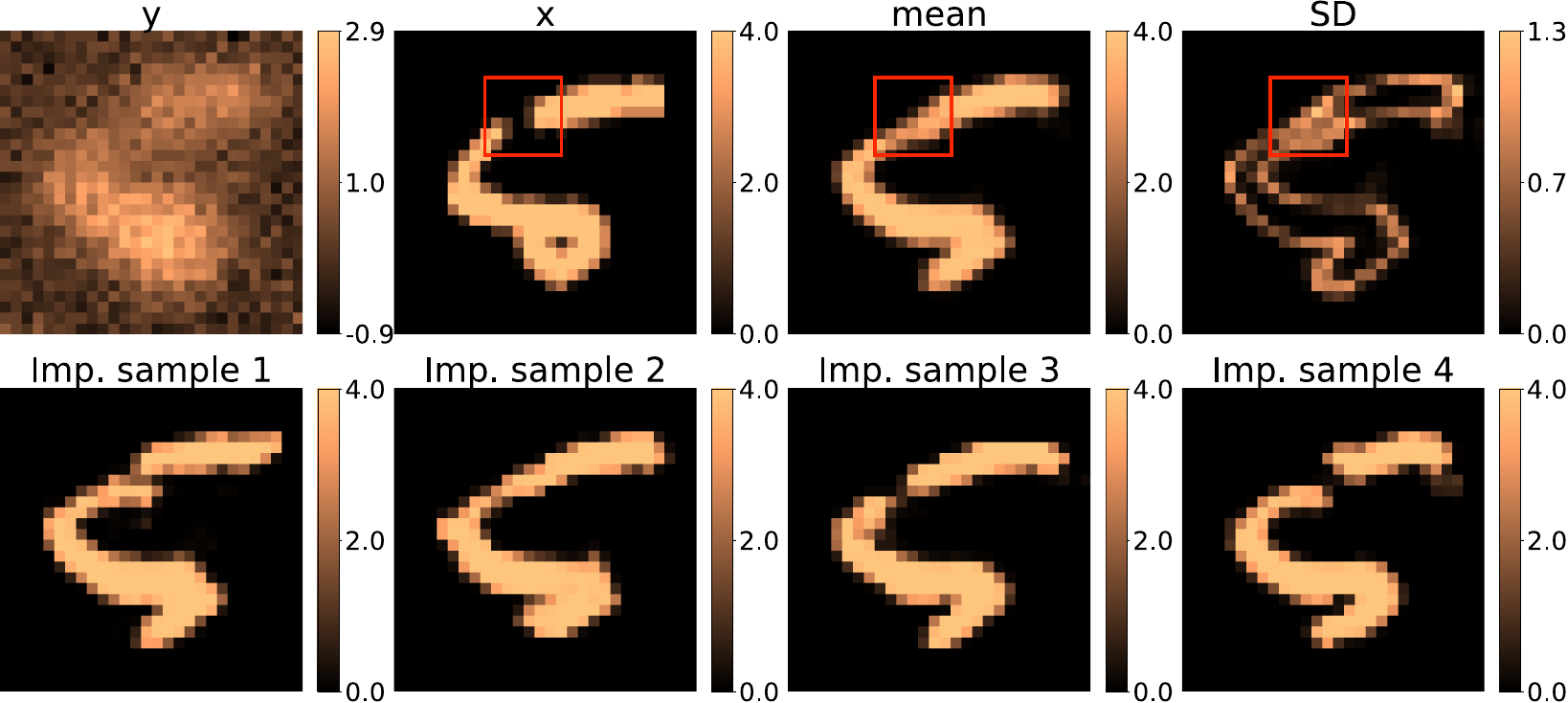}}
\caption{Inferring initial condition for test samples chosen from the same distribution as the training set (MNIST prior).}
\label{fig:mnist_test}
\end{figure}

To demonstrate the benefits of injecting the latent information using CIN, we also train a cWGAN with a stacked $\z$, which we term as cWGAN-stacked to distinguish from the proposed cWGAN model. The critic architecture of cWGAN-stacked is identical to our cWGAN. On the other hand, the generator has the following differences:
\begin{enumerate}
    \item Instead of injecting $\z$ at every level of the U-Net using CIN, we take $\z$ to have the same shape as the measurement $\y$, i.e. $28 \times 28$, and feed both $\y$ and $\z$ to generator as a stacked input $[\y,\z] \in \Ro^{28 \times 28 \times 2}$.
    This is identical to the strategy use in \cite{adler2018deep}.
    \item All the CIN operations are replaced by batch normalization.
\end{enumerate}
The remaining hyper-parameters of cWGAN-stacked are set to be the same as those taken for our cWGAN model. Note that feeding $\z$ as a stacked input restricts the latent dimension $N_z$ to be the same as that of $\y$, which corresponds to the mesh size for the PDE-based problems considered in this work. Thus, $\Nz$ will scale with the size of the mesh.

We show the predictions with the trained cWGAN-stacked on test samples in Figure \ref{fig:mnist_stacked}, where the statistics are still computed with $K=800$. Note that the quality of the recovered mean is quite poor as compared to the ground truth $\x$, as well as the mean obtained with our cWGAN using CIN. Further, the inter-sample variations captured by the SD maps is much lower with cWGAN-stacked, which was also observed in \cite{adler2018deep}. The authors in \cite{adler2018deep} attributed this lack of variation to the inability of the generator to detect the randomness in the latent variable $\z$, and designed a specialized critic to overcome the issue. This is not required for our cWGAN where CIN injects sufficient stochasticity at all levels of the generator to ensure sufficient inter-sample variations.

\begin{figure}[htp]
\centering
\subfigure[Sample 1]{\includegraphics[width=0.4\textwidth]{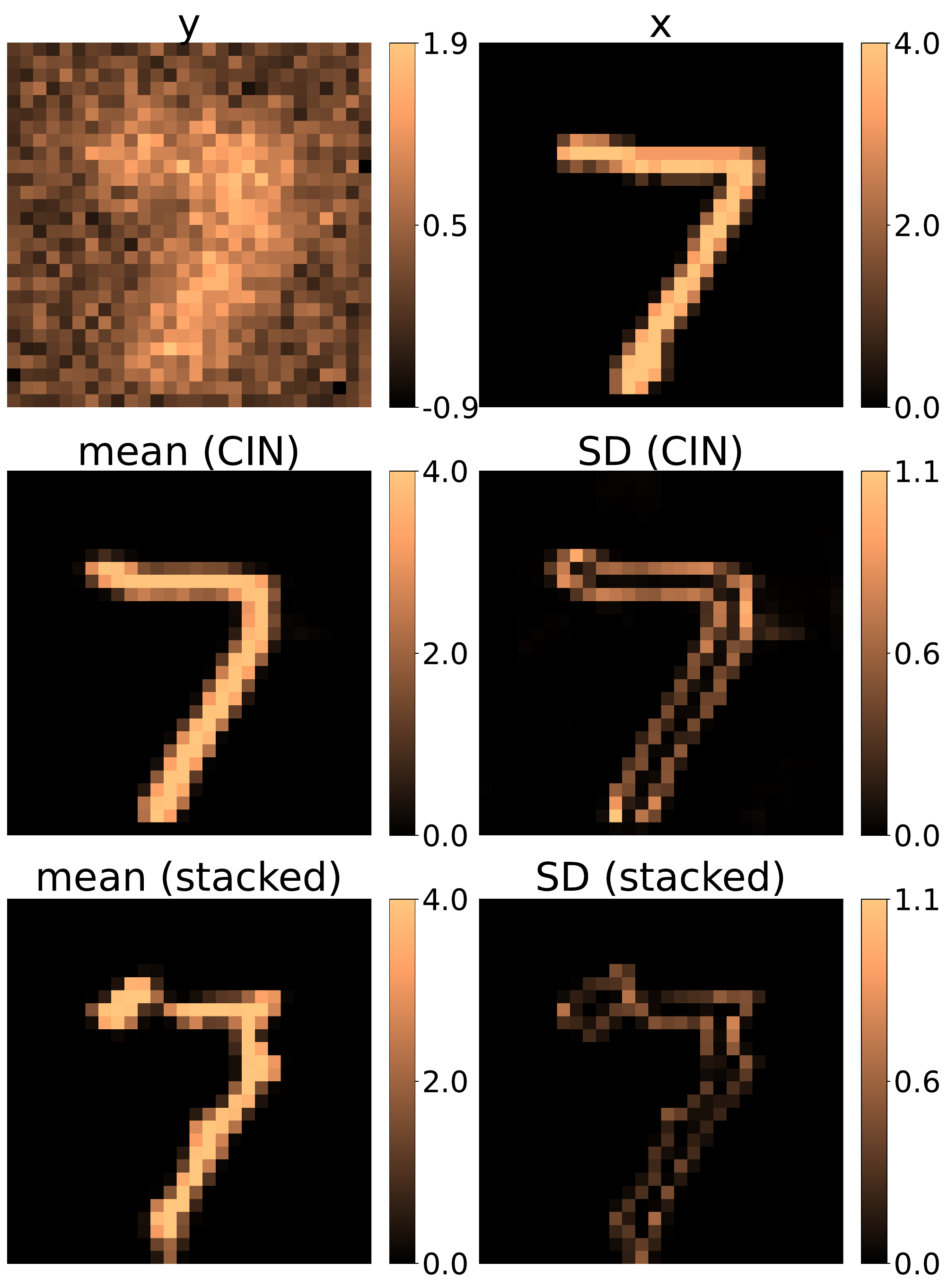}}
\subfigure[Sample 2]{\includegraphics[width=0.4\textwidth]{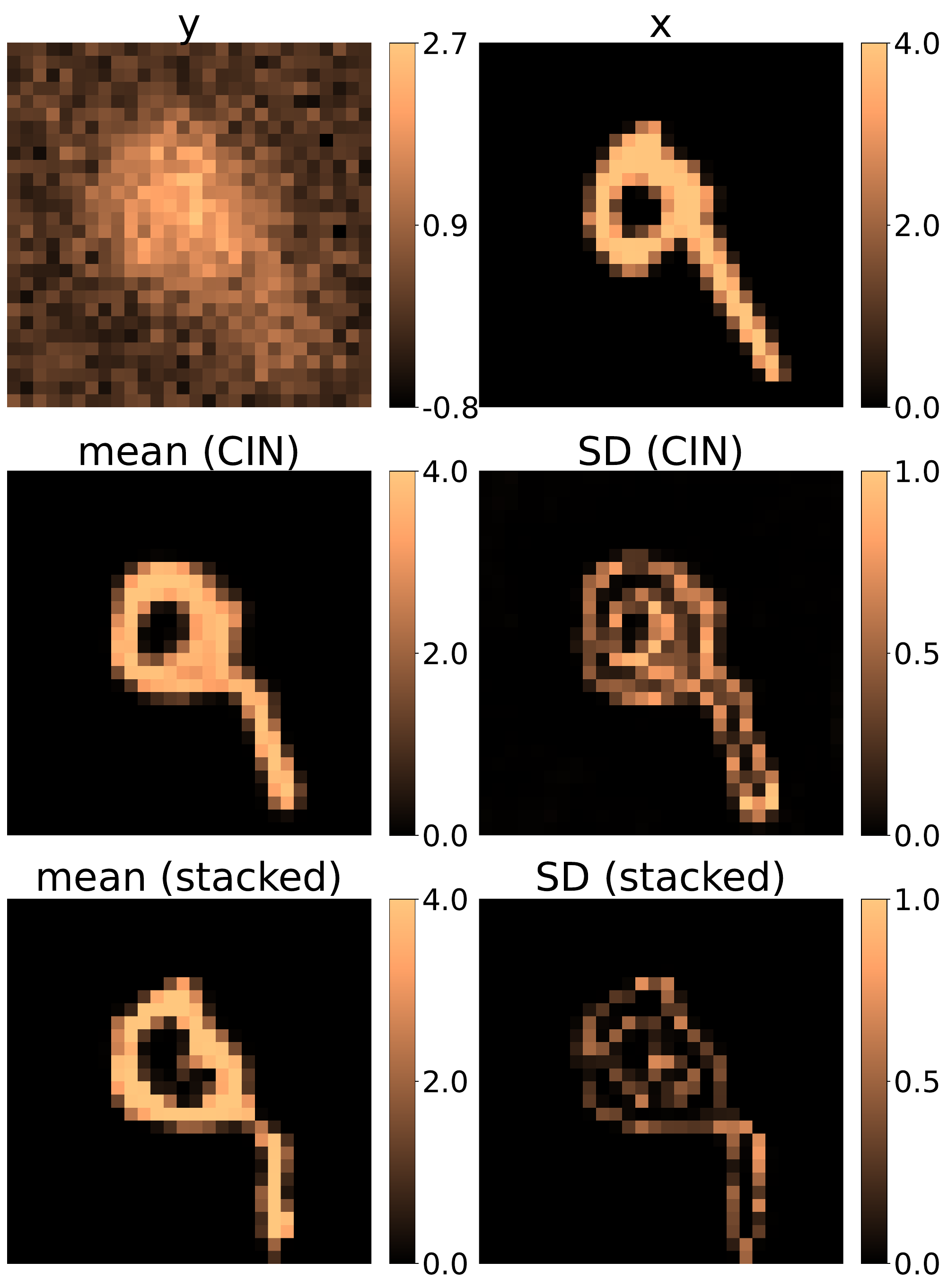}}
\caption{Comparing cWGAN and cWGAN-stacked for inferring initial condition (MNIST prior).}
\label{fig:mnist_stacked}
\end{figure}

\begin{remark}
Although the cWGAN-stacked was trained for 1000 epochs like our cWGAN, we observed the results generated by the stacked model at the end of the training were severely smeared. Thus, we present the results with the cWGAN-stacked recovered at the end of 900 epochs. The results indicate that the stacked network can suffer from training instabilities, which was not observed for our cWGAN. This phenomenon requires further investigation and will be considered in future work.  
\end{remark}

Next, we use the trained generator of the cWGAN on OOD test samples where the $\x$ is chosen from the linearly scaled notMNIST dataset\footnote{Available at: http://yaroslavvb.blogspot.com/2011/09/notmnist-dataset.html}. We observe from the results shown in Figure \ref{fig:notmnist_test}, that the generator struggles to capture high-temperature regions closer to the boundary and the broader zones in the interior. Note that MNIST digits in the training set are more spatially centered and have narrower features. Nonetheless, the generator is able to visibly reconstruct the underlying notMINST characters. This indicates the trained cWGAN is able to generalize reasonably well to OOD samples. Based on the discussion following Theorem \ref{thm:gen2}, we conclude that this surprising performance of the cWGAN can be attributed to the fact that (a) the true and learned inverse maps are local (shown below), and (b) the local spatial features of the OOD and in-distribution datasets are similar (sharp variations in the initial condition field across smooth spatial curves), even though the global features are very different.

\begin{figure}[htp]
\centering
\subfigure[Sample 1]{\includegraphics[width=0.8\textwidth]{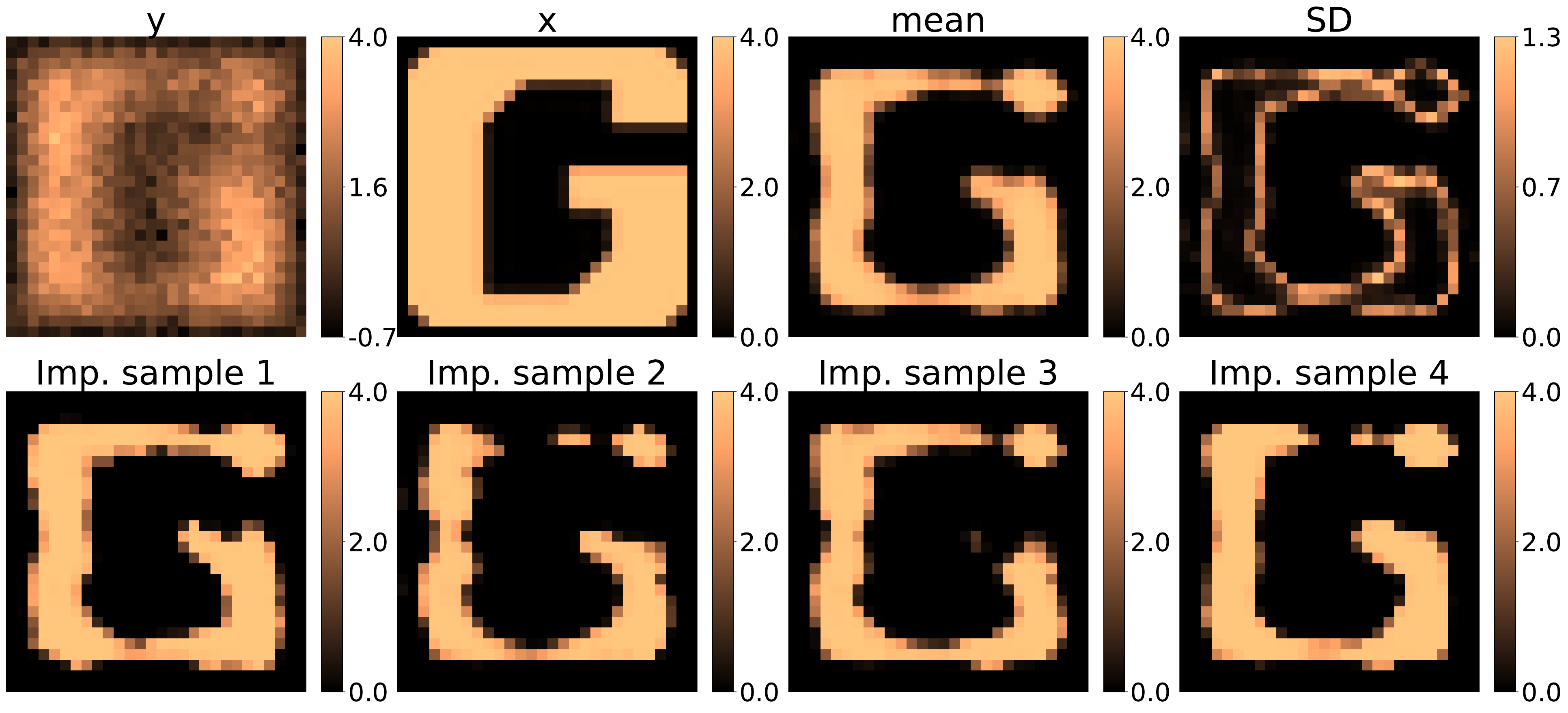}}
\subfigure[Sample 2]{\includegraphics[width=0.8\textwidth]{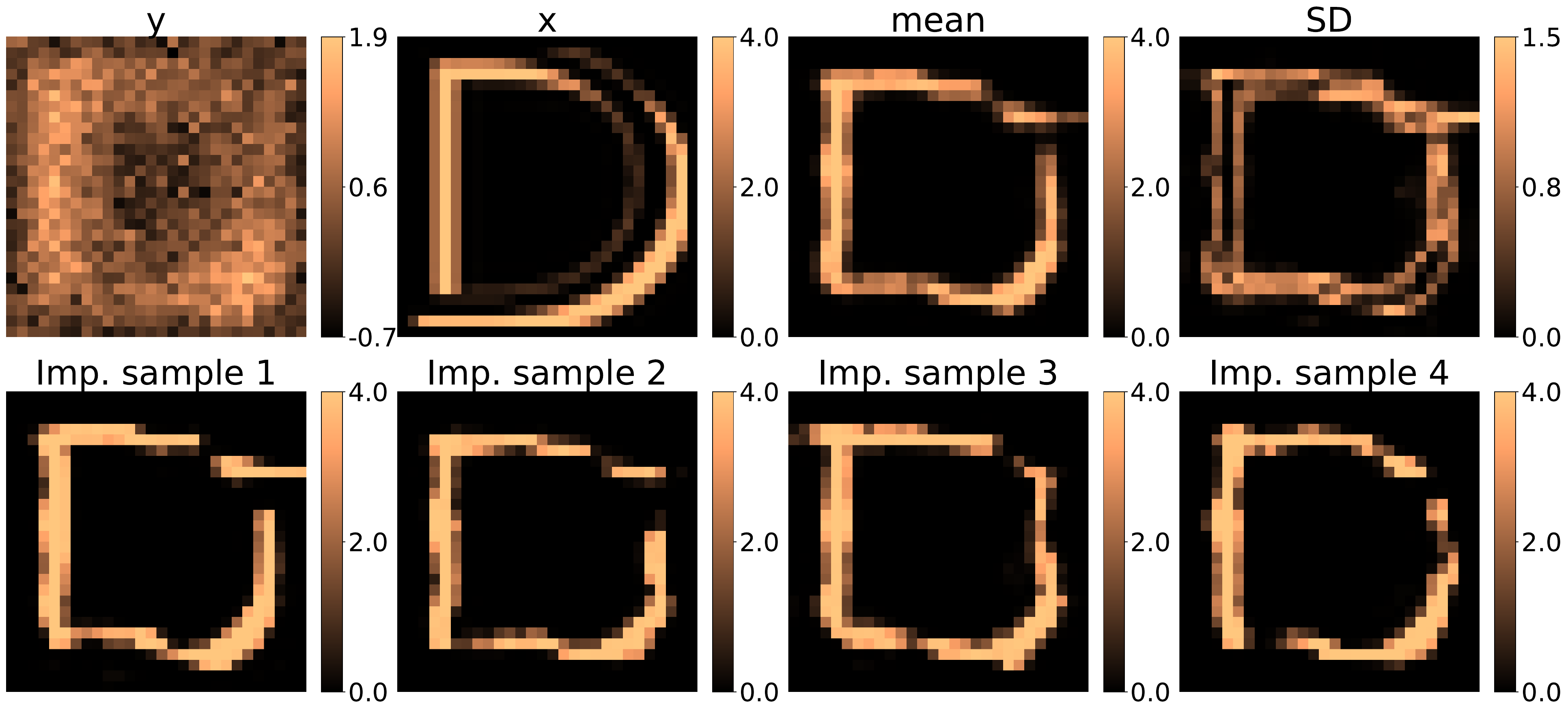}}
\subfigure[Sample 3]{\includegraphics[width=0.8\textwidth]{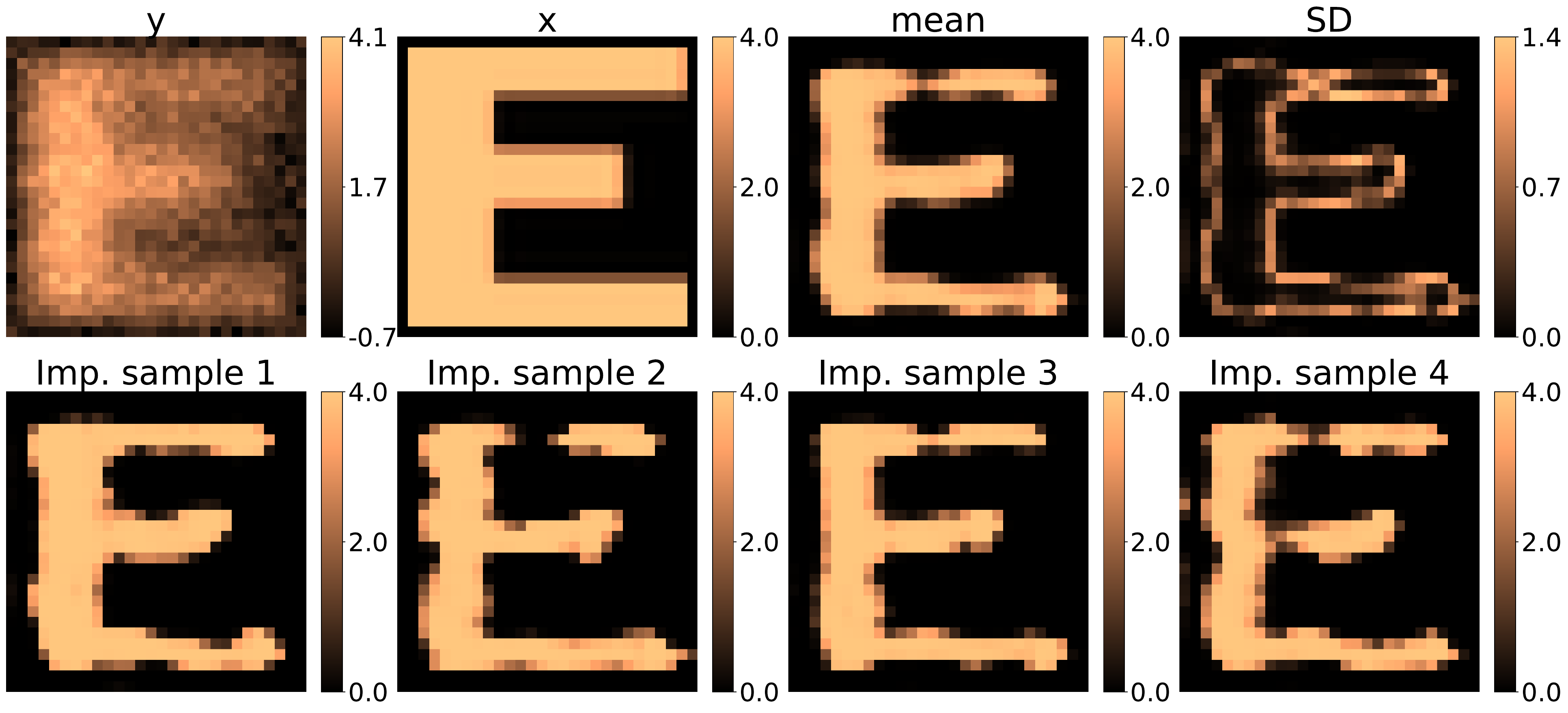}}
\caption{Inferring initial condition for OOD test samples (notMNIST prior).}
\label{fig:notmnist_test}
\end{figure}

To demonstrate the locality of the true inverse map, we assume the solution at time $T=1.0$ on the domain  $[0,2\pi]^2$ to be given by a Gaussian bump
\begin{equation}\label{eqn:bump}
u(\s,T) = \frac{1}{\sqrt{2 \pi} \sigma} \exp{\left( -\frac{|\s - \s^0|^2}{2 \sigma^2}\right)}, \quad \sigma = 0.7,
\end{equation}
centered at $\s^0 \in [0,2 \pi]^2$. We then solve the inverse problem (with $\kappa=0.2$) using an FFT algorithm to obtain the solution at $t=0$. Since the inverse heat equation is highly ill-posed, we solve a regularized version of the inverse problem, which involves truncating the high frequency modes (hyper-diffusion) before the evolving (backwards-in-time) the FFT of $u(\s,T)$. Further, to avoid instabilities when the bump \eqref{eqn:bump} is centered close to the domain boundary, we solve the inverse problem on an extended domain $[-2 \pi,4 \pi]^2$ with a uniform discretization that ensures the target domain has $28 \times 28$ nodes. We found that retaining the first 25 modes (on the extended domain) was sufficient to solve the problem. The solution profiles at $t=1.0$ and $t=0.0$ for $\s^0 = (0,0)$ on the target domain are shown in Figure \ref{fig:inverse_ic_FFT}. As expected we observe that the Gaussian bump ``tightens'' as we march backwards in time. To see the local dependence of $u(\s,0)$ on $u(\s,T)$, we move the bump by considering all $28 \times 28$ locations of $\s^0$ and visualize $u(\s,0)$ at fixed locations $\s= \s^1$. This is shown in Figure \ref{fig:inverse_ic_loc}, where the red markers indicate the spatial location $\s^1$ being considered. Clearly, the $u(\s,T)$ has a weaker influence on $u(\s,0)$ as the center of the bump $\s^0$ moves further away from $\s^1$. In other words, the regularized inverse map is local in nature.

\begin{figure}[htp]
\centering
\subfigure[$u(\s,T)$ ]{\includegraphics[width=0.3\textwidth]{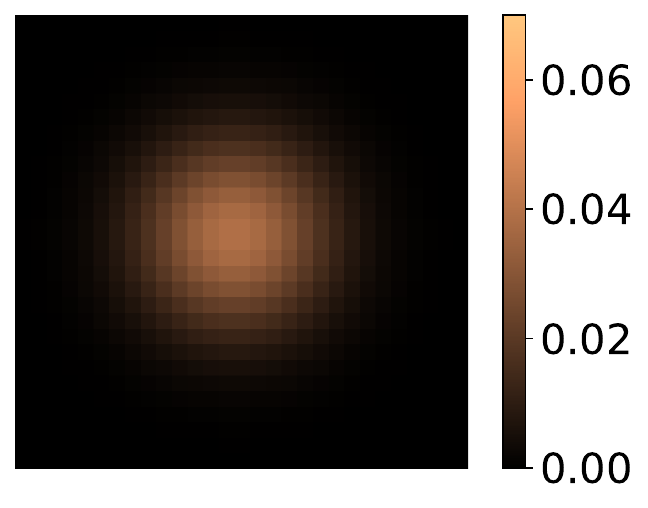}}
\subfigure[$u(\s,0)$ with FFT]{\includegraphics[width=0.3\textwidth]{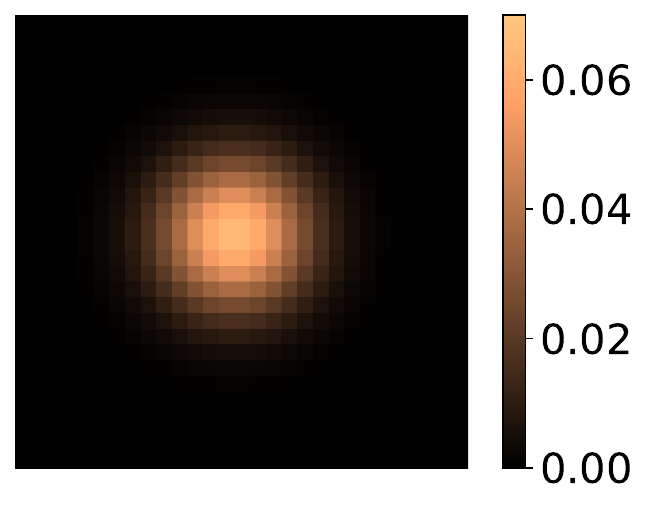}}
\caption{The profiles of $u(T)$ and the corresponding $u(0)$ obtained with an FFT algorithm, for $\s^0 = (0,0)$.}
\label{fig:inverse_ic_FFT}
\end{figure}

\begin{figure}[htp]
\centering
\subfigure[Location 1 ]{\includegraphics[width=0.3\textwidth]{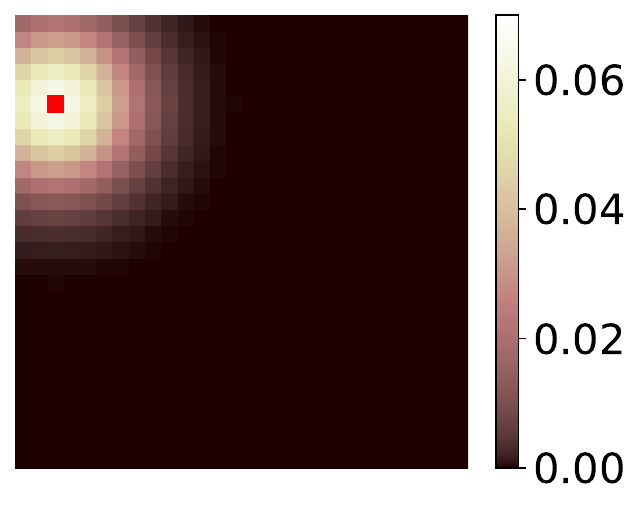}}
\subfigure[Location 2 ]{\includegraphics[width=0.3\textwidth]{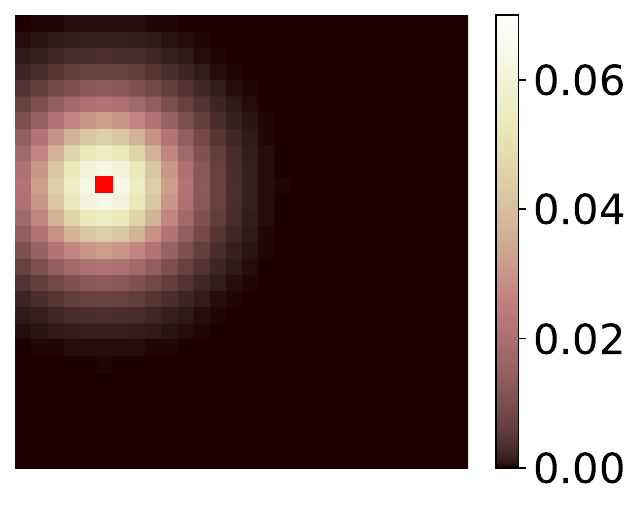}}
\subfigure[Location 3 ]{\includegraphics[width=0.3\textwidth]{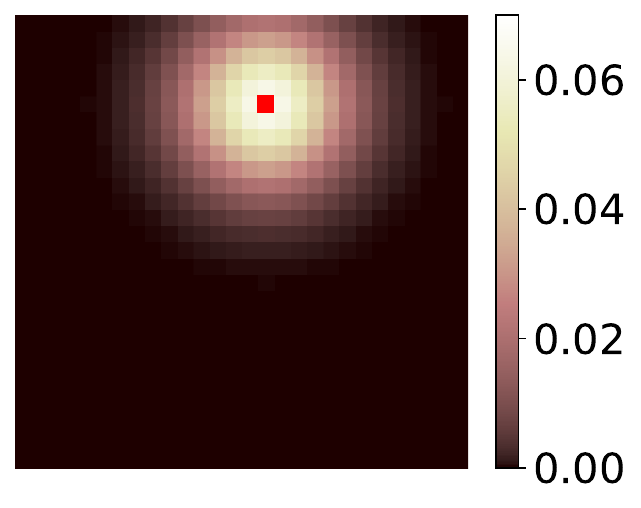}}
\subfigure[Location 4 ]{\includegraphics[width=0.3\textwidth]{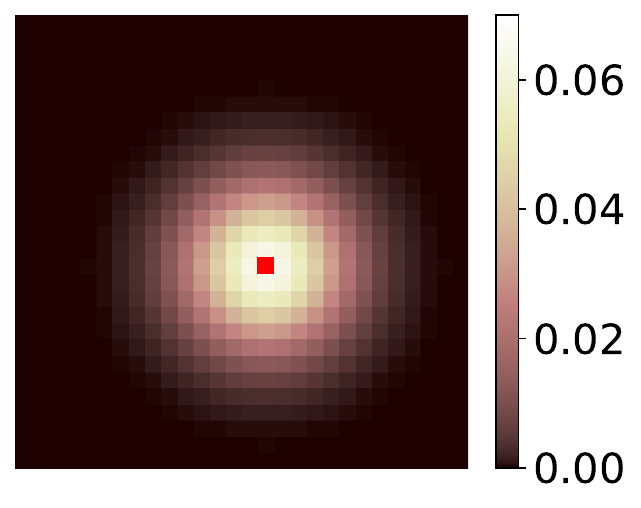}}
\subfigure[Location 5 ]{\includegraphics[width=0.3\textwidth]{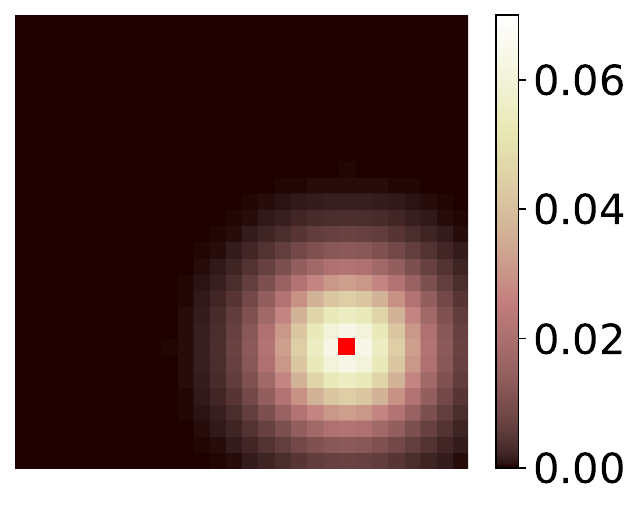}}
\subfigure[Location 6 ]{\includegraphics[width=0.3\textwidth]{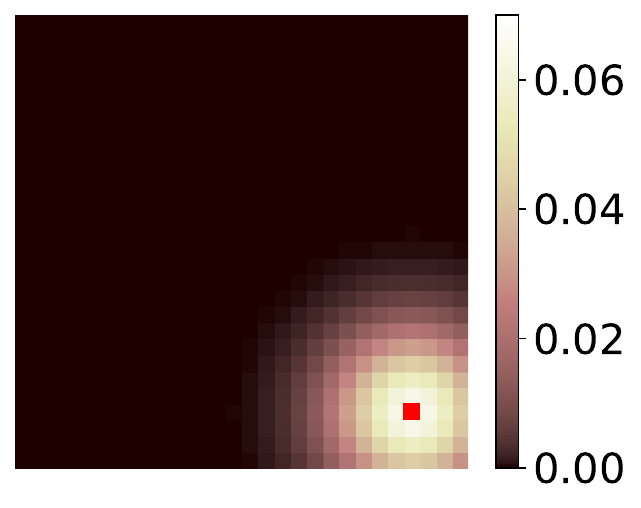}}
\caption{The value of $u(\s^1,T)$ at select locations/components $\s^1$ (marked in red) as $\s^0$ varies in $[0,2 \pi]^2$.}
\label{fig:inverse_ic_loc}
\end{figure}

Next, we demonstrate the the inverse map learned by the generator is also local. We consider the trained $\g$ and compute the gradient of the $k$-th component of the prediction with respect to the network input $\y$. We evaluate the magnitude of gradient and average over 100 distinct samples of $\y$ and 10 realizations of $\z$, i.e.,
\begin{equation}\label{eqn:grad}
\overline{\text{grad}}_k = \frac{1}{1000}\sum_{i=1}^{100} \sum_{j=1}^{10} \left|\frac{\partial g_k}{\partial \y}(\z^{(j)},\y^{(i)})\right|, \ \y^{(i)} \sim \py, \quad \z^{(j)} \sim \pz, \ 1 \leq k \leq \Nx.
\end{equation}
The averaged gradients for a few components of $\g$ are shown in Figure \ref{fig:grad_ic}.
Note that the gradient for each component is concentrated in the neighbourhood of the corresponding component in $\y$. In other words, the domain of influence of $k$-th component of $\g$'s output is a neighbourhood of the $k$-th component of $\y$, but not all of $\y$. The local nature of $\g$ is not unexpected, as most of the operations in the U-Net architecture are local. Further, the non-local conditional instance normalization in the intermediate layers does not seem to substantially alter the local influence $\y$ on the prediction. 

\begin{figure}[htp]
\centering
\subfigure[Location 1]{\includegraphics[width=0.3\textwidth]{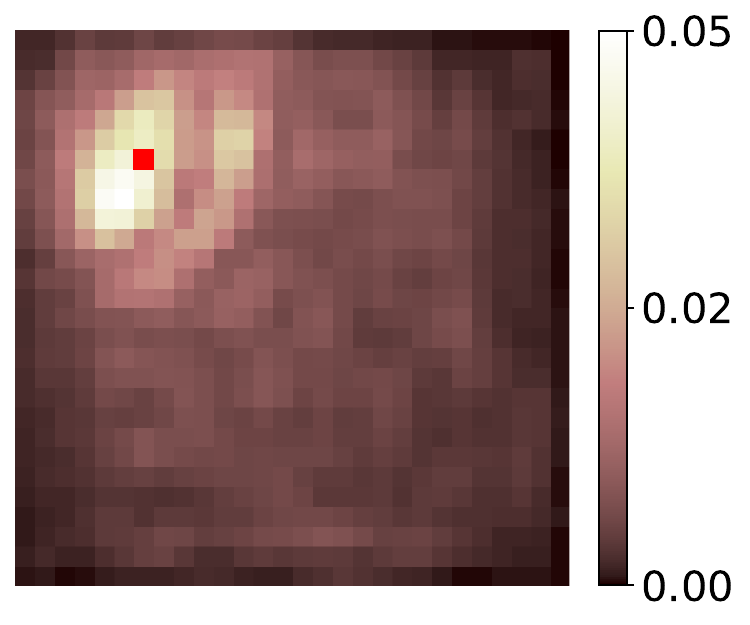}}
\subfigure[Location 2]{\includegraphics[width=0.3\textwidth]{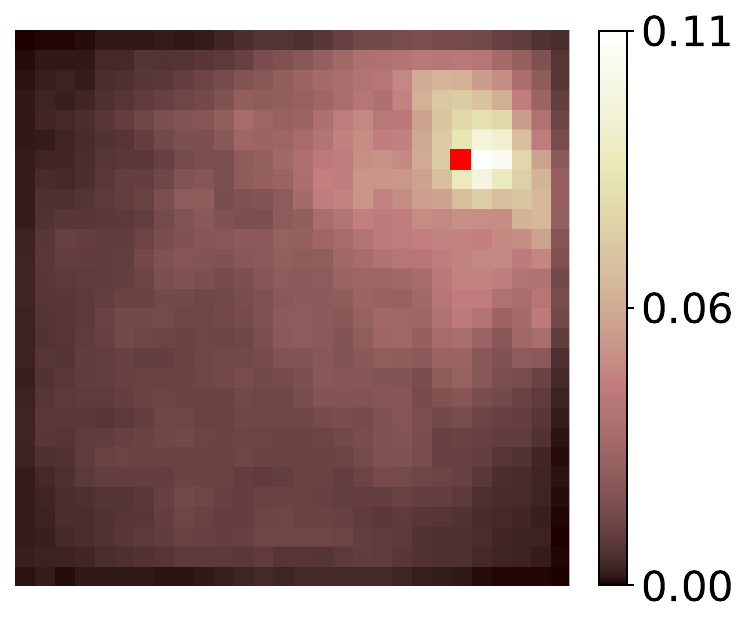}}
\subfigure[Location 3]{\includegraphics[width=0.3\textwidth]{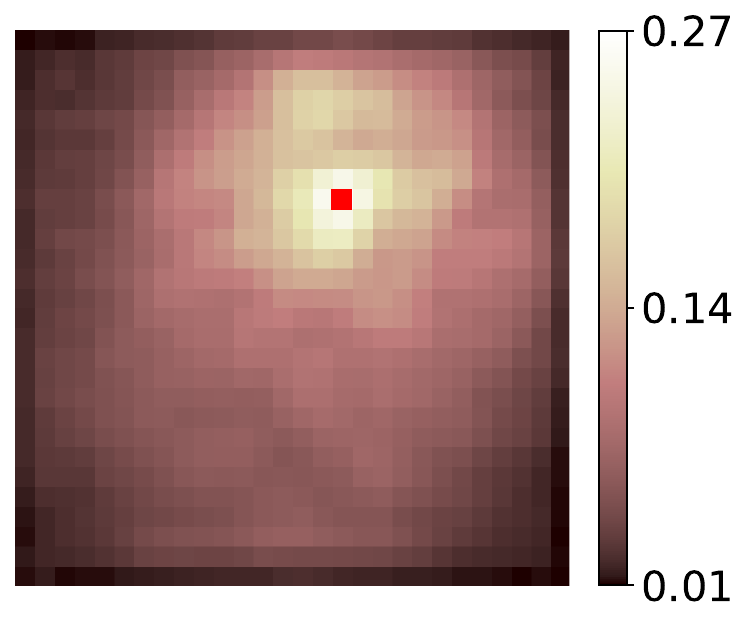}}\\
\subfigure[Location 4]{\includegraphics[width=0.3\textwidth]{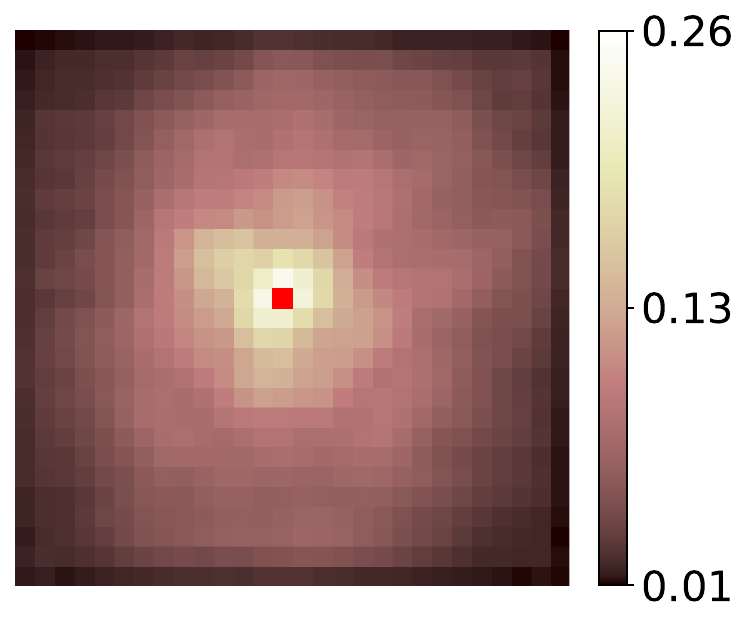}}
\subfigure[Location 5]{\includegraphics[width=0.3\textwidth]{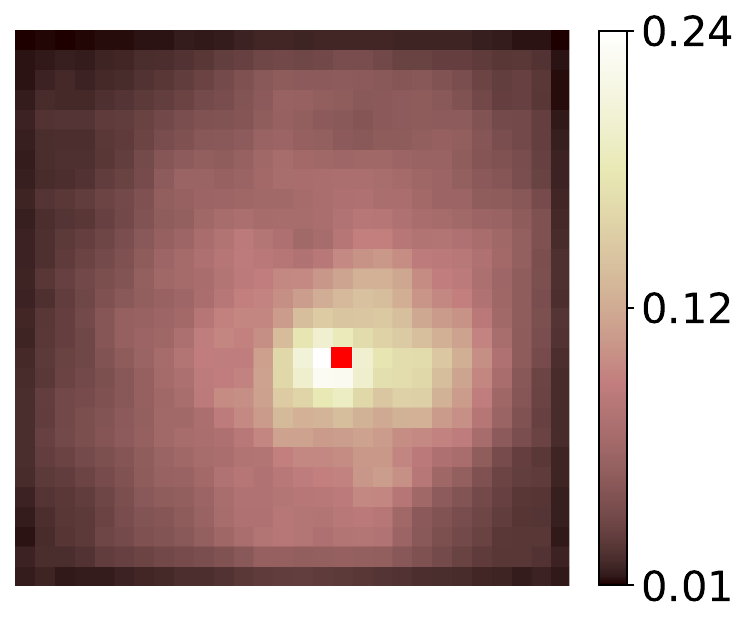}}
\subfigure[Location 6]{\includegraphics[width=0.3\textwidth]{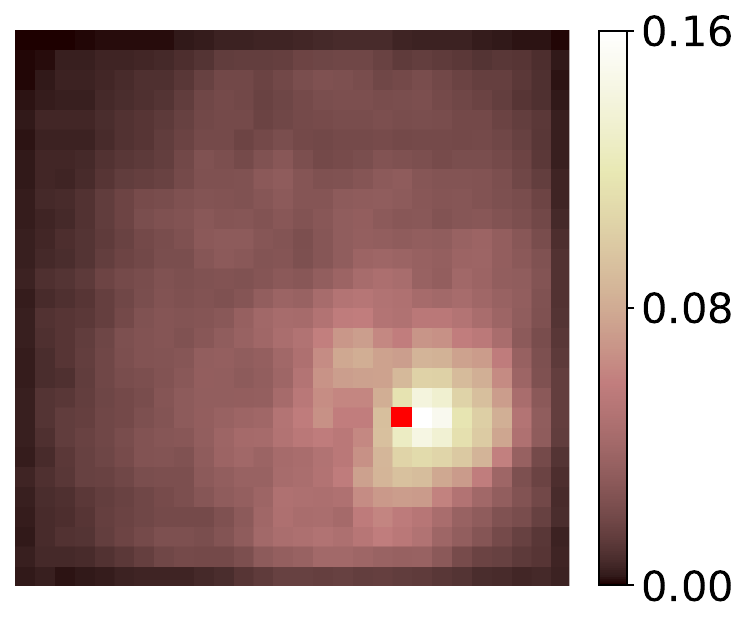}}
\caption{Average component-wise gradient of $\g$ trained on MNIST data. The red marker in each tile denotes the component/location of $x$ under consideration.}
\label{fig:grad_ic}
\end{figure}


\subsection{Steady-state heat conduction: inferring conductivity}
Here we apply our method to infer the conductivity for a steady-state heat conduction problem when the noisy temperature distribution is known a priori. We modify \eqref{eqn:pde_t_heat}-\eqref{eqn:bc_t_heat} by dropping the time-dependent term and ignoring the initial condition:
\begin{alignat}{2}
     -\nabla \cdot (\kappa (\s) \nabla u (\s)) &= b, \qquad 
     &&\forall \ \s \in \Omega  \label{eq:pde_hc} \\
    u (\s) &= 0, \qquad
     &&\forall \ \s \in \partial \Omega . \label{eq:bc_hc}
\end{alignat}
The non-linear relation between conductivity $\kappa$ and the temperature field $u$ is now dependent only on the spatial variable $\s$. Further, we pose this problem on a unit square domain $\Omega=(0,1)^2$ and set the source $b = 10$.
 
The training dataset for this problem consists of 8,000 samples of ($\x$,$\y$) pairs. In order to construct each pair, we sample $\kappa$ from a prior distribution and then solve for $u$, from \eqref{eq:pde_hc} and \eqref{eq:bc_hc}, using the standard Bubnov-Galerkin approach implemented in FEniCS \cite{AlnaesEtal2015}. We use triangular elements to discretize the domain $\Omega$ and first-order Lagrange shape functions to approximate the trial solutions and weighting functions. Both $\kappa$ and the computed field $u$ are then projected onto a square grid with $64 \times 64$ nodes, to obtain $\x$ and $\y$ (without noise) respectively. Finally, an uncorrelated Gaussian noise $\eta \sim \mathcal{N}(0,\sigma^2 \mathbb{I})$ is added to $\y$. We choose $\sigma$ to be 2.5\% of the maximum temperature value from the entire dataset. The prior distribution is constructed such that each $\kappa$ field sampled from this distribution consists of a flat circular inclusion on a constant background value of unity. The coordinates of the center of the circle, $(\xi_1,\xi_2)$, its radius, $\xi_3$, and the contrast $\xi_4$ are sampled from the uniform distributions, 
\[
\xi_1 \sim \mathcal{U}(0,1), \quad \xi_2 \sim \mathcal{U}(0,1), \quad \xi_3 \sim \mathcal{U}(0.05,0.3), \quad \xi_4 \sim \mathcal{U}(2,10).
\] 
Few samples from the dataset are shown in Figure \ref{fig:heat_conduction_samples}.

The predictions of the trained network (refer appendix A for the architecture and hyper-parameters) on a few ``in-distribution'' test samples are shown in Figure \ref{fig:heat_conduction_circle} where the mean, SD, and important samples are computed for $K=800$. In general, we see that the mean is able to capture the location and size of the target inclusion well. The SD values peak at the edges of the circles where there is a sharp transition in $\kappa$ (the same behavior was noted in Section \ref{sec:inv_ic}). We also observe an increase in SD near the corners of the square domain (see Figure \ref{fig:heat_conduction_circle}(b)) indicating a region of greater uncertainty. This is attributed to the fact that the measured temperature field along the edges, and especially at the corners, is determined by the boundary conditions alone and is not influenced by variations in the thermal conductivity. Once again, we observe that the important samples determined through RRQR capture the variations observed in the samples, especially in the contrast between the inclusion and the background. Of the four samples considered for each case, the first appears to be the closest to the target field.

\begin{figure}[htbp]
\centering
\subfigure[Sample 1]{\includegraphics[width=0.48\textwidth]{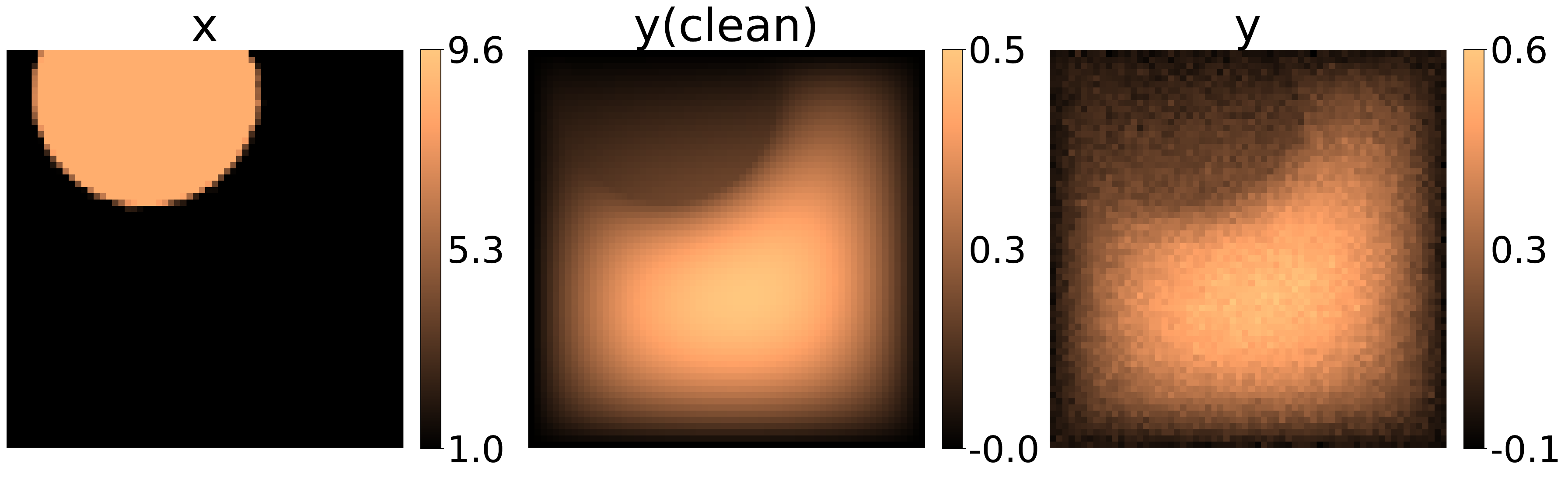}}
\subfigure[Sample 2]{\includegraphics[width=0.48\textwidth]{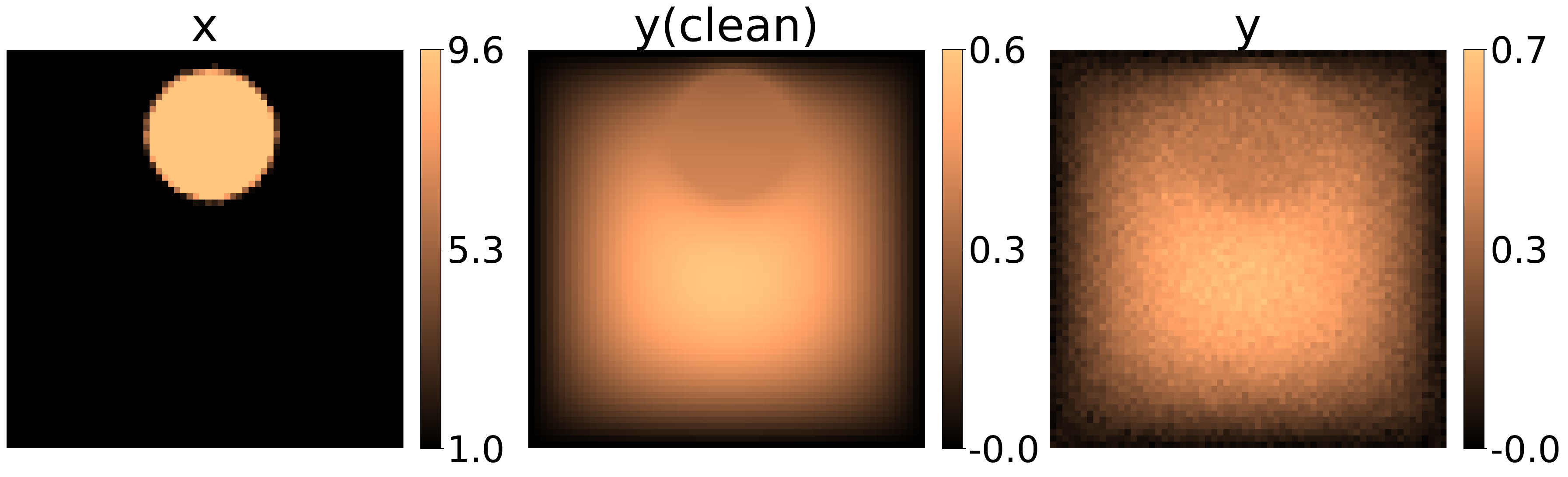}}
\subfigure[Sample 3]{\includegraphics[width=0.48\textwidth]{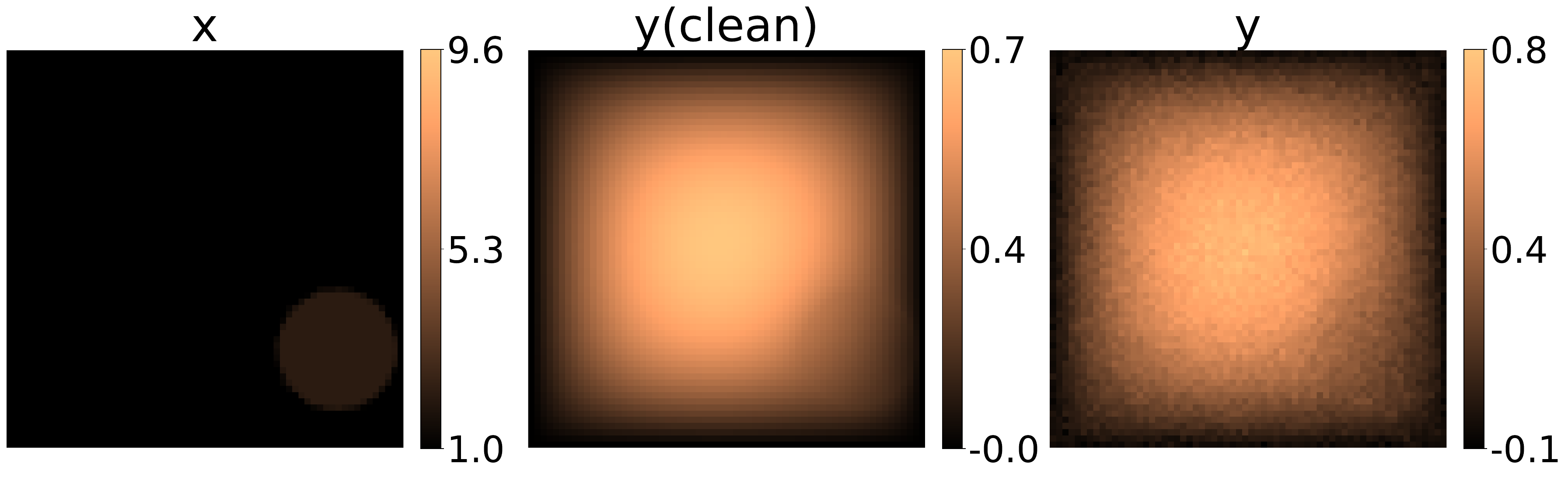}}
\subfigure[Sample 4]{\includegraphics[width=0.48\textwidth]{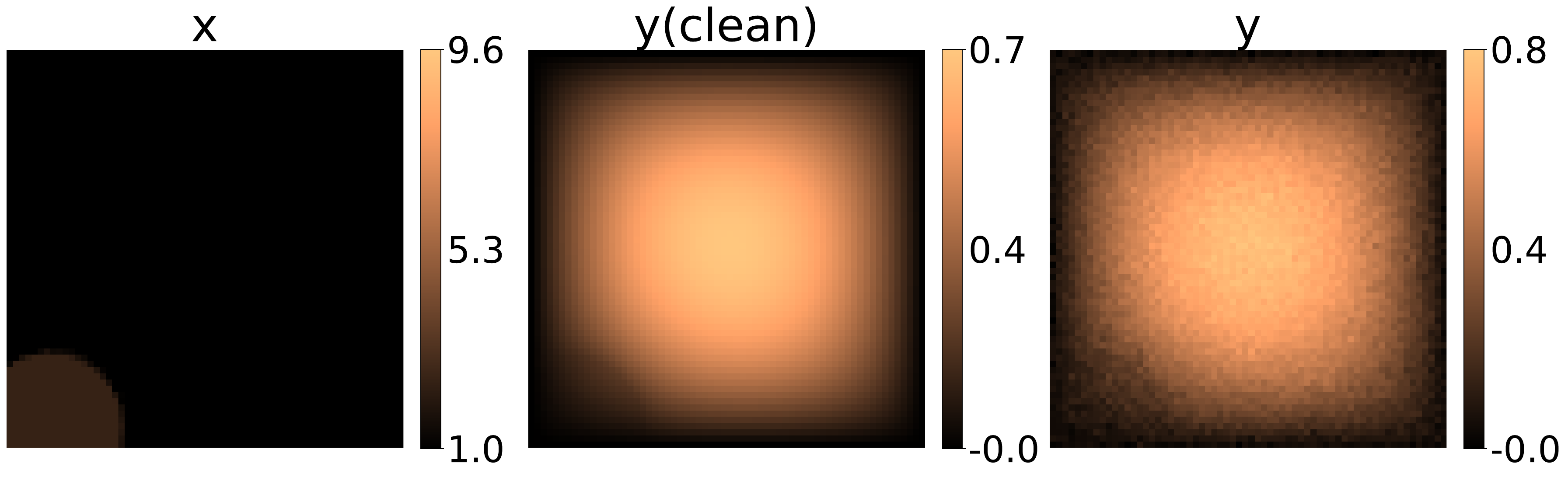}}
\caption{Samples from the dataset that was used to train the network for inferring the conductivity. Each x sample consists of a circular inclusion with a randomly chosen contrast value.}
\label{fig:heat_conduction_samples}
\end{figure} 

\begin{figure}[htbp]
\centering
\subfigure[Sample 1]{\includegraphics[width=0.8\textwidth]{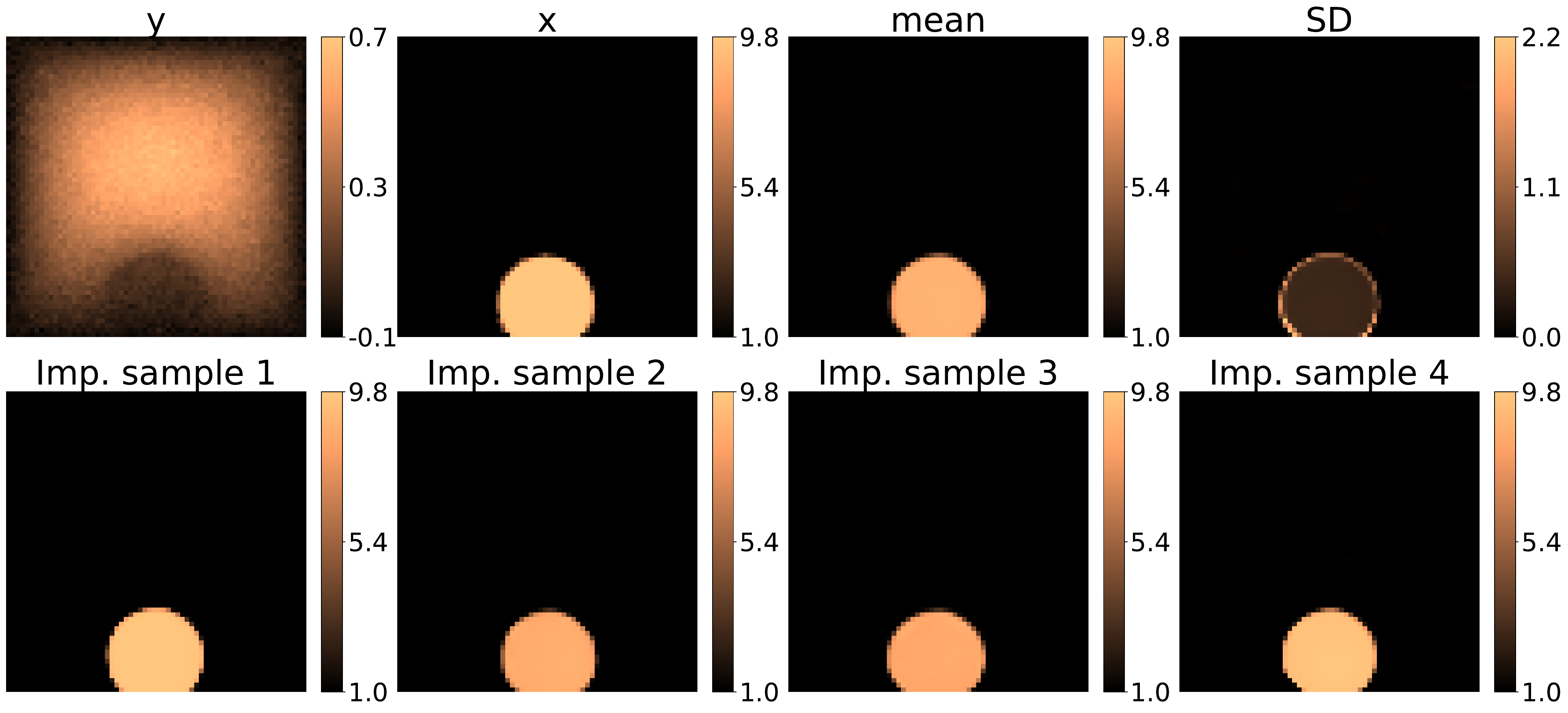}}
\subfigure[Sample 2]{\includegraphics[width=0.8\textwidth]{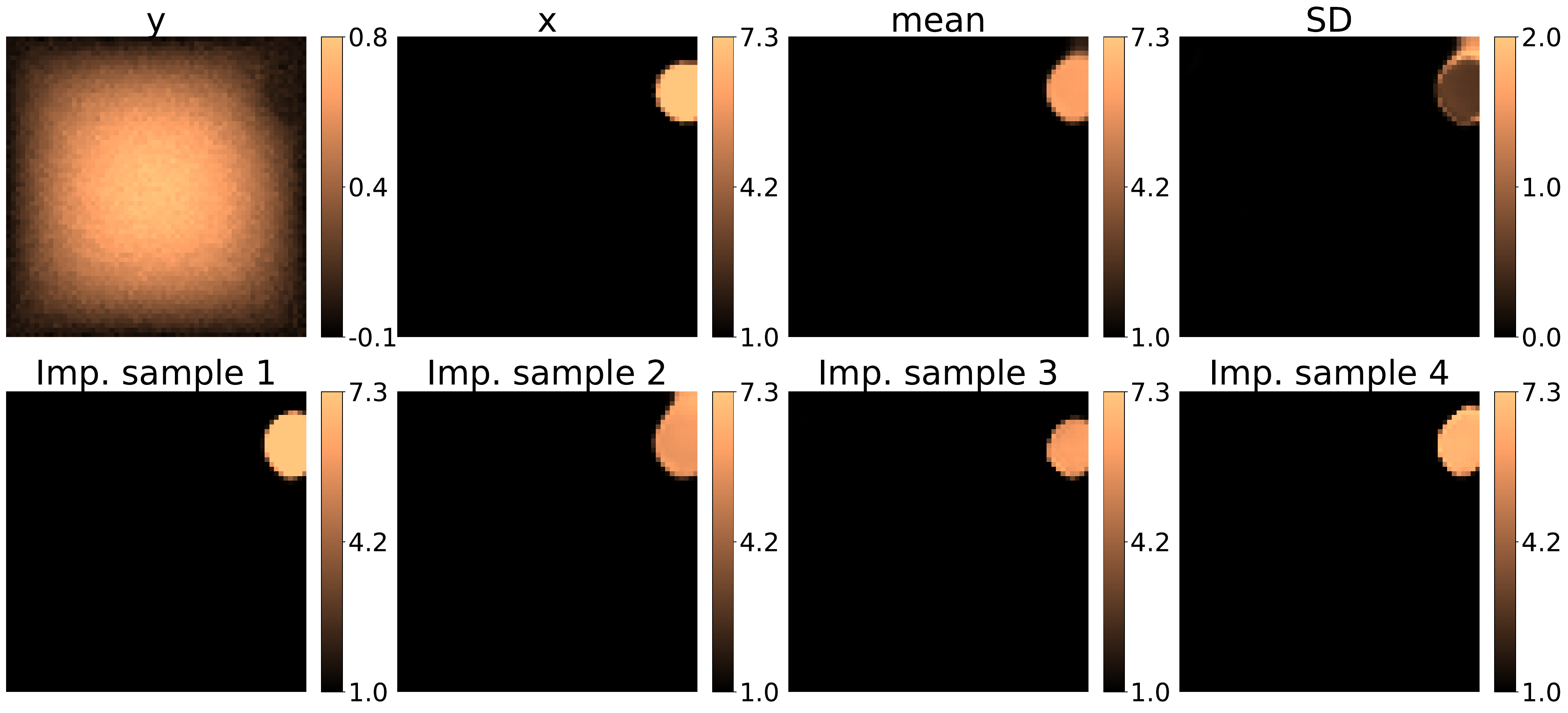}}
\subfigure[Sample 3]{\includegraphics[width=0.8\textwidth]{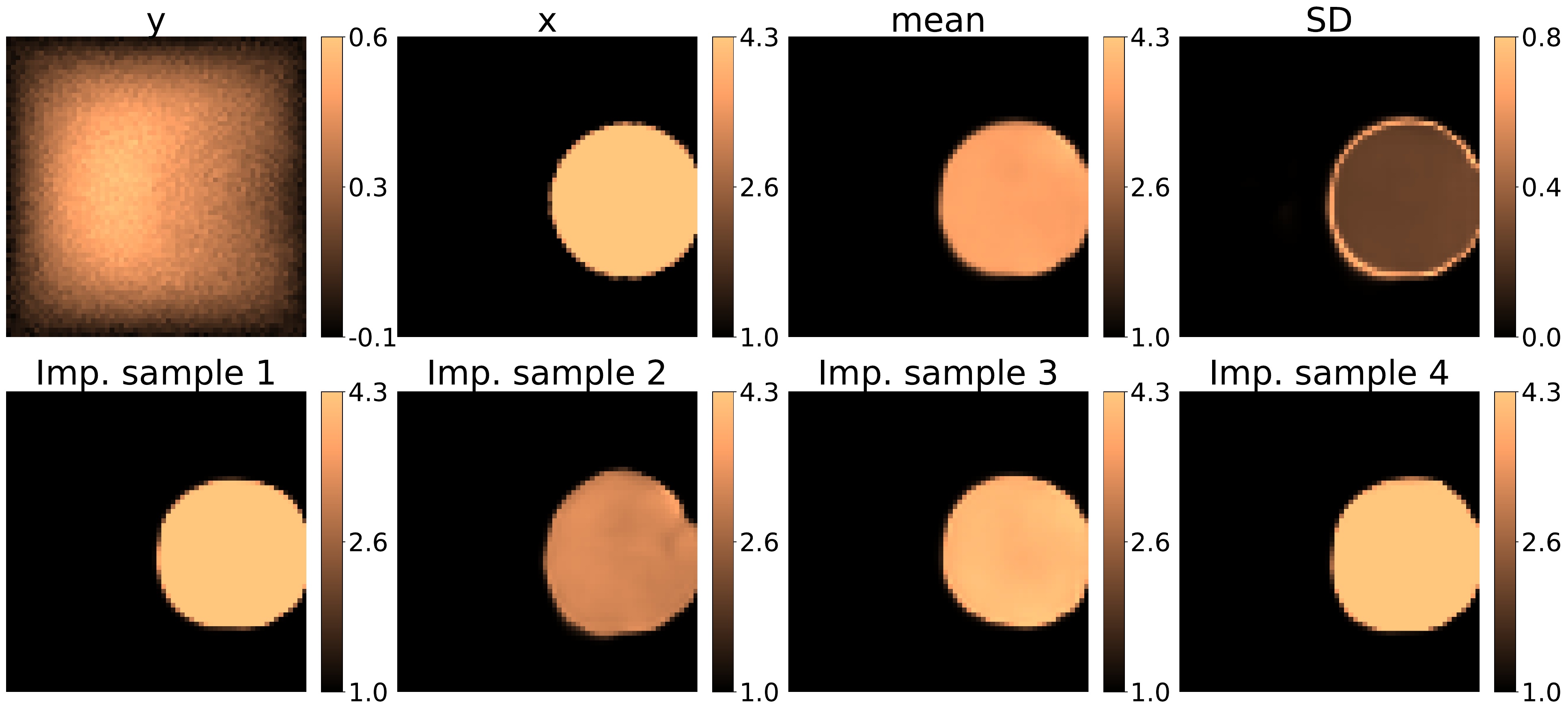}}
\caption{Inferring conductivity for test samples generated from circular priors (same distribution as the training set).}
\label{fig:heat_conduction_circle}
\end{figure}

Next we investigate the performance of the trained generator on OOD samples from two other distributions: elliptical inclusions (Figure \ref{fig:heat_conduction_ellipse}) and two circular inclusions (Figure \ref{fig:heat_conduction_2circle}). For the elliptical inclusions we observe that the generator correctly locates the centroids in all cases and predicts the contrast between the inclusion and the background accurately. It also generates inclusions that are elliptical, however they appear to be more ``circular'' when compared to the target field. For the two circular inclusions we observe that generator is able to generate two distinct inclusions (a case that is never included in the training set) as long as the inclusions are distant from each other (Figure \ref{fig:heat_conduction_2circle}(b)) or are of comparable size (Figure \ref{fig:heat_conduction_2circle}(c)). In the case of two unequal inclusions of very different sizes it misses the smaller inclusion (Figure \ref{fig:heat_conduction_2circle}(a)). Further, it incorrectly infers that the magnitude of conductivity in the two inclusions is roughly the same.  Overall, similar to the inverse initial condition problem, we conclude that cWGAN generalizes reasonably well to these OOD datasets.

\begin{figure}[htbp]
\centering
\subfigure[Sample 1]{\includegraphics[width=0.8\textwidth]{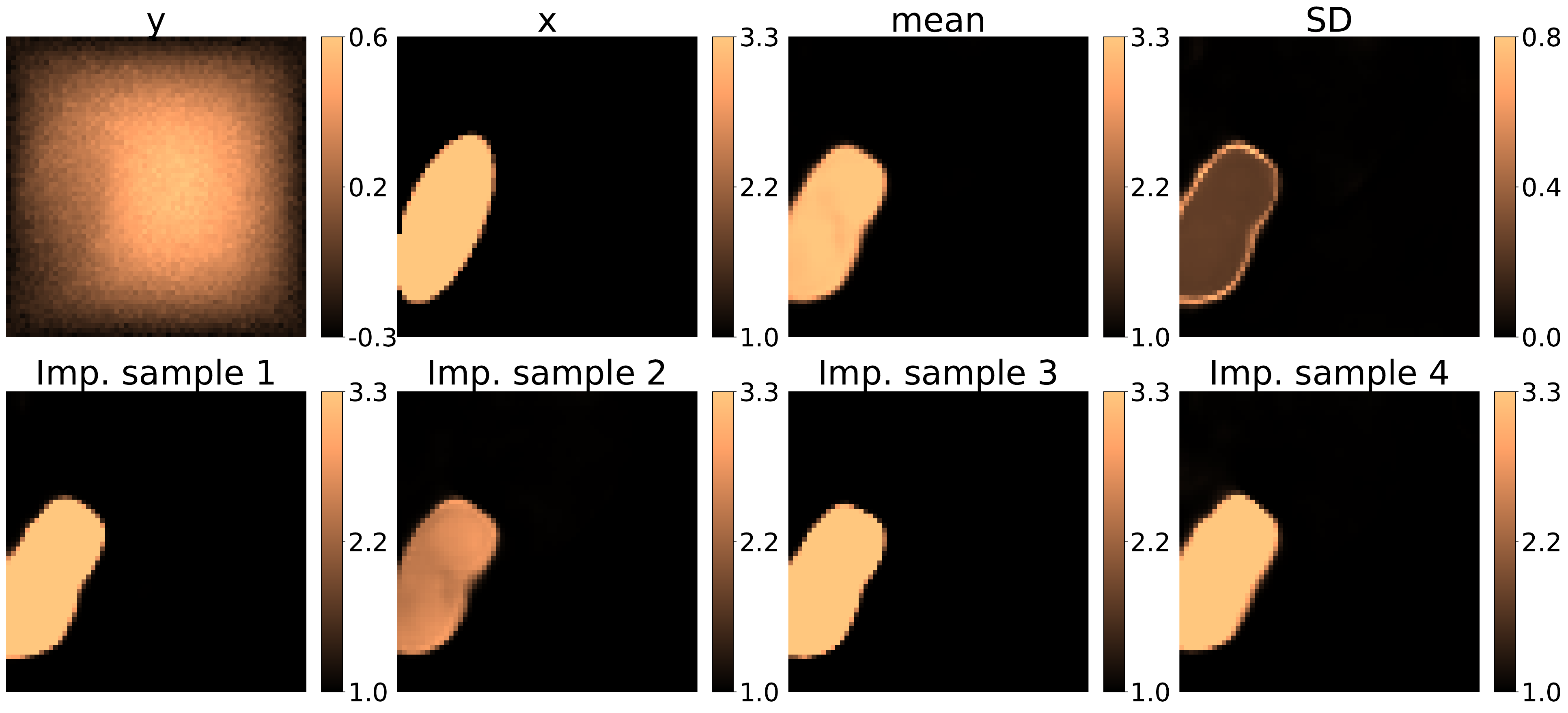}}
\subfigure[Sample 2]{\includegraphics[width=0.8\textwidth]{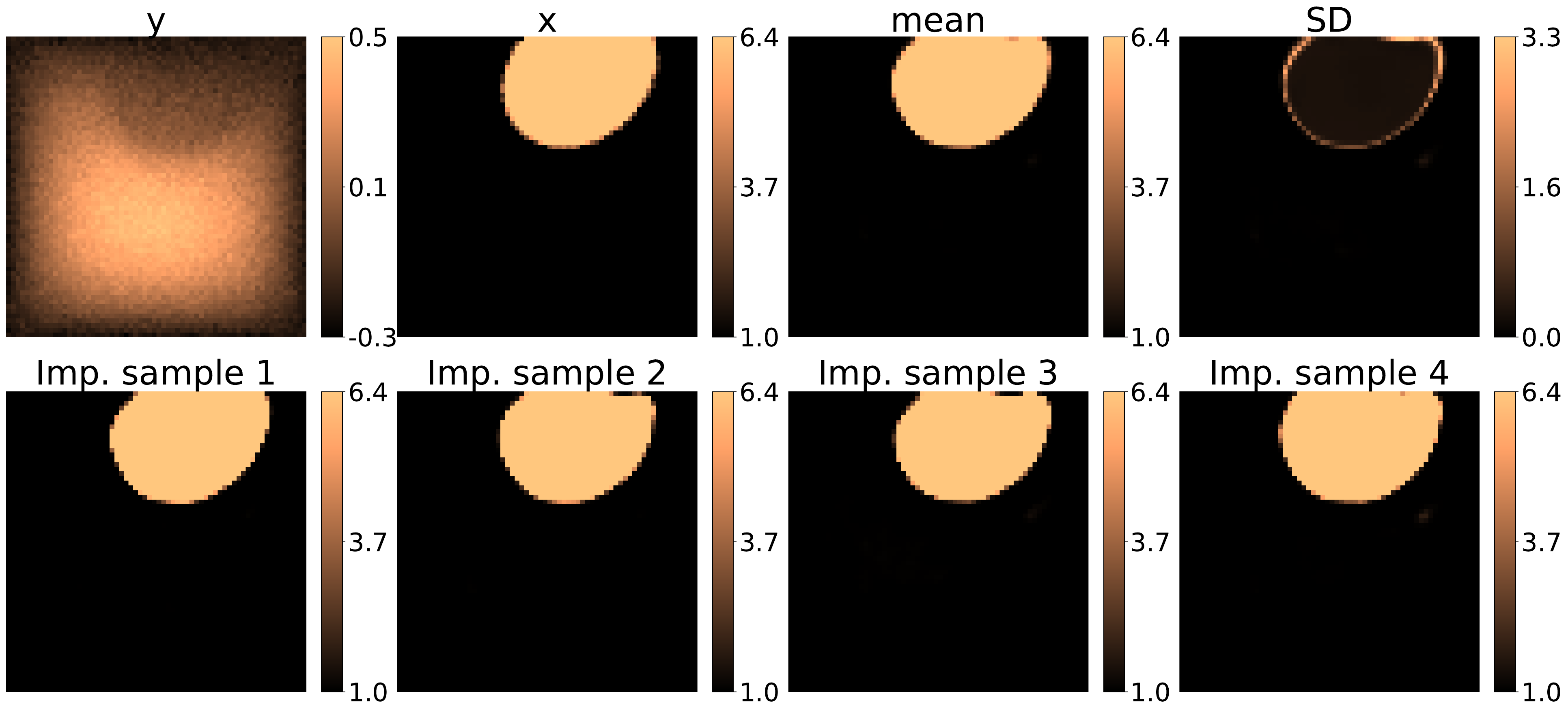}}
\subfigure[Sample 3]{\includegraphics[width=0.8\textwidth]{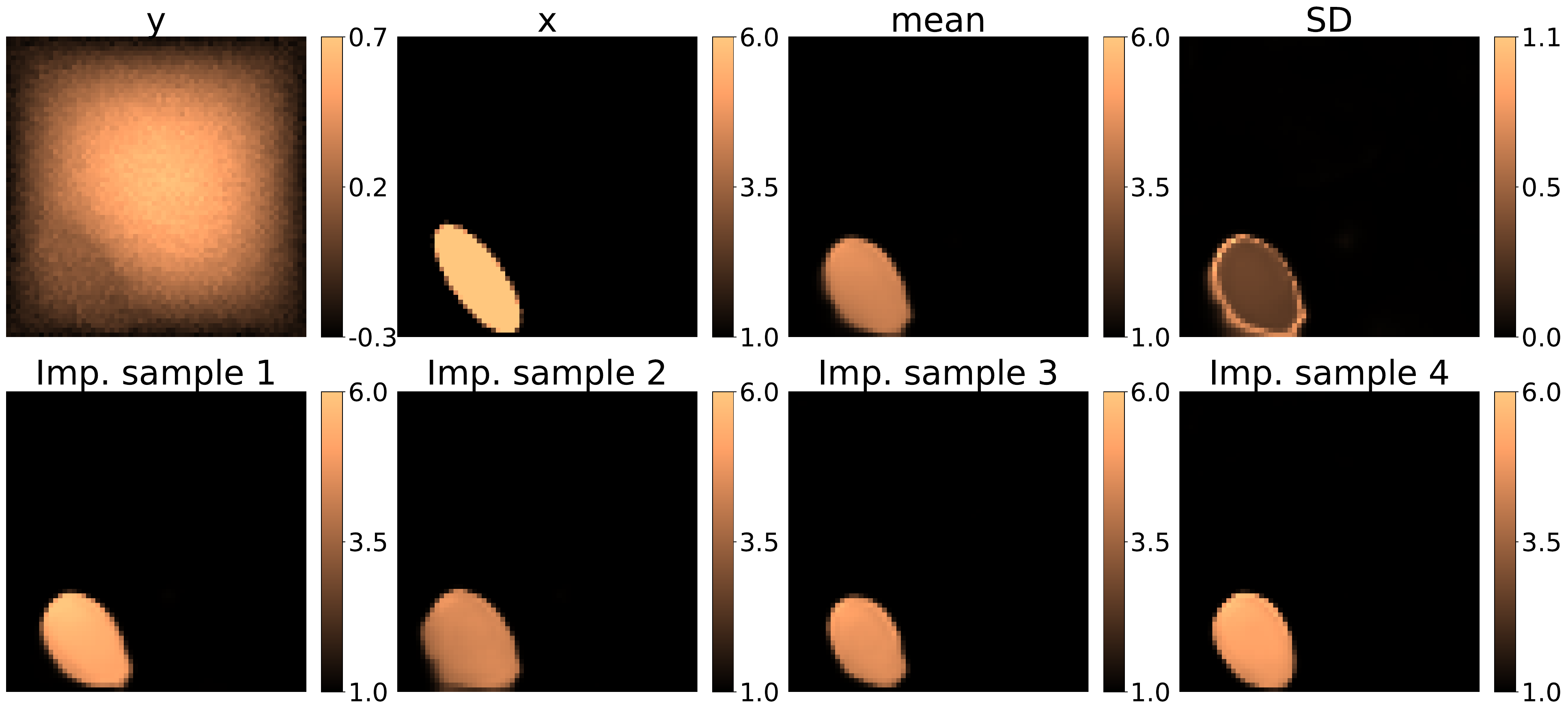}}
\caption{Inferring conductivity for OOD samples generated with elliptical priors.}
\label{fig:heat_conduction_ellipse}
\end{figure}

\begin{figure}[htbp]
\centering
\subfigure[Sample 1]{\includegraphics[width=0.8\textwidth]{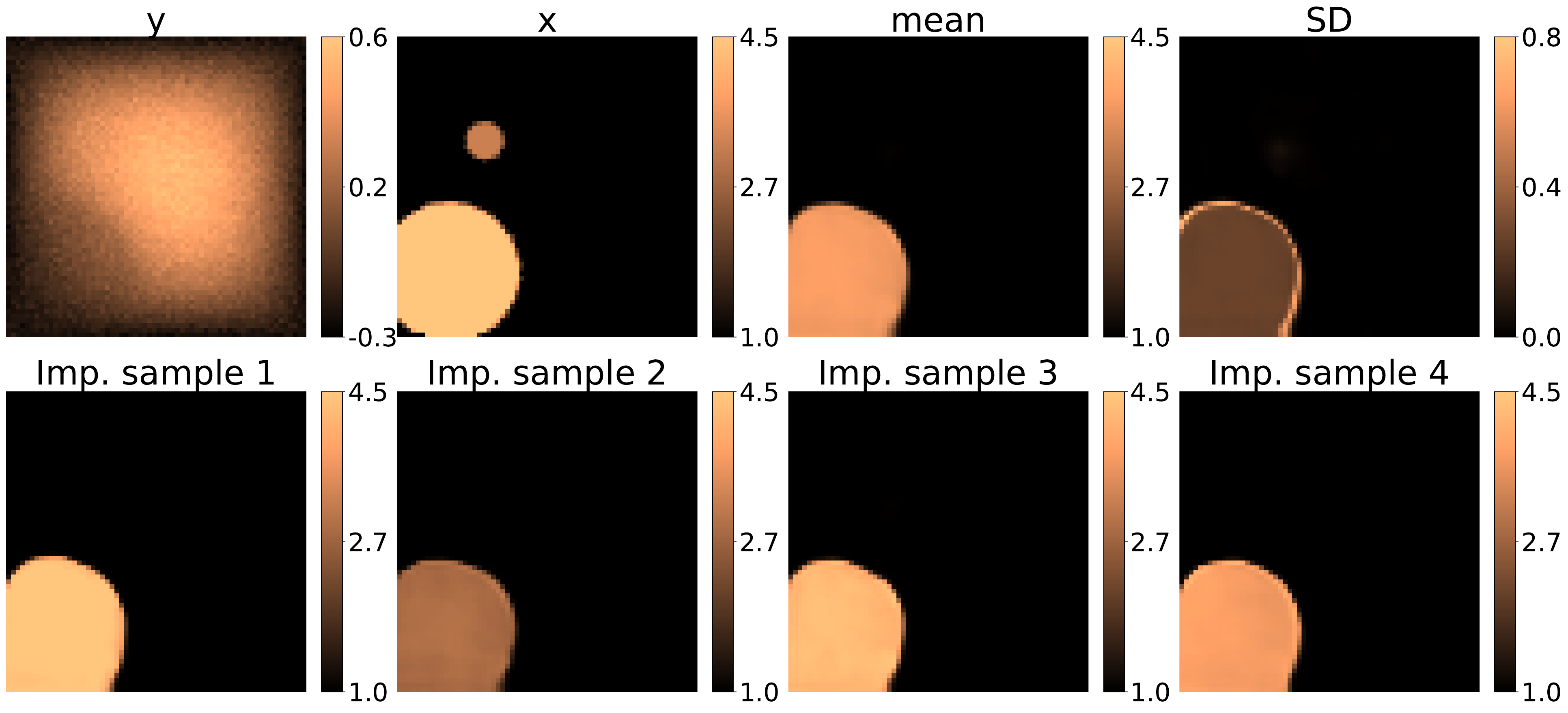}}
\subfigure[Sample 2]{\includegraphics[width=0.8\textwidth]{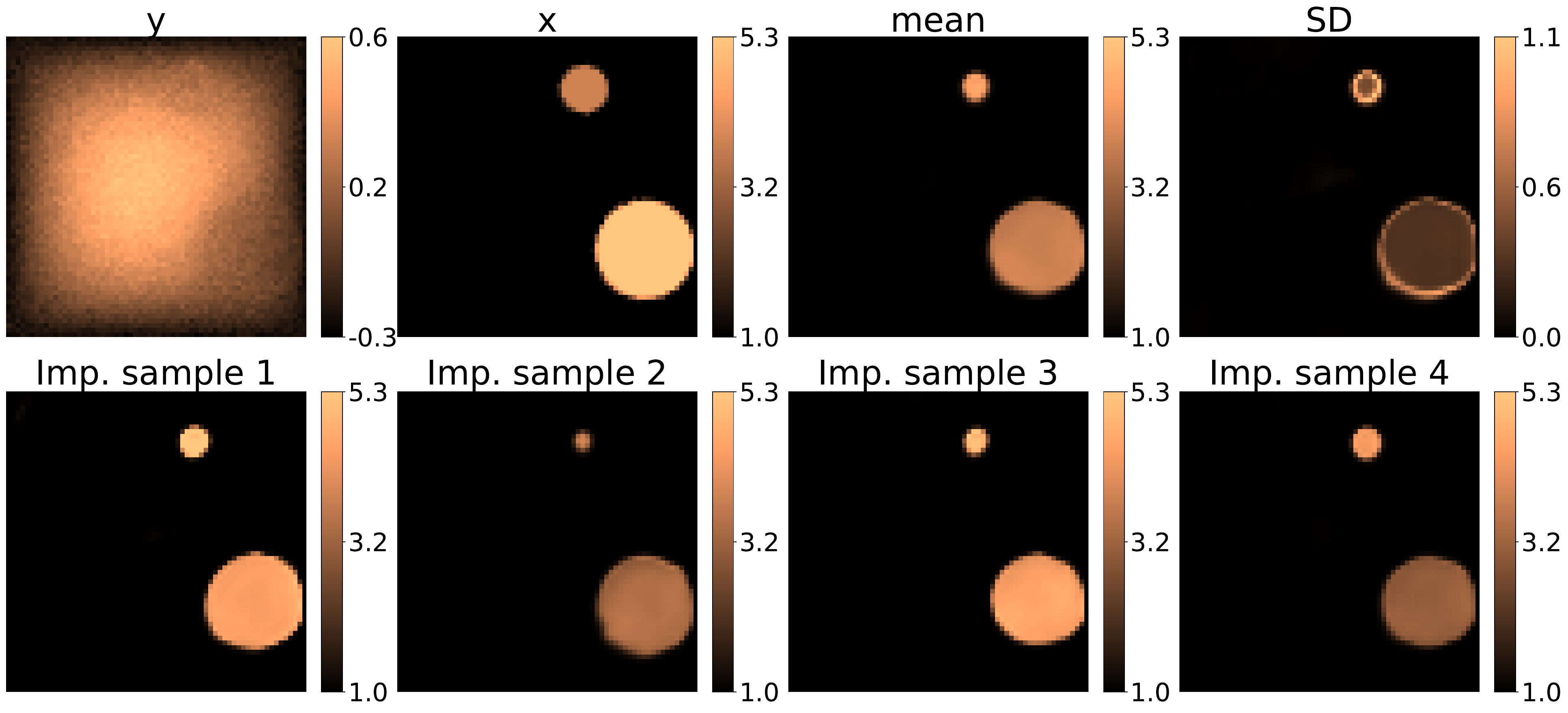}}
\subfigure[Sample 3]{\includegraphics[width=0.8\textwidth]{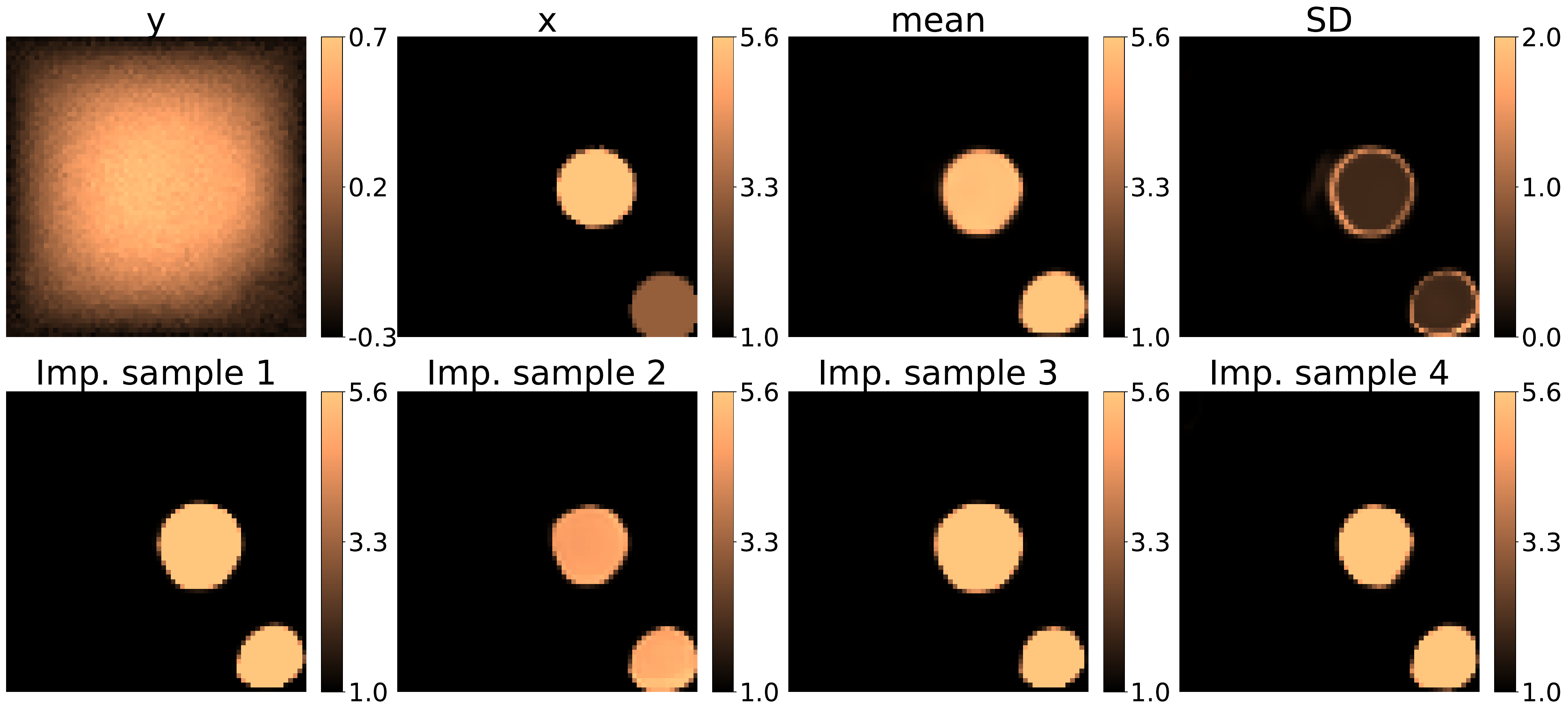}}
\caption{Inferring conductivity for OOD samples involving two circles.}
\label{fig:heat_conduction_2circle}
\end{figure}

\begin{figure}[htbp]
\centering
\subfigure[Location 1]{\includegraphics[width=0.3\textwidth]{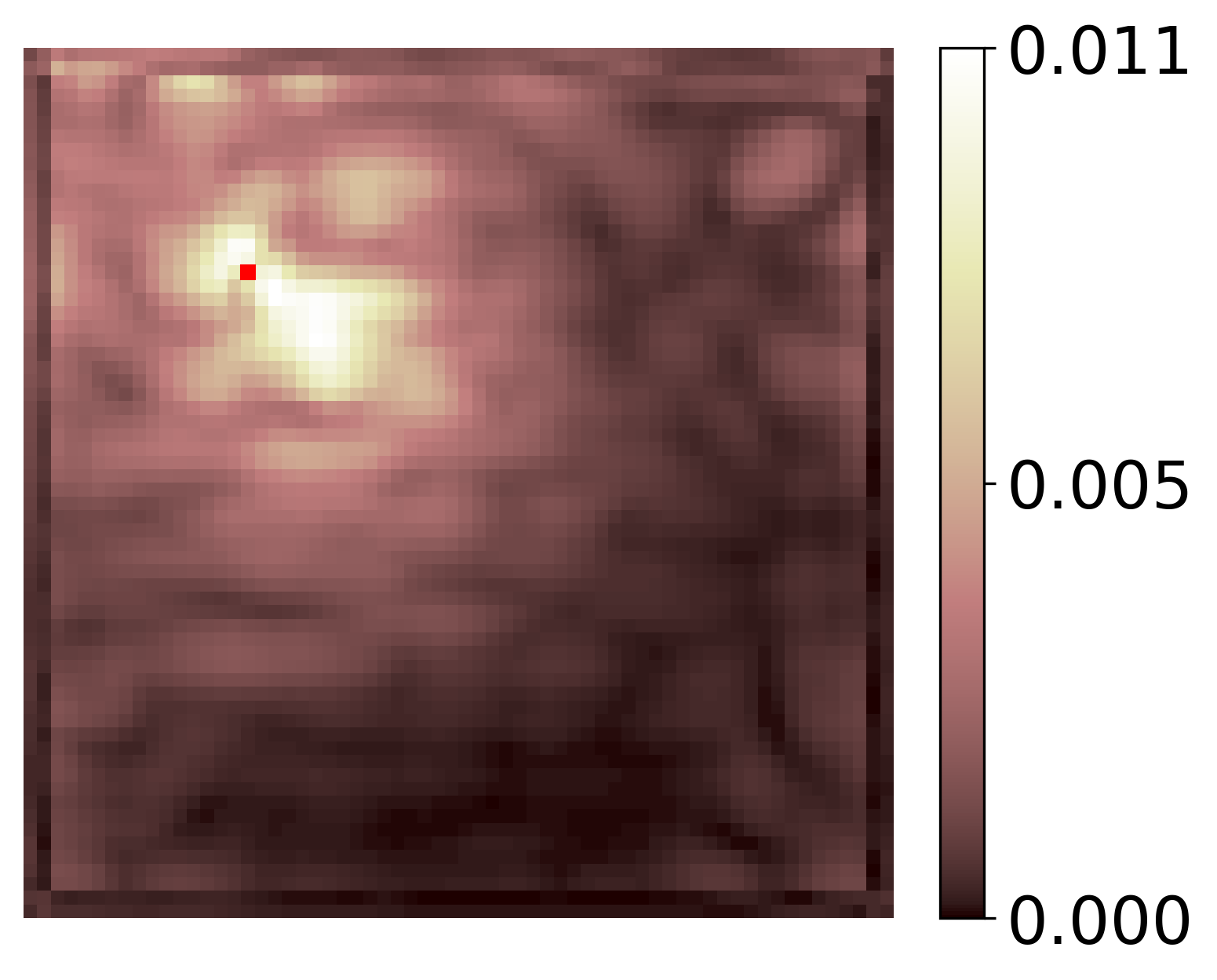}}
\subfigure[Location 2]{\includegraphics[width=0.3\textwidth]{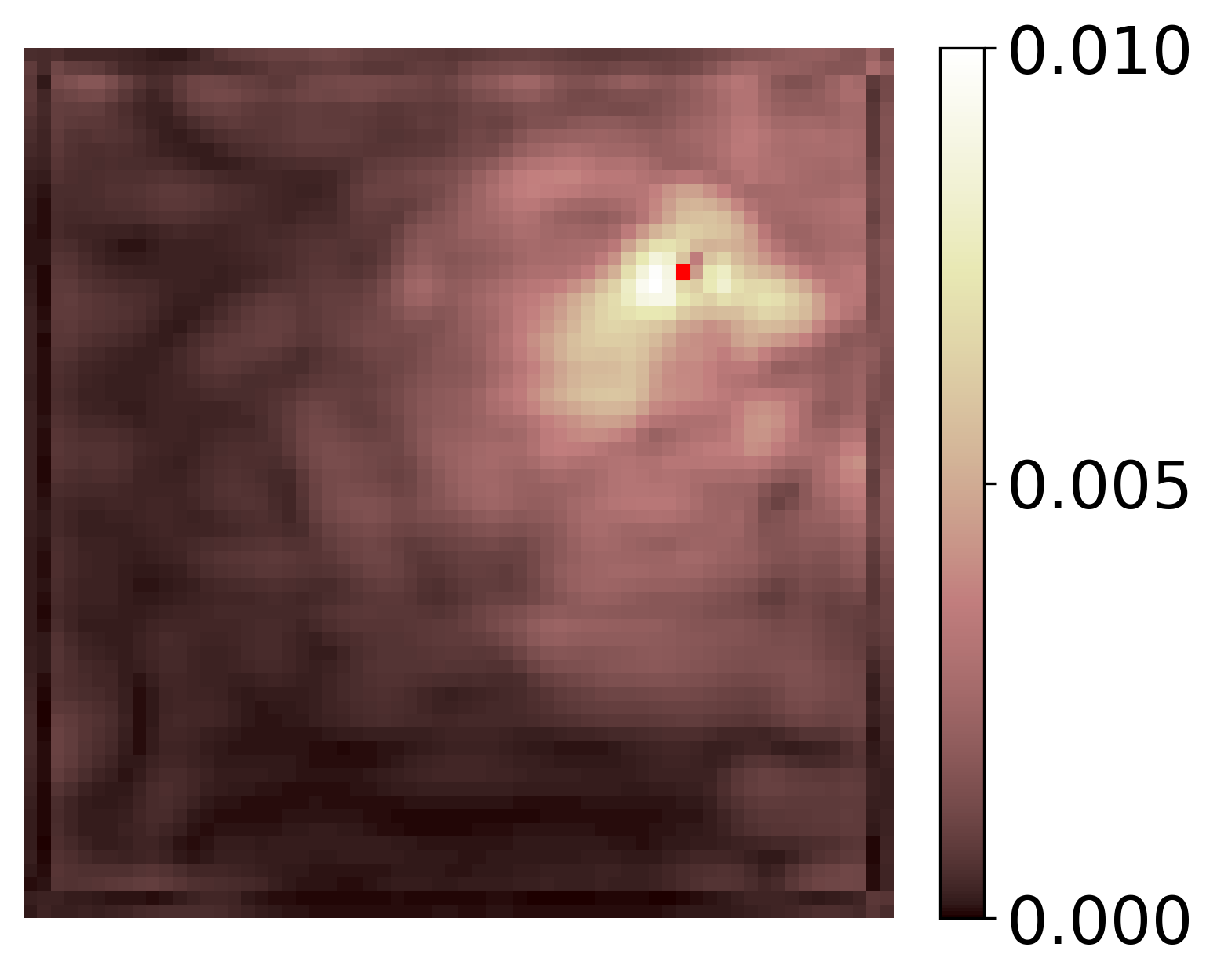}}
\subfigure[Location 3]{\includegraphics[width=0.3\textwidth]{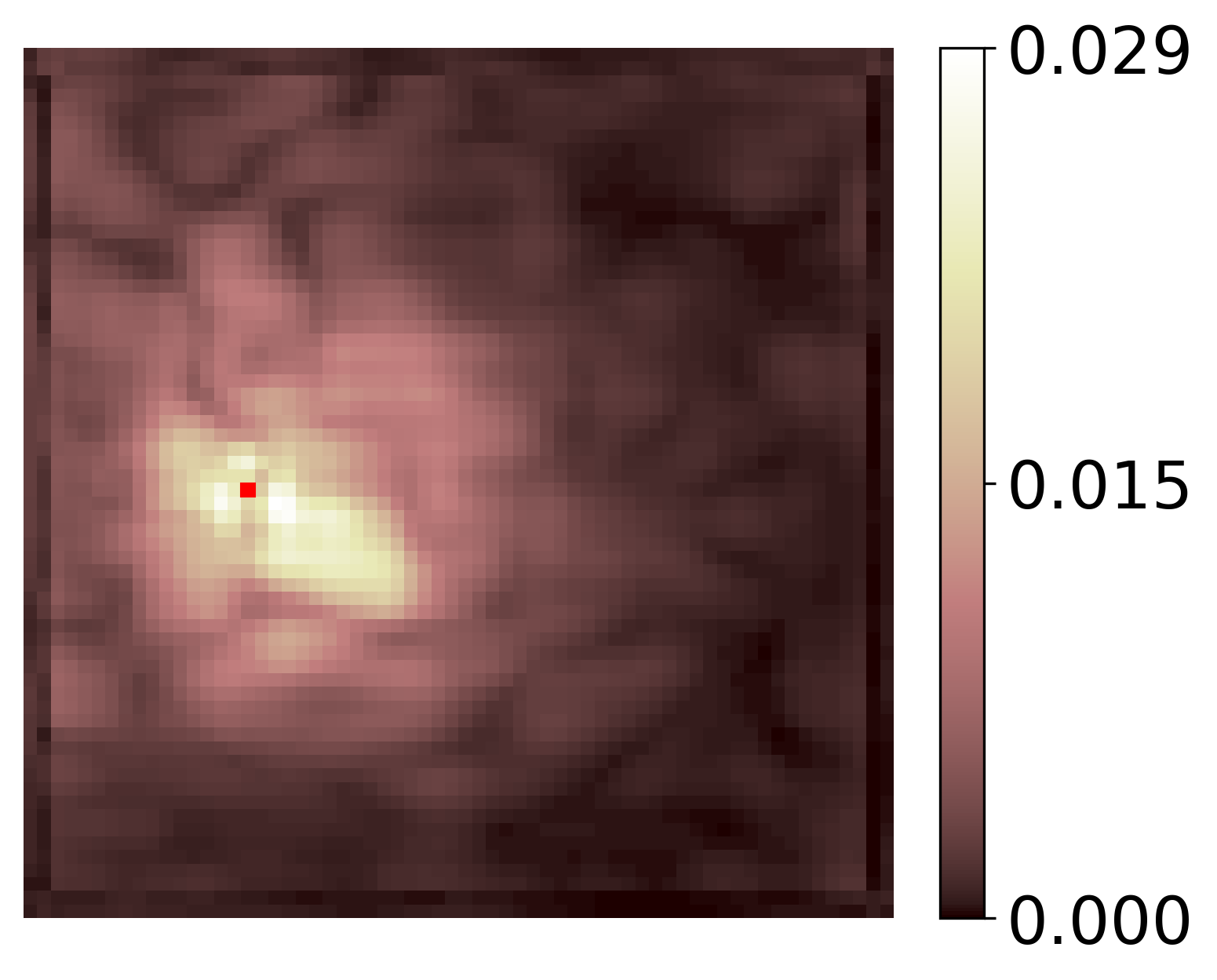}}\\
\subfigure[Location 4]{\includegraphics[width=0.3\textwidth]{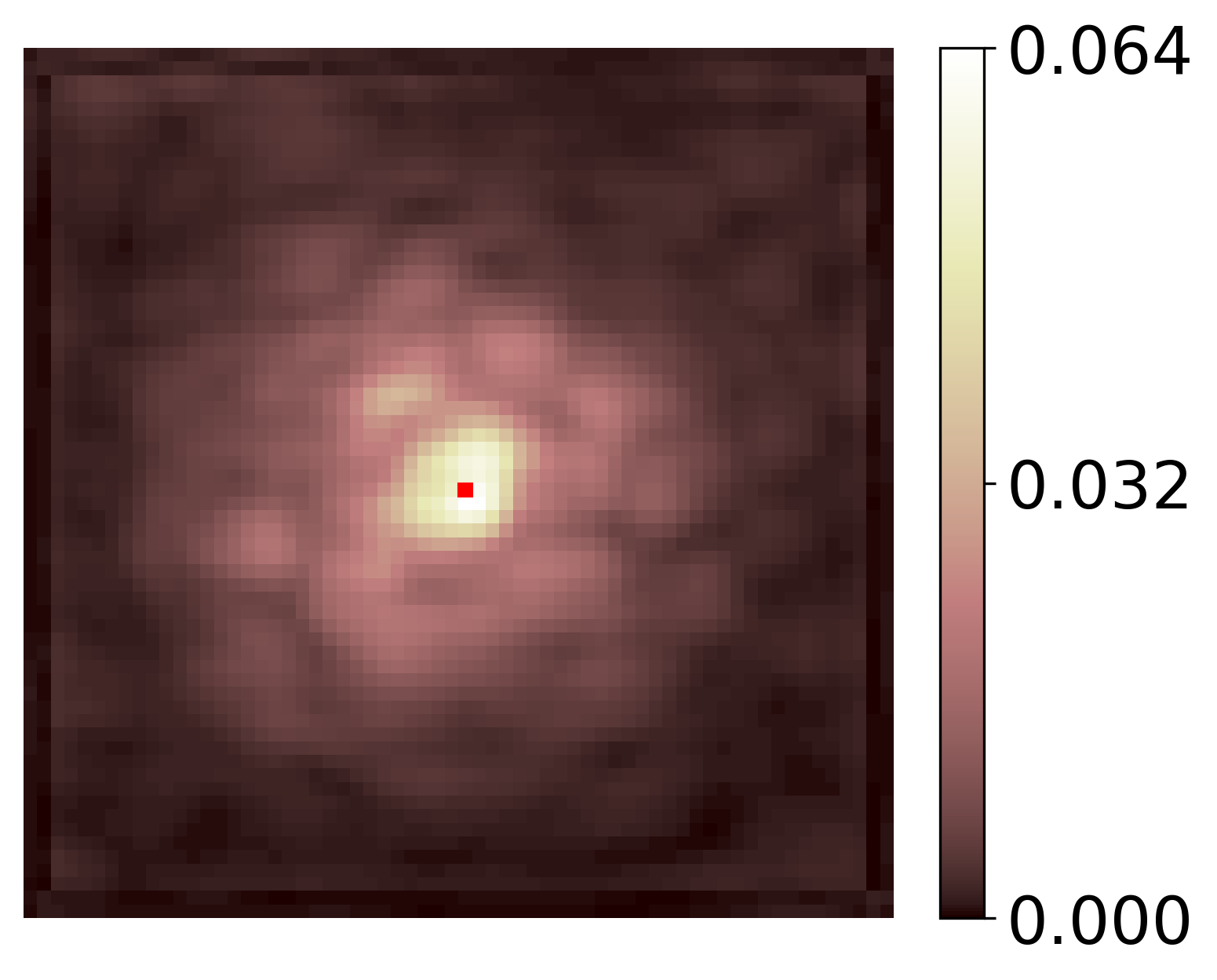}}
\subfigure[Location 5]{\includegraphics[width=0.3\textwidth]{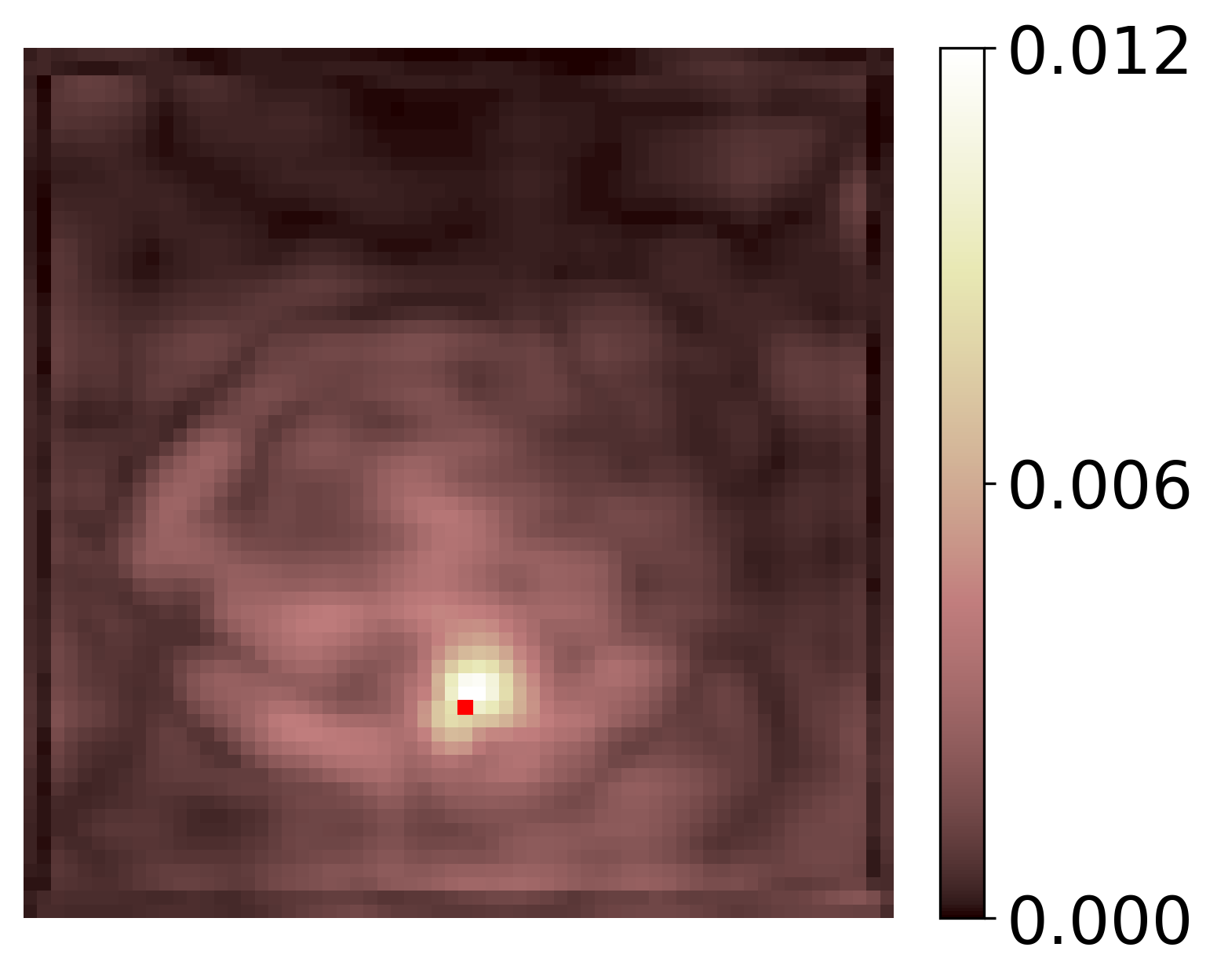}}
\subfigure[Location 5]{\includegraphics[width=0.3\textwidth]{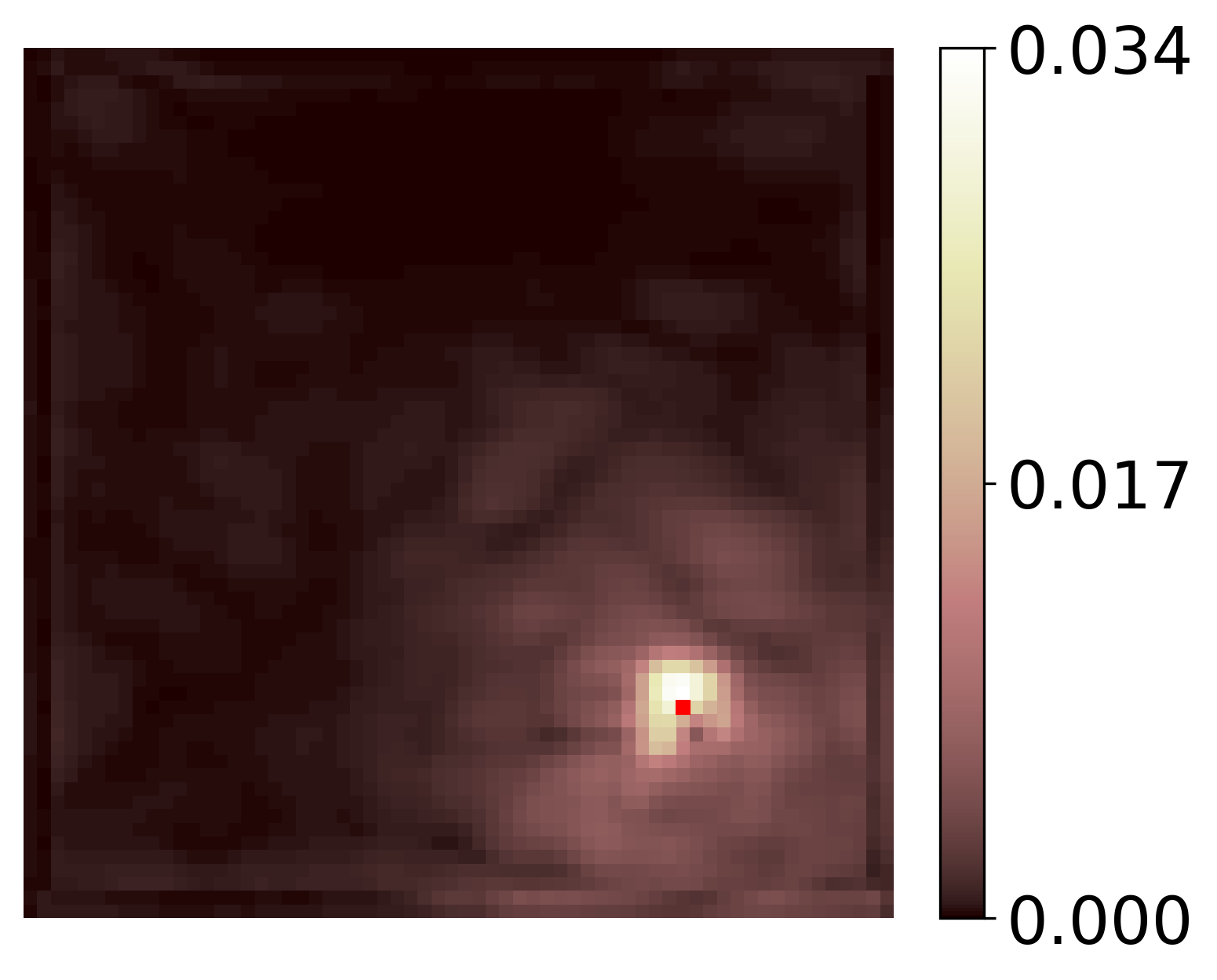}}
\caption{Average component-wise gradient data from the network used for conductivity inference. The red marker denotes the component/location (of $x$) under consideration.}
\label{fig:grad_hc}
\end{figure}

In order to better understand this generalizability of the cWGAN we return to the discussion following Theorem \ref{thm:gen2} in Section \ref{sec:general}. Similar to the inverse problem for the initial condition, in the paragraphs below we demonstrate that the true and learned inverse maps for the inverse conductivity problem are local. Thus, Theorem \ref{thm:gen2} tells us that the cWGAN will generalize well to datasets that have the same local features as the one used for training. It is clear that this is the case for the OOD and in-distribution datasets considered in this problem, since at a local level they involve sharp changes in the thermal conductivity along smooth spatial curves. 

To examine the locality of the true inverse map, we perturb the temperature by a spatially confined perturbation, $\delta u$, and determine the corresponding perturbation in the inferred thermal conductivity. In order to eliminate the effect of boundary conditions, we consider the free-space problem. The perturbed temperature and conductivity fields together satisfy 
\begin{eqnarray*}
\nabla \cdot \big( (\kappa + \delta \kappa) \nabla (u+ \delta u) \big) = f.
\end{eqnarray*} 
Using \eqref{eq:pde_hc}, and retaining only the first order terms, this yields 
\begin{eqnarray*}
\nabla \cdot \big( \delta \kappa \nabla u \big) =  - \nabla \cdot \big( \kappa  \nabla \delta u \big).
\end{eqnarray*}
Given the pair $(\kappa, u)$, and a perturbation in the measurement, $\delta u$, this equation determines the corresponding change in the conductivity, $\delta \kappa$. A regularized (by diffusion) version of this equation reads, 
\begin{eqnarray*}
\nabla \cdot \big( \delta \kappa \nabla u \big)  + \rho \nabla^2 (\delta \kappa) =  - \nabla \cdot \big( \kappa  \nabla \delta u \big),
\end{eqnarray*}
which is an advection-diffusion equation for $\delta \kappa$ with a ``velocity field'' given by $\nabla u$ and a diffusion coefficient $\rho$. It is driven by the source term involving $\delta u$. The formal solution to this problem, obtained after utilizing the free-space Green's function, $g(\s,\s')$, and performing integration-by-parts twice, is
\begin{eqnarray*}
\delta \kappa (\s) = - \int_{\mathbb{R}^2}  \nabla \cdot \big( \kappa (\s')  \nabla g(\s,\s') \big) \delta u (\s') d \s' .
\end{eqnarray*}
Now assuming that $\delta u$ is concentrated around $\s^0$ and integrates to $\delta U$, the integral above may be approximated as 
\begin{eqnarray}\label{eqn:kpert}
\delta \kappa (\s) \approx -   \delta U \big( \nabla \kappa (\s^0) \cdot    \nabla g(\s,\s^0) + \kappa (\s^0)  \nabla^2 g(\s,\s^0) \big).
\end{eqnarray}
In two-dimensions, $g(\s,\s^0) \propto \ln | \s -\s^0|$, which implies that the two terms on the right hand side of \eqref{eqn:kpert} have singularities of the type $|\s -\s^0|^{-1}$ and $|\s -\s^0|^{-2}$, and therefore decay rapidly away from $\s^0$. This implies that $\delta \kappa$ also decays rapidly away from $\s^0$, and therefore the true regularized inverse map for this problem is local. 

Next, we demonstrate that the inverse map learned by the generator is local. We compute the gradient of the $k$-th component of the prediction with respect to the network input $\y$ (see \eqref{eqn:grad}). 
The averaged gradients for several components are shown in Figure \ref{fig:grad_hc}.
Once again we observe that the gradient for each component is concentrated in the neighbourhood of the corresponding component in $\y$ indicating that this map is also spatially local.


\subsection{Elastography: inferring shear modulus}

Elastography is a promising medical imaging technique where ultrasound is used to image tissue as it is deformed using an external force. The sequence of ultrasound images thus obtained are used to determine the displacement field inside the tissue, which can be related to the spatial distribution of the shear modulus via the equations of equilibrium for an elastic solid,
\begin{alignat*}{2}
     \nabla \cdot \boldsymbol{\sigma}(\boldsymbol{u}(\s)) &= 0, \qquad  
     &&\forall \ \s \in \Omega,  \label{eq:pde_elasticity} 
\end{alignat*}
and appropriate boundary conditions. Here $\boldsymbol{\sigma}(\boldsymbol(u)) = 2\mu(\nabla^s\boldsymbol{u}+(\nabla\cdot\boldsymbol{u})\boldsymbol{I})$ is the Cauchy stress tensor for an incompressible linear isotropic elastic solid in a state of plane stress, $\boldsymbol{u}$ is the displacement field, and $\mu$ is the shear modulus. For the specific experimental configuration considered in this example, the domain $\Omega$ is a $34.608 \times 26.297$ mm rectangle. The boundary conditions are traction free in both directions on the left and right edges, and traction free along the horizontal direction on the top and bottom edges. In addition, the compression of the specimen is modeled by setting the vertical component of displacement to 0.084 mm on the top edge and to 0.392 mm along the bottom edge.

In a typical elastography problem only the displacement component along the axis of the ultrasound transducer is measured accurately. This corresponds to the vertical direction in our configuration. Thus the inverse problem is: given appropriate boundary conditions and the vertical component of the displacement field, determine the spatial distribution of the shear modulus. For the mathematical foundations of elasatography and an analaysis of the uniqueness of this problem, the reader is referred to \cite{barbone2007elastic, barbone2010review}.

The training dataset for this problem consists of 8,000 samples of ($\x$,$\y$), which correspond to the discretized values of $\mu$  and $u_2$ (vertical component of displacement), respectively. These samples are generated by first sampling $\mu$ from a prior distribution and then solving for $u_2$ using the finite element method with linear triangular elements. Both $\mu$, and the computed field $u_2$, are then projected onto a  $56 \times 56$ grid, to obtain $\x$ and $\y$ (without noise) respectively. An uncorrelated Gaussian noise $\eta \sim \mathcal{N}(0,\sigma^2 \mathbb{I})$, with $\sigma = 0.001$ is added to $\y$. The prior distribution for $\mu$ is constructed such that each field consists of a stiff inclusion on a uniform background of 4.7 kPa. The coordinates of the center of the circle, $(\xi_1,\xi_2)$, its radius, $\xi_3$, and the contrast $\xi_4$ are sampled from the uniform distributions:
\[
\xi_1 \sim \mathcal{U}(7.1,19.2), \quad \xi_2 \sim \mathcal{U}(7.1,27.6), \quad \xi_3 \sim \mathcal{U}(3.5,7), \quad \xi_4 \sim \mathcal{U}(1,8).
\] 
Some of the resulting training samples are shown in Figure \ref{fig:elasticity_samples}. 

The experimental data used for this inference problem was obtained from ultrasound scans on a tissue-mimicking phantom \cite{pavan2012nonlinear}. The phantom was manufactured using a mixture of gelatin, agar, and oil and consisted of a stiff inclusion embedded within a softer substrate. The phantom was gently compressed and the vertical displacement field within the phantom was measured. This displacement field is shown in the top-left panel of Figure \ref{fig:elasticity_true}.
 
The results from the shear modulus distribution generated by the cWGAN conditioned on the measured displacement, are shown in Figure \ref{fig:elasticity_true}. These include the mean, the SD, and the three important samples as determined by the RRQR factorization of the matrix of $K = 800$ samples generated by the cWGAN. We observe that the highest uncertainty occurs in a thin layer along the boundary of the inclusion. We also observe that there is very little variability among the three important samples. A comparison of the quantitative metrics calculated from the mean image, and independent experimental measurements is as follows \cite{pavan2012nonlinear}:
\begin{enumerate}
    \item Vertical distance from the bottom edge to the center of the inclusion:  estimated = 12.0 mm; measured = 13.0 mm; error =  7.7\%.
    \item Diameter of the inclusion: estimated =  10.3 mm; measured = 10.0 mm; error  = 3\%.
    \item Shear modulus of the inclusion: estimated = 13.3 kPa; measured  = 10.7 kPa; error = 24.3 \%. 
\end{enumerate}
These errors are within the range of errors obtained from other deterministic approaches used to solve this inverse problem \cite{francois2018validation}. The advantage of the Bayesian approach is that it provides additional estimates of uncertainty. 

\begin{figure}[htbp]
\centering
\subfigure[Sample 1]{\includegraphics[width=0.48\textwidth]{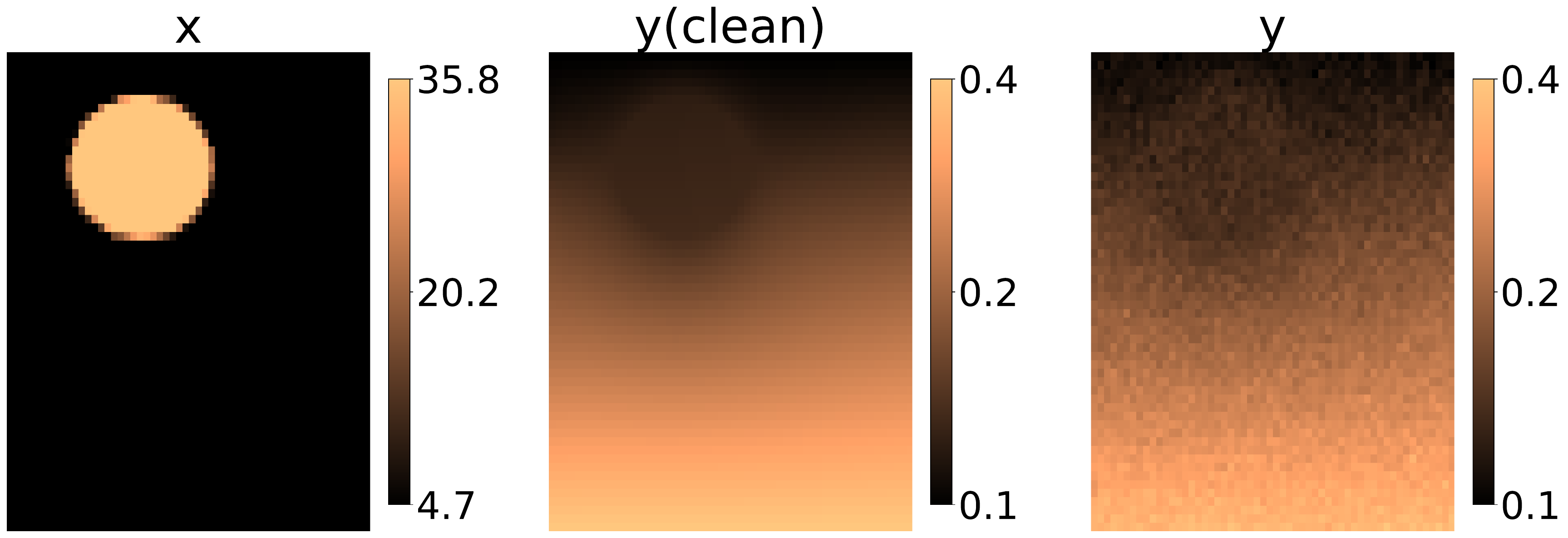}}
\subfigure[Sample 2]{\includegraphics[width=0.48\textwidth]{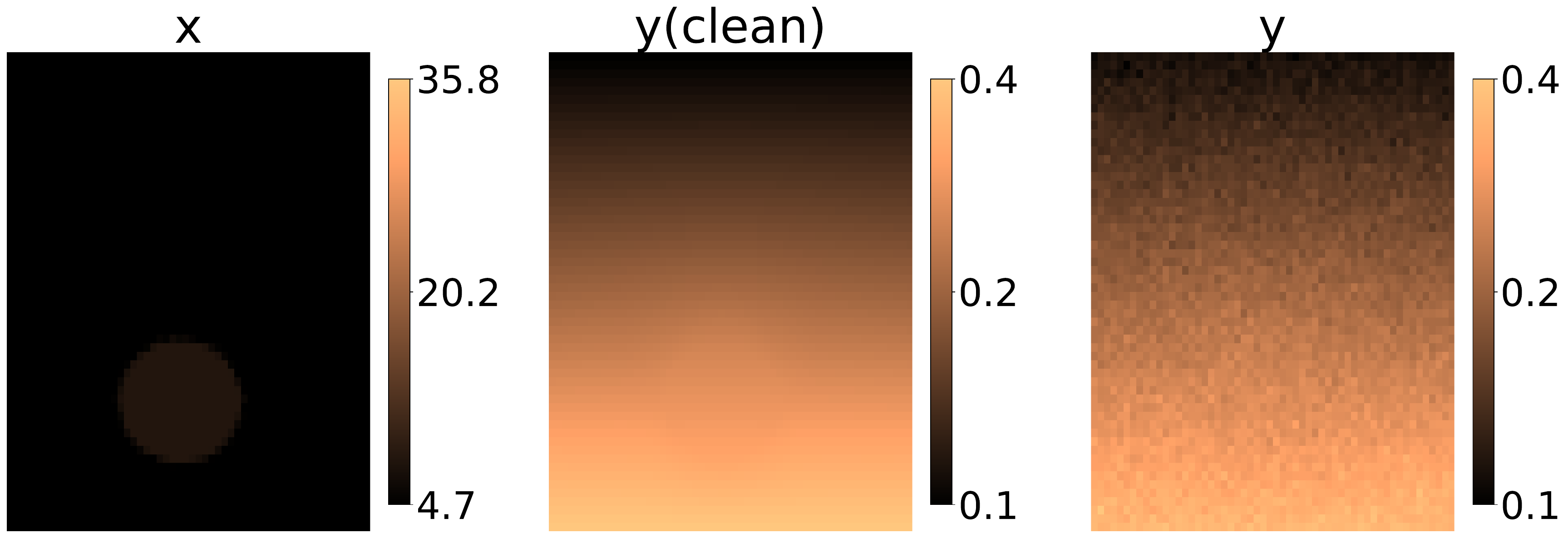}}
\subfigure[Sample 3]{\includegraphics[width=0.48\textwidth]{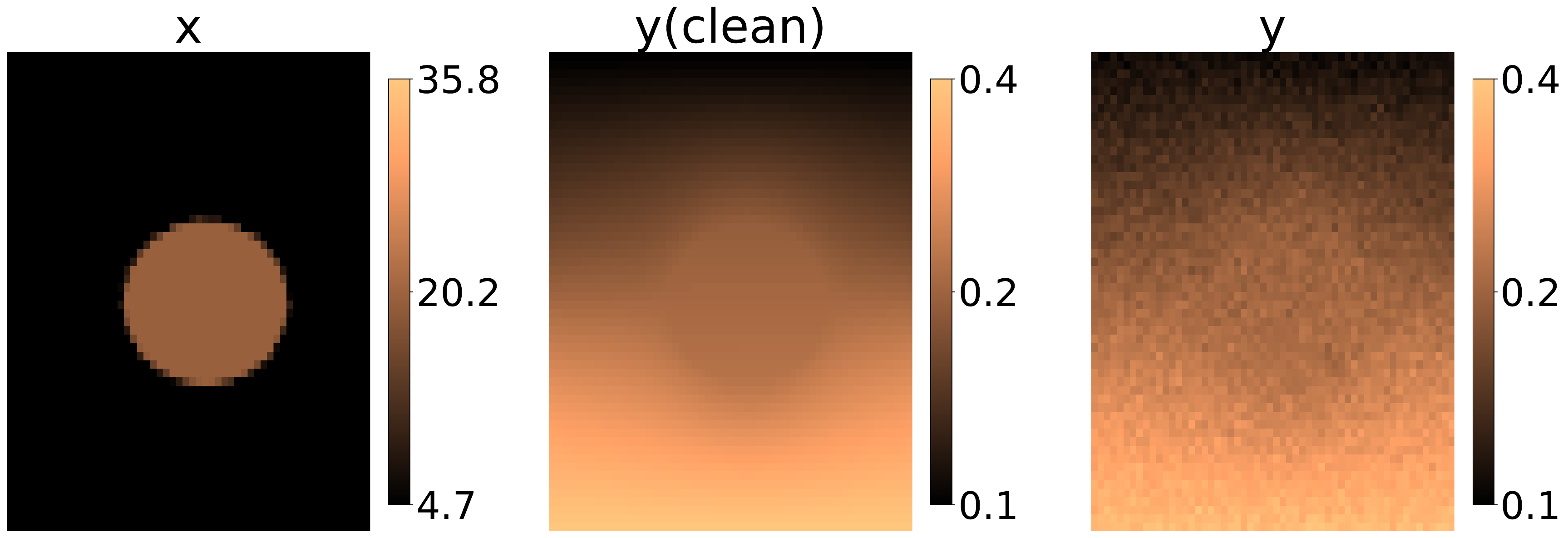}}
\subfigure[Sample 4]{\includegraphics[width=0.48\textwidth]{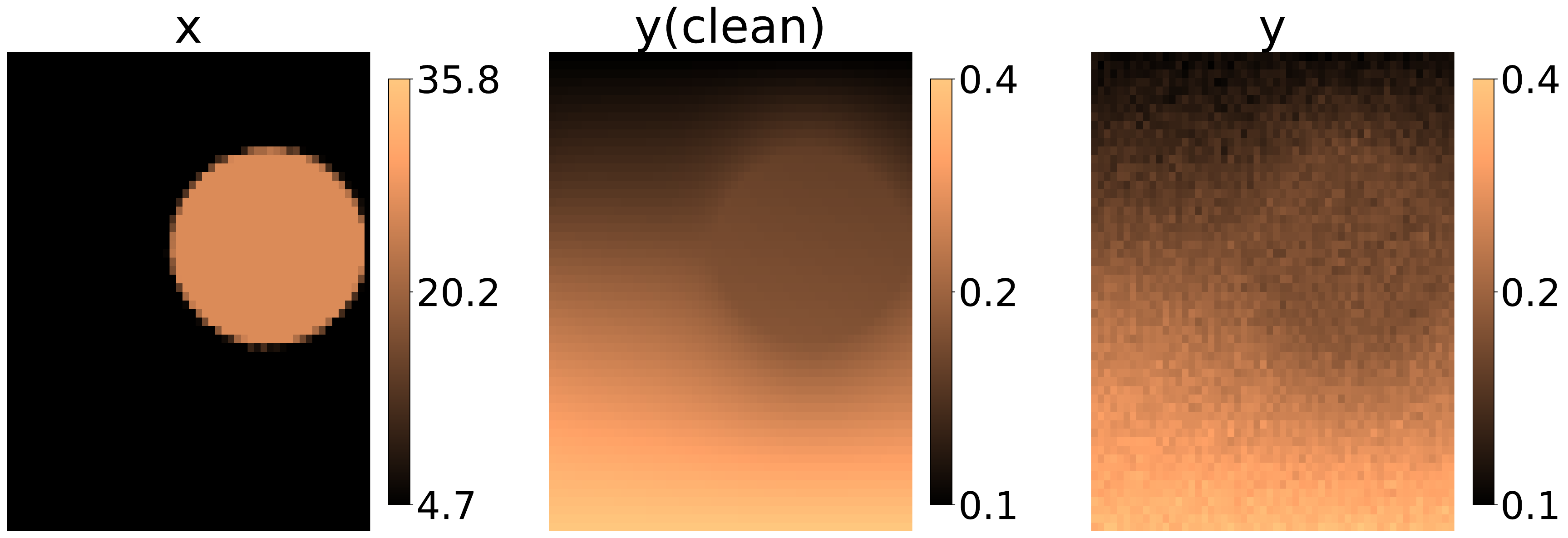}}
\caption{Samples from the dataset (circular priors) that was used to train the network for inferring shear modulus.}
\label{fig:elasticity_samples}
\end{figure} 

\begin{figure}[htbp]
\centering
\includegraphics[width=0.6\textwidth]{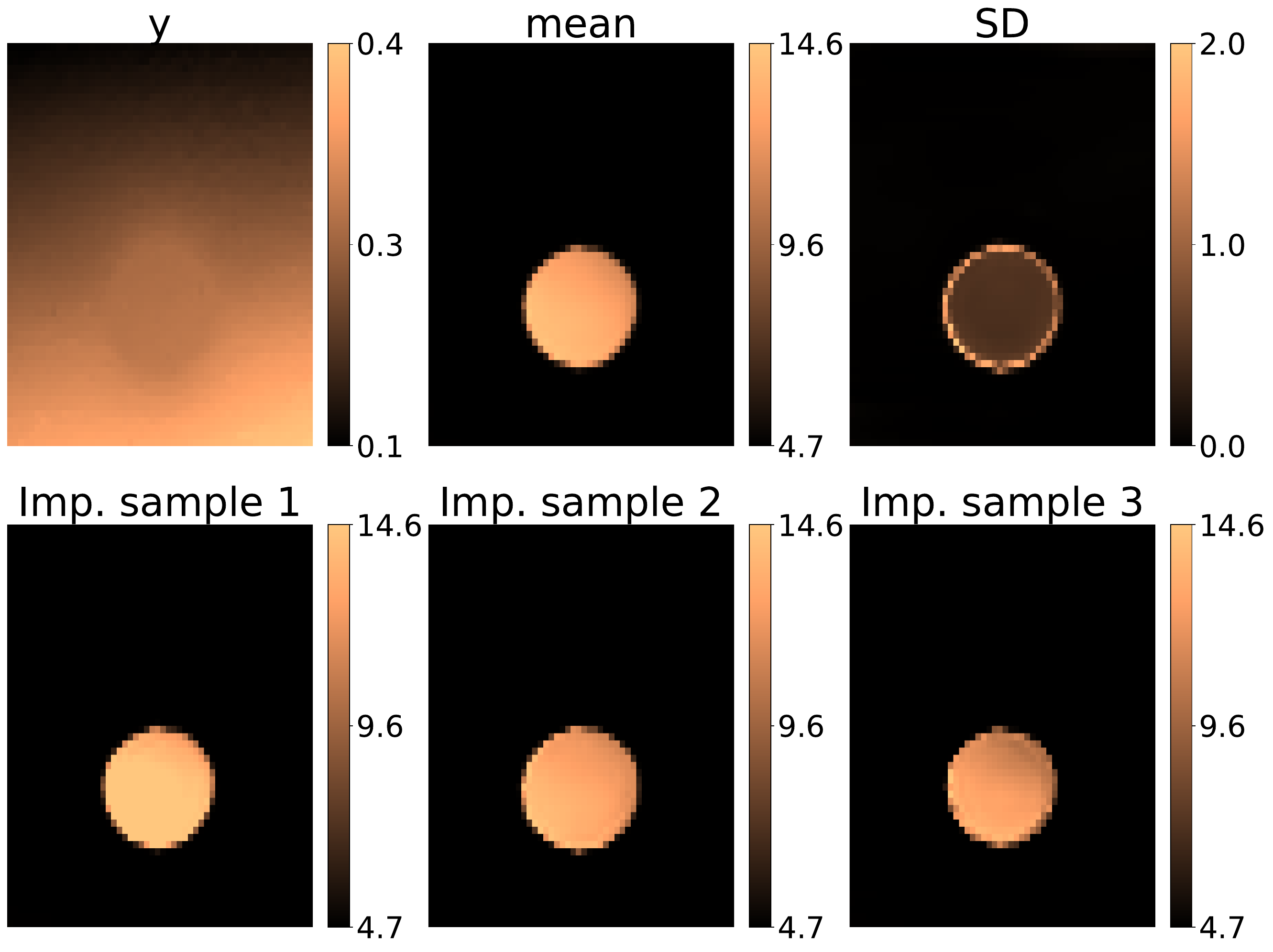}
\caption{Inference on an experimentally measured displacement field.}
\label{fig:elasticity_true}
\end{figure}

\section{Conclusion}\label{sec:conclusion}
In this manuscript, we proposed a deep learning algorithm to solve Bayesian inference problems where the forward map is modelled using a PDE. Conditional WGANs are trained on paired samples of the inferred field and measurements, with the final generator used to generate samples from the posterior distribution given a noisy (test) measurement. The U-Net architecture of the generator is able to capture multi-scale features of the measurement to reconstruct the inferred field. An important component of the proposed deep learning model is the use of CIN to inject the latent information into the generator. This has a threefold benefit i) the dimension of the latent space can be chosen independently of the dimension of the measurement or inferred field, ii) stochasticity is introduced at multiple scales of the generator, and iii) the added stochasticity allows the use of a simple critic architecture without suffering from mode-collapse.

The proposed cWGAN is used to solve inverse problems arising from three different PDE models. The numerical experiments show that the local values of SD is a good measure of the uncertainty in the reconstructed fields (see Figure \ref{fig:mnist_test}, for example). It is shown that the trained cWGANs can lead to reasonable predictions when used with OOD measurements. A theoretical reasoning is provided to explain the generalizability of the networks, which hinges on the local nature of the learned inverse map and the true (regularized) inverse map. This locality is demonstrated, either analytically or empirically, when presenting the numerical results. Finally, we show for the inverse elasticity imaging problem that the network trained on synthetic data yields satisfactory reconstructions when tested on experimental measurements. 

The results presented in this work provide a first look into the generalizability of cWGANs. However, a deeper analysis is warranted to unravel the full potential of such an approach. Further, it worth exploring the benefits of embedding the underlying physics into the network architecture, or the objective function. These will be the focus of future work.

\section*{Acknowledgments}
The authors acknowledge support from ARO grant W911NF2010050, NRL grant N00173-19-P-1119, and AIER (USC), and resources from the Center for Advanced Research Computing (USC).

\bibliographystyle{unsrt}

\appendix


\section{cWGAN loss and architecture}\label{app:gan}
The cWGANs in this work were trained using the following loss function 
\begin{equation*}
\begin{aligned}
    \loss(d,\g)  
&=  \ex{(\x,\y) \sim \pxy \\ \z \sim \pz}{d(\x,\y) - d(\g(\z,\y),\y) + \mathcal{G}\mathcal{P} },\\
\mathcal{G}\mathcal{P} &=\lambda \ex{\epsilon \sim \mathcal{U}(0,1)}{(\|\partial_{1}d( \boldsymbol{h}(\x,\y,\z,\epsilon),\y)\|_2 - 1)^2},
\end{aligned}
\end{equation*}
which has been augmented with a gradient penalty term to constrain the critic to be 1-Lipschitz on $\dbx$. Here, $\mathcal{U}(0,1)$ denotes the uniform distribution on $[0,1]$, $\partial_1 d(.,.)$ denotes the derivative with respect to the first argument of the critic, and
\[
\boldsymbol{h}(\x,\y,\z,\epsilon) = \epsilon \x + (1-\epsilon) \g(\z,\y).
\]
For all experiments considered in the present work, the gradient penalty parameter is set as $\lambda = 10$. 

As shown in Figures \ref{fig:arch}, the architecture of the generator and critic comprises various network blocks. We describe the key components below.

\begin{itemize}

    \item Conv($n,s,k$) denotes a 2D convolution with $k$ filters of size $n$ and stride $s$. If $n > 1$, reflective padding of width 1 is applied in the spatial dimensions of the input before applying the convolution. Furthermore, if the third argument is absent, the number of filter is taken as equal to the number of channels in the input tensor.
    \item The CIN block is used to inject the latent variable $\z$ into different levels of the generator's U-Net architecture. Its action (channel-wise) is given by \eqref{eqn:norm}.
    \item Res block denotes a residual block, which is shown in Figure \ref{fig:gan_blocks}(c). 
    \begin{itemize}
    \item When appearing in the generator, it takes as input an intermediate tensor $\w$ and the latent variable $\z$. The latent variable is injected into the residual block using conditional instance normalization (CIN). 
    \item When appearing in the critic, it takes as input only the intermediate tensor $\w$, with CIN replaced by layer normalization.
    \item In certain cases, when explicitly mentioned, no normalization is used. This typically happens towards the beginning of the generator and critic networks.
    \end{itemize}
    
    \item Down($k$) denotes the down-sampling block shown in Figure \ref{fig:gan_blocks}(a), where $k$ denotes the factor by which the number of input channels increases. This block makes use of 2D average pooling to reduce the spatial resolution by a factor of 2. When appearing in the critic, the block only takes as input the intermediate tensor $\w$.
    
    \item Up($k$) denotes the up-sampling block shown in Figure \ref{fig:gan_blocks}(b). It receives the output $\w$ from the previous block and concatenates it (unless specified) with the output $\widetilde{\w}$ of a down-sampling block of the same spatial size through a skip connection. The Up($k$) block reduces the number of output channels to $C^\prime/k$, where $C^\prime$ denotes the number of channels in the input $\w$. Unless specified, 2D nearest neighbour interpolation is used to increase the spatial resolution by a factor of 2. 
    
    \item Dense($k$) denotes a fully connected layer of width $k$
    
\end{itemize}
The size of the generator and critic used for the various problems are as follows:
\begin{itemize}
    \item \textbf{Inferring initial condition:} A smaller architecture is used, compared to what is shown in Figure \ref{fig:arch}. More precisely, the U-Net is shallower, with one less Down(2) block and Up(2) block. The critic has one less Down(2) block. For both datasets, we set $H=W=28$ and $C=32$.
    \item \textbf{Inferring conductivity:} The architecture is as shown in Figure \ref{fig:arch} with the only difference being that CIN was not used in the first Up(1) layer. Further, we set $H=W=64$ and $C=32$.
    \item \textbf{Inferring shear modulus:} The architecture is as shown in Figure \ref{fig:arch}, with $H=W=56$ and $C=32$.
\end{itemize}

\begin{figure}[htbp]
\centering
\subfigure[Down(k)]{\includegraphics[width=0.44\textwidth]{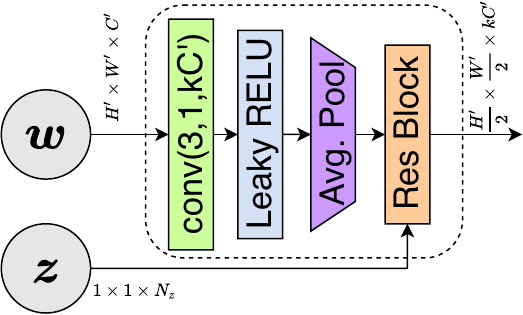}}  
\subfigure[Up(k)]{\includegraphics[width=0.48\textwidth]{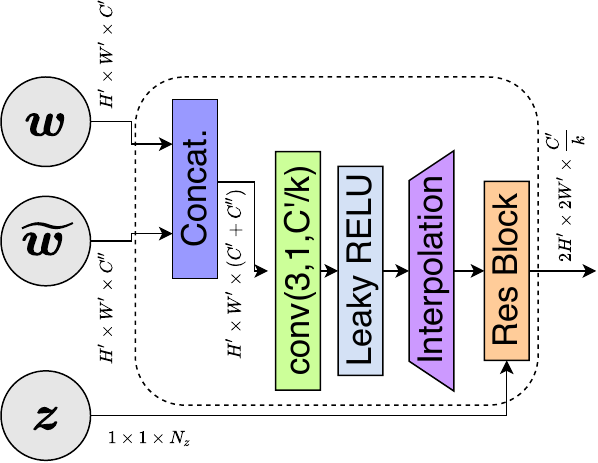}}  
\subfigure[Res block]{\includegraphics[width=0.6\textwidth]{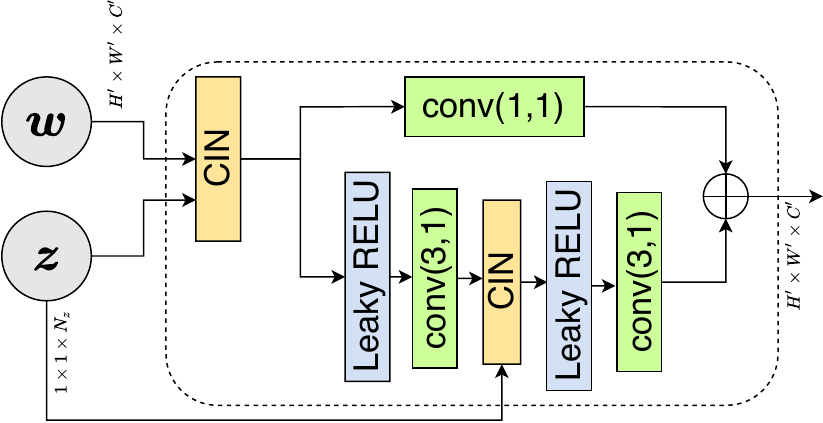}}
\caption{Key components used to build the generator and critic networks.}
\label{fig:gan_blocks}
\end{figure}

The networks are trained on TensorFlow, using the Adam optimizer for 1000 epochs with $\beta_1=0.5$, $\beta_2=0.9$ and learning-rates $10^{-3}$. The remaining hyper-parameters are listed in Table \ref{tab:hparams}. Each network takes less than 20 hours to train on a single NVIDIA Tesla P100 GPU.

\begin{table}[!htbp]
\renewcommand{\arraystretch}{1.5}
\centering
\caption{Hyper-parameters for cWGAN}
\begin{tabular}{c c c c c}
\toprule
Inferred field & \multicolumn{2}{c}{\begin{tabular}[c]{@{}c@{}}Initial condition\end{tabular}} & \begin{tabular}[c]{@{}c@{}}Conductivity\end{tabular} & \begin{tabular}[c]{@{}c@{}}Sheer modulus\end{tabular} \\
\cmidrule{2-3}
Training Data &   Rectangular &   MNIST  & Circles & Circles \\ 
\midrule
Training samples & 10,000 & 10,000 & 8000 & 8000\\
$\Nx$ & $28\times28$ & $28\times28$ & $64\times64$ & $56\times56$\\
$\Nz$ & Multiple & 100 & 50 & 50\\
Batch size & 50 & 50 & 64 & 64\\
Activation param. & 0.1 & 0.1 & 0.2 & 0.2 \\
$n_\text{critic}/n_\text{gen}$ & 4 & 4 & 5 & 5 \\
\bottomrule
\end{tabular}\label{tab:hparams}
\end{table}

\end{document}